\documentclass[letterpaper,11pt]{article}
\usepackage{amsmath,amsthm,amsfonts,amssymb}
\usepackage[margin=1in]{geometry}
\usepackage{graphicx}

\usepackage[
            CJKbookmarks=true, 
            bookmarksnumbered=true,
            bookmarksopen=true,
            colorlinks=true,
            citecolor=red,
            linkcolor=blue,
            anchorcolor=red,
            urlcolor=blue
            ]{hyperref}

\newcommand{\delp}{\delta}

\newcommand{\Bin}{\mathrm{Bin}}

\newcommand{\xitrue}{\xi_{\rm true}}
\newcommand{\xifake}{\xi_{\rm fake}}
 
\newcommand{\stepa}[1]{\overset{\rm (a)}{#1}}
\newcommand{\stepb}[1]{\overset{\rm (b)}{#1}}
\newcommand{\stepc}[1]{\overset{\rm (c)}{#1}}

\usepackage{tikz}

\usetikzlibrary{matrix,arrows,calc}
\usetikzlibrary{shapes.callouts,decorations.text} 

\tikzstyle{int}=[draw, minimum size=2em, align=center] 
\tikzstyle{dot}=[circle, draw, minimum size=2em]
\tikzstyle{dotred}=[circle, draw, minimum size=2em]
\tikzstyle{init} = [pin edge={to-,thin,black}]
\tikzstyle{initred} = [pin edge={to-,thin,red}]
\tikzstyle{plan}=[draw, minimum size=2em, text width=5em, rounded corners,align=center]
\tikzstyle{planwide}=[draw, minimum size=2em, text width=8em, rounded corners,align=center]
\tikzstyle{planverywide}=[draw, minimum size=2em, text width=10em, rounded corners,align=center]

\usepackage{tikz-cd}

\newtheorem{theorem}{Theorem}
\newtheorem{lemma}{Lemma}

\theoremstyle{definition}

\newtheorem{definition}{Definition}
\newtheorem{remark}{Remark}

\usepackage[title]{appendix}
\usepackage{color}
\usepackage{subfigure}
\usepackage{algorithm}
\usepackage{algorithmic}

\usepackage{xspace,prettyref,enumerate}

\newcommand{\diverge}{\to\infty}

\newcommand{\iiddistr}{{\stackrel{\text{\iid}}{\sim}}}

\newcommand{\zeros}{\mathbf 0}
\newcommand{\reals}{{\mathbb{R}}}

\newcommand{\naturals}{{\mathbb{N}}}

\newcommand{\identity}{\mathbf I}

\newcommand{\nb}[1]{#1}

\newcommand{\Expect}{\mathbb{E}}
\newcommand{\expect}[1]{\mathbb{E}\left[ #1 \right]}

\newcommand{\expects}[2]{\mathbb{E}_{#2}\left[ #1 \right]}

\newcommand{\prob}[1]{ \mathbb{P}\left\{ #1 \right\} }

\newcommand{\var}{\mathsf{var}}
\newcommand{\Cov}{\text{Cov}}

\newcommand{\Bern}{{\rm Bern}}
\newcommand{\Binom}{{\rm Binom}}
\newcommand{\Binomc}{\overline{\Binom}}

\newcommand{\eg}{e.g.\xspace}
\newcommand{\ie}{i.e.\xspace}
\newcommand{\iid}{i.i.d.\xspace}
\newrefformat{eq}{(\ref{#1})}
\newrefformat{chap}{Chapter~\ref{#1}}
\newrefformat{sec}{Section~\ref{#1}}
\newrefformat{alg}{Algorithm~\ref{#1}}
\newrefformat{fig}{Fig.~\ref{#1}}
\newrefformat{tab}{Table~\ref{#1}}
\newrefformat{rmk}{Remark~\ref{#1}}
\newrefformat{clm}{Claim~\ref{#1}}
\newrefformat{def}{Definition~\ref{#1}}
\newrefformat{cor}{Corollary~\ref{#1}}
\newrefformat{lmm}{Lemma~\ref{#1}}
\newrefformat{prop}{Proposition~\ref{#1}}
\newrefformat{app}{Appendix~\ref{#1}}
\newrefformat{hyp}{Hypothesis~\ref{#1}}
\newrefformat{thm}{Theorem~\ref{#1}}
\newrefformat{ass}{Assumption~\ref{#1}}
\newrefformat{conj}{Conjecture~\ref{#1}}

\newcommand{\pth}[1]{\left( #1 \right)}

\newcommand{\sth}[1]{\left\{ #1 \right\}}

\newcommand{\iprod}[2]{\left \langle #1, #2 \right\rangle}
\newcommand{\Iprod}[2]{\langle #1, #2 \rangle}
\newcommand{\indc}[1]{{\mathbf{1}_{\left\{{#1}\right\}}}}

\newcommand{\iindc}[1]{\mathbf{1}_{#1}}

\newcommand{\calC}{{\mathcal{C}}}
\newcommand{\calD}{{\mathcal{D}}}
\newcommand{\calE}{{\mathcal{E}}}
\newcommand{\calF}{{\mathcal{F}}}
\newcommand{\calG}{{\mathcal{G}}}
\newcommand{\calH}{{\mathcal{H}}}
\newcommand{\calI}{{\mathcal{I}}}

\newcommand{\calK}{{\mathcal{K}}}
\newcommand{\calL}{{\mathcal{L}}}
\newcommand{\calM}{{\mathcal{M}}}
\newcommand{\calN}{{\mathcal{N}}}

\newcommand{\calS}{{\mathcal{S}}}
\newcommand{\calT}{{\mathcal{T}}}

\newcommand{\calX}{{\mathcal{X}}}

\DeclareMathAlphabet{\varmathbb}{U}{bbold}{m}{n}

\renewcommand{\hat}{\widehat}
\renewcommand{\tilde}{\widetilde}

\usepackage{bbm}

\newcommand{\vecc}{\mathrm{vec}}

\newcommand{\polylog}{\mathrm{polylog}}

\newcommand{\ER}{Erd\H{o}s-R\'{e}nyi\xspace}

\begin{document}
\title{Efficient random graph matching via degree profiles}
\author{Jian Ding, Zongming Ma, Yihong Wu, Jiaming Xu\thanks{
J.\ Ding and Z. Ma are with Department of Statistics, The Wharton School, University of Pennsylvania, Philadelphia, USA, \texttt{\{dingjian,zongming\}@wharton.upenn.edu}.
Y.\ Wu is with Department of Statistics and Data Science, Yale University, New Haven, USA, \texttt{yihong.wu@yale.edu}.
J.\ Xu is with The Fuqua School of Business, Duke University, Durham, USA, \texttt{jx77@duke.edu}.
}}
\maketitle

\begin{abstract}
Random graph matching refers to recovering the underlying vertex correspondence between two random graphs with correlated edges; a prominent example is when the two random graphs are given by Erd\H{o}s-R\'{e}nyi graphs $G(n,\frac{d}{n})$. This can be viewed as an average-case and noisy version of the graph isomorphism problem. Under this model, the maximum likelihood estimator is equivalent to solving the intractable quadratic assignment problem. This work develops an $\tilde{O}(n d^2+n^2)$-time algorithm which perfectly recovers the true vertex correspondence with high probability, provided that the average degree is at least $d = \Omega(\log^2 n)$ and the two graphs differ by at most $\delta = O( \log^{-2}(n) )$ fraction of edges. For dense graphs and sparse graphs, this can be improved  to $\delta = O( \log^{-2/3}(n) )$ and $\delta = O( \log^{-2}(d) )$ respectively, both in polynomial time. The methodology is based on appropriately chosen distance statistics of the degree profiles (empirical distribution of the degrees of neighbors). Before this work, the best known result achieves $\delta=O(1)$ and $n^{o(1)} \leq d \leq n^c$ for some constant $c$ with an $n^{O(\log n)}$-time algorithm \cite{barak2018nearly} and $\delta=\tilde O((d/n)^4)$ and $d = \tilde{\Omega}(n^{4/5})$ with a polynomial-time algorithm \cite{dai2018performance}.

\end{abstract}

\tableofcontents

\section{Introduction}
	\label{sec:intro}

Graph matching~\cite{conte2004thirty,Livi2013}, also known as network alignment~\cite{feizi2016spectral},
aims at finding a bijective mapping between the vertex sets of two networks so that
the number of adjacency disagreements between the two networks is minimized. 
It reduces to the graph isomorphism problem  
in the noiseless setting where the two networks can be matched perfectly.

The paradigm of graph matching has found numerous applications  across a variety of diverse fields, such as network privacy, 
computational biology, computer vision, and natural language processing. 
For instance, it was convincingly demonstrated~\cite{narayanan2009anonymizing,narayanan2008robust} that  hidden vertex identities in a network can nevertheless be recovered by matching the anonymized network (such as Netflix)
to a secondary network with known vertex identities (such as the Internet Movie Database).
In system biology, graph matching is used in discovering protein functions by matching protein-protein interaction
networks across different species~\cite{singh2008global,kazemi2016proper}. In computer vision, using graphs to represent images, where vertices are regions in the images and edges encode the adjacency relationships
between different regions,  graph matching is widely applied in finding similar images~\cite{conte2004thirty,schellewald2005probabilistic}. 
In natural language processing, using graphs to represent sentences, 
where vertices are phrases and edges represent syntactic and semantic relationships,
graph matching is used in question answering, machine translation, and information retrieval~\cite{haghighi2005robust}.

Given two graphs with adjacency matrices $A$ and $B$, the graph matching problem 
can be viewed as a special case of the \emph{quadratic assignment problem} (QAP)~\cite{Pardalos94thequadratic,burkard1998quadratic}: 
namely, 
\begin{align}
\max_{\Pi}  \Iprod{A}{\Pi B \Pi^\top},
\label{eq:QAP}
\end{align}
 where $\Pi$ ranges over all $n \times n$ permutation matrices, and
 $\iprod{\cdot}{\cdot}$ denotes the matrix inner product.
QAP is NP-hard in the worst case. Moreover, approximating QAP within a factor of $2^{\log^{1-\epsilon} (n) }$ for $\epsilon > 0$ is also NP-hard~\cite{makarychev2010maximum}.

These hardness results, however, are applicable in the worst case, where the observed networks are designed by an adversary.  In contrast, the networks in many aforementioned applications can be modeled by random graphs with latent structures; as such, our focus is not in the worst-case instances, but rather in recovering the underlying vertex permutation with high probability in order to reveal the hidden structures. 

\subsection{Correlated \ER graphs model}
\nb{Driven by applications in social networks and biology, 
a recent line of work~\cite{Pedarsani2011,Lyzinski2013Seeded,yartseva2013performance,korula2014efficient,kazemi2015growing,feizi2016spectral,cullina2016improved,cullina2017exact,lubars2018correcting,barak2018nearly,dai2018performance,cullina2018partial} 
initiated the statistical analysis of graph matching by assuming that $A$ and $B$ are generated randomly. The simplest such model is the following correlated \ER graph model:}



\begin{definition}[Correlated \ER model $\calG(n,q; s)$]
Given an integer $n$ and  $q, s \in [0,1]$, 
let $A$ and $B$ denote the adjacency matrix of two 
\ER random graphs $\calG(n,q)$ on the same vertex set $[n]$. 
Let $\pi^*:[n]\to [n]$ denote a latent permutation. 
We assume that conditional on $A$, for all $i<j$, 
$B_{ \pi^*(i) \pi^*(j)  }$ are independent and distributed as
\begin{align}
B_{ \pi^*(i) \pi^*(j)  } \sim
\begin{cases}
\Bern(s) & \text{ if } A_{ij}=1 \\
\Bern \left(  \frac{q (1-s) }{1-q} \right) & \text{ if } A_{ij}=0
\end{cases},
\label{eq:def_correlated_ER}
\end{align}
where $\Bern(s)$ denotes a Bernoulli distribution with mean $s$.
\end{definition}

Equivalently, the two graphs can be viewed as edge-subsampled subgraphs of a  parent  \ER graph $G \sim \calG(n,p)$ with $p=q/s$. Let $A$ be the adjacency matrix of a graph obtained by keeping or deleting 
each edge of $G$ independently with probability $s$ and $\delp\triangleq 1-s$ respectively.
Repeat the sampling process independently and relabel the vertices according to 
the latent permutation $\pi^*$ to obtain $B$.\footnote{To ensure the Bernoulli parameter
in~\prettyref{eq:def_correlated_ER}
is well-defined, we need to assume $q(1-s) \le 1-q$, or equivalently $s \ge 2-1/q$. 
Similarly, to ensure the edge probability in the parent graph $p=q/s \le 1$, we need to assume $ s \ge q$.}
Note that by \prettyref{eq:def_correlated_ER}, the parameter
$s$ can be viewed as a measure of the edge correlations.
Alternatively, $\delp=1-s$ can be interpreted as 
the fraction of edges in $A$ that are substituted in $B$ on average.

Upon observing $A$ and $B$, the goal is to exactly recover 
the latent vertex correspondence $\pi^*$ with probability converging to $1$ as $n \to \infty$.
For instance, in network de-anonymization,  the parent  \ER graph $G$ corresponds to 
the underlying friendship network of a group of people, $A$ corresponds 
to a Facebook friendship network of the same group of people with known identities, and $B$ is the Twitter network of the same set of users with identities removed; 
the task is to de-anonymize the vertex identities in the Twitter network 
by finding the underlying mapping between the vertex sets of $A$ and $B$.

In the noiseless case of $s=1$, graph matching under the $\calG(n,q;1)$ model reduces to the problem of random graph isomorphism for \ER graph $\calG(n,q)$. 
In this case, a celebrated result~\cite{wright1971graphs} (see also \cite[Chap.~9]{bollobas1998random}) shows that
exact recovery of the underlying permutation is information-theoretically
possible if and only if $ nq \ge \log n+ \omega(1)$ for $q\le 1/2$;\footnote{
\nb{Throughout the paper, we use standard big $O$ notation,
e.g., for any sequences $\{a_n\}$ and $\{b_n\}$, $a_n=\Theta(b_n)$ (or $a_n \asymp b_n$) if $1/c\le a_n/ b_n \le c$ holds for all $n$ for some absolute constant $c>0$; $a_n =\Omega(b_n)$ and $b_n = O(a_n)$ (or $a_n \gtrsim b_n$ and $b_n \lesssim a_n$) if $a_n/b_n \ge c$. We use big $\tilde{O}$ notation to hide logarithmic factors. }} in other words, the symmetry (i.e.,~the automorphism group) of the graph is trivial with high probability.
Recent work~\cite{cullina2016improved,cullina2017exact} has extended this result to the noisy case where $s<1$, 
showing that
exact recovery is information-theoretically possible if and only if $nqs \ge \log n + \omega(1)$, under the additional
assumption that $q \le O(\log^{-1} n)$ and $q (1-s)^2/s \le O( \log^{-3}(n))$.\footnote{Achievability and converse bounds for 
more general correlated \ER random graph models 
are also available in~\cite{cullina2016improved,cullina2017exact}.}


From a computational perspective, in the noiseless case of $s=1$, linear-time
algorithms have been found to attain the recovery threshold of $nq= \log n+\omega(1)$~\cite{bollobas1982distinguishing,czajka2008improved}. However, 
in the noisy case, very little is known about the performance guarantees 
of graph matching algorithms that run in polynomial time. Recently a quasi-polynomial-time ($n^{O(\log n)}$) algorithm is proposed in \cite{barak2018nearly} which succeeds when $nqs \in \left[ n^{o(1)}, n^{1/153}] \cup [n^{2/3}, n^{1-\epsilon} \right]$ and $s \ge (\log n)^{-o(1)}$. 
Another recent work~\cite{dai2018performance} adapts the classical degree-matching algorithms in \cite{babai1980random}
and \cite[Section 3.5]{bollobas1998random} from the noiseless case to the noisy case, and shows that 
it exactly recovers $\pi^*$ with high probability, provided that 
$q \gg \log^{7/5}(n)/n^{1/5}$ and $1-s \ll q^4/\log^6 (n)$. This result 
requires $1-s$, the fraction of edges differed in the two observed graphs, to decay polynomially in $q$ 
and is thus far from being optimal.


\subsection{Main Results}


In this work, we significantly improve the state of the art of efficient graph matching algorithms in terms of time complexity, noise tolerance, and sparsity. In particular, we give an $\tilde{O}(n d^2+n^2)$-time algorithm for exactly recovering the true permutation
$\pi^*$ with high probability under the correlated \ER graph model, when 
the fraction of differed edges $\delta=1-s$ can be as large as 
$1/\log^2(n)$ 
and the average degree $d$ can be as low as $\log^2 n$.
Furthermore, we obtain two improved polynomial-time algorithms that aim for dense and sparse graphs respectively. 
These results are summarized as below:

\begin{theorem}\label{thm:main}
Consider the correlated \ER model $\calG(n,q;1-\delp)$ with $q \le q_0$ for some sufficiently small constant $q_0$.
If 
\begin{align}
nq \gtrsim \log^2 n \quad \text{ and } \quad \delp \lesssim \frac{1}{ \left( \log n \right)^2 }, \label{eq:main_cond1}
\end{align}
then there exists an $\tilde{O}\left(nd^2 +n^2 \right)$-time algorithm (cf.\ \prettyref{alg:dist}) that recovers $\pi^*$ with probability $1-O(1/n)$.

Furthermore,
\begin{itemize}
    \item 
    if 
\begin{align}
q= e^{ - O \left(  \left( \log n \right)^{1/3} \right)  } \quad \text{ and } \quad 
\delp \lesssim \frac{1}{\left( \log n \right)^{2/3}}, \label{eq:main_cond2}
\end{align}
then there exists a polynomial-time algorithm (cf.\ \prettyref{alg:distdeg}) that  recovers $\pi^*$ with probability $1-\exp \left( - \Omega( \log^{1/3} (n) ) \right)$;

\item 
if
\begin{align}
\nb{\frac{\log n}{n}} \lesssim q \leq n^{-\epsilon} \quad \text{ and }
\quad 
\delp \lesssim \frac{1}{\left( \log (nq) \right)^2}, \label{eq:main_cond3}
\end{align} 
for some constant $\epsilon>9/10$, then there exists a polynomial-time algorithm (cf.~\prettyref{alg:dist_3_hop}) that  recovers $\pi^*$ with probability $1-O\left(n^{9-10\epsilon}\right)$.
\end{itemize}
\end{theorem}

\subsection{Key algorithmic ideas and techniques for analysis}
\label{sec:idea}

Many existing matching algorithms for random graph isomorphism are \emph{signature-based}: first attach some appropriately chosen signature $\mu_i$ to  vertex $i$ in $A$ and $\nu_k$ to vertex $k$ in $B$, then match each pair based on their similarity, or equivalently, some distance between the signatures. 
For example, degree matching 
simply uses the vertex degree as the signature. In addition, spectral method can be viewed as
assigning the $i$th entry in the leading eigenvector(s) of the matrix $A$ (resp.~$B$) as the signature $\mu_i$ (resp.~$\nu_i$). 
However, these signatures are highly sensitive to noise. Indeed, it can be shown that (cf.~\prettyref{rmk:orderstat} in \prettyref{sec:gaussian}) for degree sorting to yield the exact matching, the minimum spacing between the ordered degrees needs to overcome the effective noise,  which entails $\delp = \tilde{o}(q^2)$.
For spectral methods, due to the lack of low-rank structure and the vanishing spectral gap of \ER graphs, the eigenstructure is extremely fragile. Indeed, it can be shown via perturbation bounds that even for dense graphs, matching via top eigenvectors requires $\delp = O(n^{-c})$ for some constant $c$ to succeed, which agrees with the numerical experiments in \prettyref{sec:exp}.
Therefore, to deal with sparser graphs and smaller edge correlation, we need to find better signatures that are more robust to random perturbation. 

Note that in the absence of any label information, we can only compute signatures that are \emph{permutation-invariant}. 
The main finding of this work is that \emph{degree profiles}, that is, empirical distribution of the degrees of neighbors, can be used as a signature which is significantly more noise-resilient than degrees or eigenvectors. Using a suitable distance between distributions to construct the matching (see the forthcoming \prettyref{alg:dist}), this allows us to correctly match graphs that differ by almost linear number of edges. Specifically, for each vertex $i$ in $A$, its degree profile $\mu_i$ is defined as the empirical distribution of the degrees of $i$'s neighbors. Similarly, for each vertex $k$ in $B$, let $\nu_k$ denote its degree profile.
Then we match vertex $i$ to vertex $k$ which minimizes the total variation ($L_1$-distance) between the appropriately discretized versions of $\mu_i$ and $\nu_k$ (into $\polylog(n)$ bins).
The intuitive explanation for why this works is the following:
\begin{itemize}
    \item 
if $k=\pi^*(i)$, which we call a ``true pair'', then they have a large number of common neighbors, whose degrees, thanks to the edge correlations between $A$ and $B$, are correlated random variables, which tend to lie in the same bin. This leads to a small distance between the degree profiles $\mu_i$ and $\nu_k$; 

\item 
if $k \neq \pi^*(i)$, which we call a ``fake pair'', then $\mu_i$ and $\nu_k$ are empirical distributions consisting mainly  independent samples, and their distance is typically large.
\end{itemize}
Clearly, in reality the situation is significantly more complicated due to various dependencies and the possibility that fake pairs can still have a non-negligible number of common neighbors. 
Furthermore, since for each vertex there exists a unique match but many more ($n-1$) potential mismatches, one need to carefully control the total variation distance between degree profiles for true pairs and fake pairs as well as their large deviation behavior (their distance being atypically small). Nevertheless, our analysis rigorously justifies the above intuition and shows the distance statistic for true pairs and fake pairs are indeed separated with high probability under the condition \prettyref{eq:main_cond1}.

Ideas related to degree profiles have been used for the random graph isomorphism problem. In particular, it is shown in \cite{czajka2008improved,mossel2017shotgun} that \emph{degree neighborhood} (i.e., the multiset of the degrees of neighbors of each vertex) constitutes a canonical labeling for $G(n,q)$ with high probability provided that $q \gg \frac{\log^2 n}{n}$. In the absence of noise, it suffices to prove that the degree neighborhoods of different vertices are distinct with high probability. 
However, how to match vertices in the noisy case and by how many edges the two graphs can differ is far less clear. 
In fact, although degree neighborhood (multiset) contains the same amount of information as degree profile (empirical distribution), for the development of our matching algorithm as well as the analysis, it is crucial to adopt the view of degree profiles as \emph{probability measures}, which enables us to construct a greedy matching based on natural distances between probability distributions. The main observation is that although each degree profile is centered around the same mean (binomial distribution), the stochastic fluctuations are nearly independent for fake pairs and correlated for true pairs. This perspective allows us to leverage insights from empirical process theory to study the large deviation behavior of distances between degree profiles.

For relatively dense graphs with edge probability $q = \exp(-O( \log^{1/3} n))$, we further relax  the condition from $\delp \lesssim \log^{-2} n$ to $\delp \lesssim \log^{-2/3} n$ by combining the degree
profile matching with vertex degrees in conjunction with the paradigm of seeded graph matching (cf.~\prettyref{alg:distdeg}). In particular, we show that even if for some vertices
the distance statistics between degree profiles of fake pairs can be smaller than that of the true match, with high probability this does not occur for vertices of sufficiently high degrees. Although the matched high-degree vertices occupy only a vanishing fraction of the vertex set, they provide enough initial ``seeds'' (correctly matched pairs) to match the remaining vertices with high probability under the condition \prettyref{eq:main_cond2}. 
A key challenge in the analysis is to carefully control the dependency between vertex degrees and degree profiles,
and to characterize the statistical correlation among vertex degrees.
Furthermore, 
we provide an efficient seeded graph matching subroutine via maximum bipartite matching, which is guaranteed to succeed with 
$\Omega(\frac{\log n}{q})$ seeds, even if the seed set is chosen adversarially. 
\nb{A different seeded matching algorithm was previously proposed in~\cite{barak2018nearly} allowing possibly incorrect seeds and assuming a relaxed condition on the graph sparsity; however, the number of seeds needed in the worst-case is $\Omega(\max\{\frac{\log n}{q}, qn \log n\})$ (see the condition in Lemma 3.21 and before Lemma 3.26 in \cite{barak2018nearly}), which cannot be afforded in the dense regime.}

Note that degree profile matching is a \emph{local algorithm} that uses only $2$-hop neighborhood information for each vertex. It turns out that 
for relatively sparse graphs with edge probability $q \le n^{-\epsilon}$ for a fixed constant $\epsilon>9/10$, 
we can further relax the condition from $\delta \lesssim \log^{-2}(n)$ to $\delta \lesssim \log^{-2}(nq)$, using the $3$-hop neighborhood information. This is carried out in three steps:
for each neighbor $j$ of vertex $i$ in $A$ and each neighbor $j'$ of vertex $k$ in $B$, we first compute the total variation  distance between the degree profiles of $j$ and $j'$ as before, and then threshold the distances to construct a bipartite graph between the neighbors of vertex $i$ and the neighbors of vertex $k$, and finally define a similarity score $W_{ik}$ as the size of the maximum matching of this bipartite graph (cf.~\prettyref{alg:dist_3_hop}). 
We show that these new similarity measures for true pairs and fake pairs are separated with high probability under the condition \prettyref{eq:main_cond3}.
Finally, we mention that in the noiseless case, \nb{the algorithm of \cite{bollobas1982distinguishing} that achieves the optimal threshold for sparse graphs (with average degree $\polylog(n)$) uses as the signature the distance sequence of each vertex}, which consists of the number of $\ell$-hop neighbors for $\ell$ from $1$ up to $\Theta(\frac{\log n}{\log\log n})$. This significantly improves the performance of degree matching \cite{babai1980random}. 
It remains open whether local algorithms that use larger neighborhood information 
can further improve the graph matching performance in the noisy case. 

\subsection{Further Related Work}

\paragraph{Convex relaxation}

There exists a large body of literature on convex relaxation of the graph matching problem; for a comprehensive discussion we refer the reader to \cite{dym2017ds++}. 
One popular approach is doubly stochastic relaxation, 
which entails replacing the objective \prettyref{eq:QAP} by minimizing $\|AX-XB\|_F^2$, with $\|\cdot\|_F$ standing for the Frobenius norm, and relaxing the decision variable $X$ from the set of permutation matrices into its convex hull, i.e., all doubly stochastic matrices \cite{aflalo2015convex,fiori2015spectral}. This leads to a quadratic programming problem which is solvable in polynomial time but still much slower than the degree profile algorithm. Some initial statistical analysis for the correlated \ER graph model was carried out in \cite{lyzinski2016graph}; however, its performance guarantees remain far from being well-understood.

There exists a conceptual connection between the degree profile matching algorithm and the doubly stochastic relaxation. In graph theory, two graphs are said to be fractionally isomorphic if their adjacency matrices $A$ and $B$ satisfy $AX=XB$ for some doubly stochastic matrix $X$. 
A result due to Ramana, Scheinerman, and Ullman (cf.~\cite[Theorem 6.5.1]{FGT}) states that a necessary and sufficient condition for fractional isomorphism is that two graphs have identical \emph{iterated degree sequences}; see \cite[Sec.~6.4]{FGT} for a precise definition. 
In particular, the first term of the iterated degree sequence corresponds to the degree distribution of the graph (i.e.~the empirical distribution of the vertex degrees), while the second term is precisely the empirical distribution of degree profiles. 
In this perspective, our algorithm can be thought as using the leading two terms in the iterated degree sequence to construct the matching. 
Thus it is to be expected that degree profile matching algorithm outperforms degree matching but not the doubly stochastic relaxation.

Another approach is the semidefinite programming (SDP) relaxation for QAP \cite{zhao1998semidefinite} which is provably tighter than the doubly stochastic relaxation (cf.~\cite{kezurer2015tight}). However, this entails solving an SDP in the lifted domain of $n^2\times n^2$ matrices and the computational cost becomes prohibitively high even for moderate $n$.

\paragraph{Seeded Graph Matching}  
Another  recent line of work~\cite{Pedarsani2011,yartseva2013performance,korula2014efficient,Lyzinski2013Seeded,Fishkind2018Seeded,Shirani2017Seeded}
in graph matching considers a relaxed version of the problem, 
where an initial seed set of correctly matched vertex pairs is revealed. This is motivated by the fact that in many practical applications, some side information on the vertex identities are available and have been successfully utilized to match many real-world networks~\cite{narayanan2009anonymizing,narayanan2008robust}.
It is shown in~\cite{yartseva2013performance} 
that if $nq=\Theta(\log n)$ and the number of seeds is $\Omega(n/ (s^2\log n)^{4/3} )$, 
then a percolation-based graph matching
algorithm correctly matches all but $o(n)$ vertices  in polynomial time with high probability. 
Another work~\cite{korula2014efficient} shows that if $q<1/6$, then with at least $24 \log n / (qs^2)$ seeds, one can match all vertices correctly in polynomial time with high probability. 
More recently, it is shown in~\cite{mossel2018seeded}  that 
the information-theoretic limit $nqs \ge \log n+\omega(1)$ in terms of the graph sparsity can be attained  
in polynomial time, provided that $s=\Theta(1)$ and the number of seeds is $\Omega(n^{3\epsilon})$ in the sparse graph regime 
($nq \le n^{\epsilon}$ for $\epsilon<1/6$) and $\Omega(\log n)$ 
in some dense graph regime.  


\subsection{Notation and Organization}

Denote the identity matrix by $\identity$.
We let $\|X\|_F$ denote the Frobenius norm of a matrix $X$
and $\|x\|_2$ denote the $\ell_2$ norm of a vector $x$.
For any positive integer $n$, let $[n]=\{1, \ldots, n\}$.
For any set $T \subset [n]$, let $|T|$ denote its cardinality and $T^c$ denote its complement.
\nb{Let $\delta_{x}$ denote the Dirac measure (point mass) at $x$}. 
We say a sequence of events $\calE_n$ indexed by a positive integer $n$ 
holds with high probability, if the probability of $\calE_n$ converges to $1$ as $n \to +\infty$.
Without further specification, all the asymptotics are taken with respect to $n \to \infty$. 
All logarithms are natural and we use the convention $0 \log 0=0$.
For two real numbers $a$ and $b$, we use $a \vee b = \max\{a, b\}$ (resp.~$a \wedge b = \min\{a, b\}$) 
to denote the maximum (resp.~minimum) between $a$ and $b$. 
We denote by $\Bern(p)$ the Bernoulli distribution with mean $p$
and $\Binom(n,p)$ the Binomial distribution with $n$ trials and success probability $p$.

The rest of the paper is organized as follows:
In \prettyref{sec:gaussian}, we provide a self-contained account of the problem of matching two Wigner random matrices.
This part is intended as a warm-up for \ER graphs and serves to explain the main intuition behind the degree profile algorithms and the connection to empirical process theory and small ball probability.
\prettyref{sec:alg} describes the matching algorithms for the correlated \ER model
 and presents their theoretical guarantees. Specifically, \prettyref{sec:dist} introduces the main algorithm for degree profile matching, with further improvements given in \prettyref{sec:dense} and \prettyref{sec:sparse} for dense and sparse graphs, respectively.
\prettyref{sec:pf} provides the proof of correctness, with some auxiliary lemmas deferred to \prettyref{app:aux}. 
\prettyref{app:seed} contains our seeded graph matching result.
Empirical evaluations of various algorithms on both simulated and real graphs are given in Section \ref{sec:exp}.




\section{Warm-up: Matching Gaussian Wigner matrices}
	\label{sec:gaussian}

In this section we take a slight detour to consider the Gaussian version of the graph matching problem, which can also be viewed as a statistical model for the QAP problem \prettyref{eq:QAP} with correlated Gaussian weights. 
Although the proofs for correlated \ER graphs 
do not exactly follow the same program, by studying this simpler model, we aim to convey the main idea behind the degree profile algorithm and sketch how to deduce the theoretical guarantees from results in empirical process theory and small ball probability.

\subsection{Correlated Wigner model} 
Consider two random symmetric matrices $A$ and $B'$, whose entries 
$\{(A_{ij}, B'_{ij}): 1 \leq i \leq j \leq n\}$ are iid correlated standard normal pairs with correlation coefficient $\rho$, i.e., $(A_{ij}, B'_{ij}) \iiddistr N(\zeros, (\begin{smallmatrix} 1&\rho \\ \rho& 1 \end{smallmatrix}) )$.
In other words, $A$ and $B'$ are two correlated Wigner matrices. 
Let $\pi^* \in S(n)$ be a permutation on
$[n]$ and $\Pi^*$ be its corresponding $n \times n$ permutation matrix.
Let $B = \Pi^* B' (\Pi^*)^\top$. 
Observing the two matrices $A$ and $B$, the goal is to estimate the latent permutation $\pi^*$ correctly with high probability.

 Without loss of generality, we assume $\rho >0$ and let $\rho = \sqrt{1-\sigma^2}$ for some $0 < \sigma^2 < 1$, and, furthermore, $\Pi^*=\identity$. Therefore, we can write $B = \sqrt{1-\sigma^2} A  + \sigma Z$, where $A$ and $Z$ are two independent Wigner matrices.

\subsection{Matching via empirical distributions} 
\label{sec:gaussian-emp}
Next we describe a procedure for matching Wigner matrices as well as an improved version, which serve as the precursors to \prettyref{alg:dist} and \prettyref{alg:distdeg} for \ER graphs.


The main idea is to use the empirical distribution of each row as the signature, and rely on appropriate distance between distributions to construct the matching. Specifically, for each $i$, define
\[
\mu_i = \frac{1}{n} \sum_{j=1}^n \delta_{A_{ij}}
\]
which is the empirical distribution of the $i$th row of $A$. Similarly, define
\[
\nu_k = \frac{1}{n} \sum_{j=1}^n \delta_{B_{kj}}
\]
for the $B$ matrix. 
Marginally, for any $i,k$, both $\mu_i$ and $\nu_k$ are the empirical distributions of $n$ standard normal samples. The difference is that if $i$ and $k$ form a true pair, the samples are correlated; otherwise, the samples are independent.\footnote{To be precise, all but two elements (namely, $A_{ik}$ and $B_{ki}$) are independent. This can be easily dealt with by excluding those two from the empirical distribution, which, by the triangle inequality, changes the distance statistic by at most $\frac{1}{n}$.}
Therefore, assuming the underlying permutation is the identity, $(\mu_i, \nu_k)$ behave in distribution as two $n$-point empirical distributions
\begin{align}
\mu=\frac{1}{n} \sum_{j=1}^n \delta_{X_j}, \quad \nu=\frac{1}{n} \sum_{j=1}^n \delta_{Y_j}
\label{eq:mu_nu_def}
\end{align}
according to two cases:
\begin{itemize}
	\item For ``true pairs'' ($i=k$), the $X$ and $Y$ samples consist of independent correlated pairs, namely, 
	\begin{equation}
	    	(X_1,Y_1),\ldots,(X_n,Y_n) \iiddistr N \left(\zeros,\Big[\begin{smallmatrix} 1 & \rho\\\rho & 1  \end{smallmatrix} \Big] \right).	
	    	\label{eq:dicho1}
	\end{equation}
	\item For ``fake pairs'' ($i \neq k$), the $X$ and $Y$ sample are independent, namely,
	\begin{equation}
	    	(X_1,\ldots,X_n,Y_1,\ldots,Y_n) \iiddistr N(0,1).
	    	\label{eq:dicho2}
	\end{equation}
\end{itemize}
Therefore, although both empirical distributions have the same marginal distribution, for true pairs the atoms are correlated and the two empirical distributions tend to be closer than the typical distribution for fake pairs. This offers a test to distinguish true and fake pairs.

Now we introduce our procedure. For two probability measures $\mu$ and $\nu$, we define their distance via the $L_p$-distance between their cumulative distribution function (CDF) $F$ and $G$:
\begin{align}
d_p(\mu,\nu) \triangleq \|F - G\|_p = \left(\int_\reals dt |F(t)-G(t)|^p \right)^{1/p}, \label{eq:dp_distance}
\end{align} 
where $p \in [1,\infty]$  is some fixed constant.
e.g., 
\begin{itemize}
	\item $p=1$: 1-Wasserstein distance,
	\item $p=2$: Cram\'er-von Mises goodness of fit statistic,
	\item $p=\infty$: Kolmogorov-Smirnov distance; 
\end{itemize}
the asymptotic performance of the algorithm turns out to not depend on $p$.
For each vertex $i$, we match it to the vertex $k$ that minimizes the distance statistic $Z_{ik} \triangleq d_p(\mu_i,\nu_k)$. Next we show that when $\sigma \le \frac{c}{\log n}$ for sufficiently small constant $c$, this algorithm succeeds with high probability.

To this end, let us recall the central limit theorem of empirical processes (cf.~\cite{Shorack.Wellner}).
Let $F_n$ and $G_n$ denote the empirical CDF of $X_i$'s and $Y_i$'s, respectively, i.e.,
\begin{align*}
F_n(t)= \frac{1}{n} \sum_{i=1}^n \indc{X_i \leq t}, \quad G_n(t) = \frac{1}{n} \sum_{i=1}^n \indc{Y_i \leq t}.
\end{align*}
Let $\Phi$ denote the standard normal CDF on the real line. 
Then it is well-known that, as $n\to\infty$, $\sqrt{n}(F_n - \Phi)$ converges in distribution to a Gaussian process $\{B_t: t \in \reals\}$, with covariance function
$\Cov(B_s,B_t) = \min\{\Phi(s),\Phi(t)\} - \Phi(s) \Phi(t)$.
In fact, $B$ is a time change of the standard Brownian bridge, which is the limiting process if the samples are drawn from the uniform distribution on $[0,1]$.
Similarly, $\sqrt{n}(G_n - \Phi)$ converges in distribution to another Gaussian process  $B'$ with the same distribution as $B$.

Next we analyze the behavior of true pairs. 
To get a sense of the order of magnitude of the distance statistic, let us consider the special case of $p=2$ for convenience, for which direct calculation suffices. 
Define $F(s,t) = \prob{X \leq s,Y \leq t}$. 
Note that we can write $Y=\sqrt{1-\sigma^2} X + \sigma Z$, where $X,Z\iiddistr N(0,1)$.
Then
\begin{align}
\Expect[\|F_n-G_n\|_2^2]
= & ~ \int_{\reals} \Expect[(F_n(t)-G_n(t))^2] dt \nonumber \\
\stepa{=} & ~ \frac{2}{n} \int_{\reals} (F(t)-F(t,t)) dt \nonumber \\	
= & ~ \frac{2}{n} \left(  \int_{-\infty}^{0 }  (F(t)-F(t,t)) dt +  \int_{0}^{+\infty }  \left(  \left( 1-F(t,t) \right) - \left(1- F(t) \right)  \right)  dt  \right)\nonumber \\
\stepb{=} & ~ \frac{2}{n} \left( \expect{ \max (X, Y)} - \expect{ X} \right) \nonumber \\
= & ~ \frac{2}{n}  \expect{ \max (X, Y)} \nonumber \\
\stepc{=} & ~ \frac{2}{n} \frac{1}{\sqrt{\pi} } \sqrt{ 1- \rho} = \frac{2}{n} \frac{1}{\sqrt{\pi} } 
\underbrace{\sqrt{ 1- \sqrt{1-\sigma^2} }  }_{\Theta(\sigma)}, \label{eq:FnGn}
\end{align}
where (a) is due to $ \Expect[(F_n(t)-G_n(t))^2] = \frac{1}{n} \Expect[(\indc{X \leq t} - \indc{Y \leq t})^2] = \prob{X \leq t}+\prob{Y \leq t}-2\prob{X \leq t,Y\leq t}$; 
(b) follows because $\expect{U} =\int_{0}^{+\infty }  \left( 1- F_U(u) \right) du - \int_{-\infty}^0 F_U(u) du$ for any
random variable $U$ whenever at least one of the two integrals is finite; (c) follows from directly differentiating 
the moment generating function of $\max(X,Y)$, see \eg, \cite[Eq.~(9)]{nadarajah2008exact}.
In fact, one can show that for small $\sigma$, for any $1\leq p\leq\infty$, 
\begin{equation}
\|F_n-G_n\|_p = O_P\pth{\sqrt{\frac{\sigma}{n}}}.
\label{eq:gaussian-true}
\end{equation}
Indeed, by the central limit theorem for bivariate empirical processes, as $n\diverge$, $\sqrt{n}(F_n-\Phi,G_n-\Phi)$ converges in distribution to a Gaussian process $(B,B')$ indexed by $\reals$, which satisfies 
$\Cov(B_t,B_t') = \prob{X \leq t,Y\leq t} - \prob{X \leq t} \prob{Y\leq t}$, and furthermore  $\sqrt{n} \|F_n-G_n\|_p \to \|B-B'\|_p$  in distribution.
Since $\Expect|B_t-B'_t|^2 = 2 (\Phi(t) - \prob{X \leq t,Y\leq t})$, following the same calculation that leads to \prettyref{eq:FnGn}, we have $\Expect[\|B-B'\|_2^2] = \Theta(\sigma)$, which corresponds to \prettyref{eq:gaussian-true} for $p=2$.

Next, we turn to the behavior of fake pairs. Since $B$ and $B'$ are independent and since $B - B' \overset{\text{law}}{=} \sqrt{2} B$, we expect
$\sqrt{n}\|F_n - G_n\|_p \to \|B-B'\|_p$ (see \cite[Theorem 1.1]{dBGM99} for the precise statement). In particular, we have
\begin{equation}
\|F_n - G_n\|_p = \Theta_P\pth{\frac{1}{\sqrt{n}}}.
\label{eq:gaussian-fake}
\end{equation}
Comparing \prettyref{eq:gaussian-true} and \prettyref{eq:gaussian-fake}, we see that the typical distance for true pairs is smaller than that of fake pairs by a factor of $\sqrt{\sigma}$. However, since there are $n-1$ wrong matches for a given vertex, 
\nb{we need to consider the large-deviation behavior of \prettyref{eq:gaussian-fake}}. 
Recall the classical result from the literature of small ball probability; see \cite{LS01} for an excellent survey.
Let $B$ be some Gaussian process e.g. the Brownian bridge defined on $\reals$. 
\nb{Then the probability for the process to be contained in a small ball of radius $\epsilon$ behaves as (cf.~\cite[Sec.~4 and 6.2]{LS01})
\begin{equation}
\prob{\|B\|_p \leq \epsilon} \leq \exp\pth{- \Theta\pth{\frac{1}{\epsilon^2}}}
\label{eq:smallball}
\end{equation}
for some constant $C$.} Indeed, one can show that 
\begin{equation}
\prob{\|F_n - G_n\|_p \leq \sqrt{\frac{\sigma}{n}}} \leq \exp\pth{- \Theta\pth{\frac{1}{\sigma}}}.
\label{eq:smallball1}
\end{equation}
Setting this probability to $o(\frac{1}{n^2})$ and applying a union bound, we conclude that the matching algorithm succeeds with high probability if $\sigma \le \frac{c}{\log n}$ for sufficiently small constant $c$.

\subsection{Improvement with seeded matching} 
\label{sec:gaussian-imp}

In this subsection we improve the previous matching algorithm with empirical distributions to $\sigma = 
O((\log n)^{-1/3})$. To this end, we turn to the idea of \emph{seeded matching}. Given a partial permutation that gives the correct matching for a subset of vertices, which we call \emph{seeds}, one can extend it to a full matching by various methods, e.g., by solving a bipartite matching (see \prettyref{alg:seed}). It turns out for Wigner matrices, it suffices to obtain $\Omega(\log n)$ seeds, which can be found by combining both the distance-based matching and degree thresholding. 
The same idea applies to \ER graphs, except that for edge density $p$, the number of seeds needed is $\Omega(\frac{\log n}{p})$, a fact which will be exploited in \prettyref{sec:distdeg}.

To explain the main idea, let $a_i = \frac{1}{\sqrt{n}}\sum_{j=1}^n A_{ij}$ and $b_k = \frac{1}{\sqrt{n}} \sum_{j=1}^n B_{kj}$ be the standardized row sums, which are the counterparts of ``degrees'' for Gaussian matrices. Consider the set of pairs $(i,k)$ such that both $a_i$ and $b_k$ exceed some threshold $\xi$.
Then for any fake pair $i\neq k$, by independence, we have 
\[
\prob{a_i \geq \xi, b_k \geq \xi} =\prob{a_i \geq \xi} \prob{b_k \geq \xi}  = Q(\xi)^2,
\]
 where $Q\triangleq 1-\Phi$ is the complementary CDF for the standard normal distribution. For true pairs, 
since we have the representation 
\begin{equation}
b_i=\sqrt{1-\sigma^2} a_i + \sigma z_i,
\label{eq:abz}
\end{equation}
 where $a_i,z_i\iiddistr N(0,1)$, we have
\[
\prob{a_i \geq \xi, b_i \geq \xi} \geq \prob{a_i \geq \frac{\xi}{\sqrt{1-\sigma^2}}, z_i \geq 0} \geq \frac{1}{2} Q\pth{\frac{\xi}{\sqrt{1-\sigma^2}}}
\geq Q(\xi ) \exp(- O(\sigma^2 \xi^2)).
\]
\nb{Now let us consider the seed set consisting of those high-degree pairs $i$ and $k$ whose empirical distributions satisfy $d_p(\mu_i,\nu_k) \lesssim \sqrt{\frac{\sigma}{n}}$. 
Thus to create enough seeds, we need 
\begin{align}
n Q(\xi ) \exp\left(- O(\sigma^2 \xi^2) \right) \geq \Omega(\log n), \label{eq:Gaussian_seed_cond}
\end{align}
and to eliminate all fake pairs we need (in view of the small-ball estimate \prettyref{eq:smallball1})
\begin{align}
    n^2 Q(\xi )^2 \exp \left(- \Omega \left(\sigma^{-1} \right) \right) = o(1). \label{eq:Gaussian_fake_cond}
\end{align}}
Choosing $\xi =\Theta(\sqrt{\log n})$ and substituting it into \prettyref{eq:Gaussian_seed_cond}, we get that $Q(\xi) =
\Omega \left( \frac{\log n}{n}\right) \exp(O(\sigma^2 \log n))$. Substituting 
this back into \prettyref{eq:Gaussian_fake_cond}, we conclude that 
$\sigma \le \frac{c}{(\log n)^{1/3}}$ for some small constant $c$ suffices.

We end this section with a few remarks:

\begin{remark}[Order statistics]
\label{rmk:orderstat}	
As described in \prettyref{sec:idea}, degree matching fails unless 
the fraction of differed edges is polynomially small. 
Similarly, for the Gaussian model directly sorting the degrees (row sums) in both matrices fails to  yield the correct matching unless $\sigma \leq n^{-c}$ for some constant $c$. Indeed, sort the row sums $a_i$'s decreasingly as $a_{(1)} \geq \ldots \geq a_{(n)}$ and similarly for $b_{(1)} \geq \ldots \geq b_{(n)}$. Thus, degree matching amounts to match the vertices according to the sorted degrees. 
Since $a_i$'s are iid standard normal, it is well-known from the extreme value theory \cite{DN70} that, with high probability, the order statistics behaves approximately as $a_{(i)} \approx \Phi^{-1}(i/n)$ which is approximately $\sqrt{2 \log \frac{n}{i}}$ for $i \le n/2$ and $-\sqrt{2 \log \frac{n}{n+1-i}}$ for $i \ge n/2$. In particular, $a_{(1)} = a_{\max} \approx \sqrt{2 \log n}$ and 
$a_{(n)} = a_{\min} \approx -\sqrt{2 \log n}$.
Furthermore, the $i$th spacing of the order statistics is approximately
\begin{equation}
\sqrt{2 \log \frac{n}{i}} - \sqrt{2 \log \frac{n}{i+1}} = \Theta\pth{ \frac{1}{i \sqrt{\log \frac{n}{i}}} }    
\label{eq:orderstat-gap}
\end{equation}
Therefore, and intuitively so, for most of the samples the spacing is as small as $\Theta(\frac{1}{n})$. In view of \prettyref{eq:abz}, we can write 
$b_i = a_i + \Delta_i$, where $\Delta_i = (\sqrt{1-\sigma^2}-1) a_i + \sigma z_i$.
Thus degree matching succeeds if $|\Delta_i| \leq \min\{|a_{i-1}-a_i|, |a_i-a_{i+1}|\}$ for all $i$. 
Since $|z_i| \leq O(\sqrt{\log n})$ and $|a_i|  \leq O(\sqrt{\log n})$ for all $i$ with high probability, this shows that degree matching requires very small noise $\sigma = o(\frac{1}{n \sqrt{\log n}})$, which is much worse than degree profiles. Simulation shows that this condition is necessary up to logarithmic factors. 

Following the same idea in this subsection, an immediate improvement is to use degree matching to produce enough seeds to initiate the seeded graph matching process. Indeed, this is possible because the spacing of the first few order statistics is much bigger and more robust to noise. 
More precisely, 
in order to produce $\Omega(\log n)$ seeds, it suffices to ensure that the \textit{minimum spacing} of the first $i$ order statistics, which is at least $\tilde{\Omega}(\frac{1}{i^2 \sqrt{\log n}})$, far exceeds the noise which is $ O(\sigma \sqrt{\log n})$. With $i=\Theta(\log n)$, this translates to $\sigma = o(\frac{1}{(\log n)^4})$, which is comparable to but still worse  than the guarantee of degree profiles of $\sigma = O(\frac{1}{\log n})$ 
\nb{as established in \prettyref{sec:gaussian-emp}}. 
More importantly, a fundamental limitation of degree matching is that it fails for sparse graphs, because the number of seeds needed is $\Omega(\frac{\log n}{q})$ where $q$ is the edge density of the observed graphs \nb{(cf.~\prettyref{lmm:seed} and \cite[Theorem 1]{korula2014efficient})}. Following the similar analysis above for binomial distribution, for the correlated \ER graph model $\calG(n,q;1-\delp)$, 
it is well-known that (cf.~\cite[Theorem 3.15]{bollobas1998random})
the minimum of the first $i$ spacing of sorted degrees is $\tilde{\Omega}(\frac{\sqrt{n q}}{i^2})$ with high probability 
and \nb{the degrees of a true pair differ by at most $\tilde{O}( \sqrt{\delp nq})$}. Thus, producing $\Omega(\frac{\log n}{q})$ seeds requires the deletion probability to be as small as $\delp = \tilde o(q^4)$. This explains the recent result of \cite{dai2018performance}, which shows that degree-matching algorithm with seeded improvement succeeds under some extra conditions.
\end{remark}

\begin{remark}[From Gaussian matrices to \ER graphs]
\label{rmk:GB}	
To extend the matching algorithm based on empirical distributions from Gaussian matrices to \ER graphs, the main difficulty is that Bernoulli random variables are zero-one valued and hence directly
implementing the same empirical distribution matching algorithm using adjacency matrices does not work. As mentioned in \prettyref{sec:idea}, the idea is to use the {degree profile} of each vertex, that is, the empirical distribution of the degrees of the neighbors, each of which is binomially distributed and well-approximated by Gaussians. 
Indeed, the ideas in \prettyref{sec:gaussian-emp} and \prettyref{sec:gaussian-imp} lead to \prettyref{alg:dist} and \prettyref{alg:distdeg}, respectively, for \ER graphs.
However, the major technical difficulty is to address the dependency in the degree profiles. 
\nb{In the Gaussian case, each pair of degree profiles follows the simple dichotomy in \prettyref{eq:dicho1}--\prettyref{eq:dicho2}, behaving as a pair of empirical distributions of correlated (resp.~independent) samples for  true (resp.~fake) pairs.
This is no longer the case for \ER graphs}. 
For this reason, the approach for  \ER graphs deviates from the program for Gaussian matrices, in that
the algorithms in \prettyref{sec:pf} are based on a quantized version of the total variation distance as opposed to distances between empirical CDFs, and the analysis in \prettyref{sec:pf} 
does not explicitly resort to empirical process theory, although it is still guided by similar intuitions.
\end{remark}

\section{Matching algorithms for correlated \ER graphs}
\label{sec:alg}

\subsection{Preliminary definitions}
For each vertex $i$, define its \emph{open} neighborhood $N_A(i)$ (resp.~$N_B(i)$) in graph 
$A$ (resp.~$B$) as the set of vertices connecting to $i$ by an edge in $A$ (resp.~$B$);
define its \emph{closed} neighborhood $N_A[i]$ (resp.~$N_B[i]$) in graph 
$A$ (resp.~$B$) as the union of its open neighborhood in $A$ (resp.~$B$) and $\{i\}$.

Denote the degrees by 
\begin{align}
a_i = & ~ |N_A(i)| = \sum_{j \in [n]} A_{ij}  \label{eq:deg1}\\
b_i = & ~ |N_B(i)| = \sum_{j \in [n]} B_{ij}\label{eq:deg2}.
\end{align}
For each $i$ and $j$, define
\begin{align}
a_j^{(i)} = & ~ \frac{1}{\sqrt{(n-a_i-1) q (1-q)}} \sum_{\ell  \notin N_A[i]} (A_{\ell j}-q) \label{eq:degmod1}\\
b_j^{(i)} = & ~ \frac{1}{\sqrt{(n-b_i-1) q (1-q)}} \sum_{\ell  \notin N_B[i]} (B_{\ell j}-q),\label{eq:degmod2}
\end{align}
Note that $a_j^{(i)}$ (resp.~$b_j^{(i)}$) can be viewed as the standardized version of the ``outdegree'' of 
vertex $j$ by excluding $i$'s closed neighborhood in $A$ (resp.~$B$). 

To each vertex $i$ in $A$, attach a distribution which is the empirical distribution of the set $\{a_j^{(i)}: j \in N_A(i) \}$:
\begin{equation}
\mu_i \triangleq  \frac{1}{a_i} \sum_{j \in N_A(i)} \delta_{a_j^{(i)}},
\label{eq:degprofile1}
\end{equation}
and the centered version  (viewed as a signed measure)
\begin{equation}
\bar \mu_i \triangleq  \mu_i  - \Binomc(n-a_i-1,q),
\label{eq:degprofile1c}
\end{equation}
where $\Binomc(k,q)$ denotes the standardized binomial distribution, that is, the law of $\frac{X-kq}{\sqrt{k q (1-q)}}$ for $X\sim \Binom(k,q)$.
\nb{The centering in \prettyref{eq:degprofile1c} is due to the fact that conditioned on the neighborhood $N_A(i)$, each $a_i^{(j)}$ is distributed as $\Binomc(n-a_i-1,q)$
marginally.}
Similarly, for $B$ we define
\begin{equation}
\nu_i \triangleq  \frac{1}{b_i} \sum_{j \in N_B(i)} \delta_{b_j^{(i)}}.
\label{eq:degprofile2}
\end{equation}
and the centered version
\begin{equation}
\bar \nu_i \triangleq  \nu_i - \Binomc(n-b_i-1,q).
\label{eq:degprofile2c}
\end{equation}
Intuitively, $\mu_i$ is the degree profile for the neighbors of $i$ in $A$, if the summation in \prettyref{eq:degmod1} is over all $[n]$. We exclude edges within the neighborhood itself to reduce dependency and simplify the analysis.
Note that conditioned on $N_A(i)$,  $\{a_j^{(i)}: j \in N_A(i)\}$ are iid as $\Binomc(n-a_i-1,q)$; 
conditioned on $N_B(i)$,  $\{b_j^{(i)}: j \in N_B(i)\}$ are iid as $\Binomc(n-b_i-1,q)$.

Fix $L\in\naturals$  to be specified later.
Define $I_1, \ldots, I_L$ as the uniform partition of $[-1/2,1/2]$ such that $|I_\ell| = 1/L$. 
For each $i$ and $k$, define the following distance statistic:
\begin{equation}
Z_{ik} \triangleq \sum_{\ell \in [L]}  |\bar \mu_i(I_\ell) - \bar \nu_k(I_\ell)|.
\label{eq:Z}
\end{equation}
In other words, 
\begin{equation}
    Z_{ik} = d(\bar \mu_i, \bar \nu_k) \triangleq \|[\bar \mu_i]_L - [\bar \nu_k]_L \|_1,
    \label{eq:distance}
\end{equation}
 where $[\mu]_L$ denotes the discretized version of $\mu$ according to the partition $I_1,\ldots,I_L$, with
\begin{equation}
[\mu]_L(\ell) \triangleq \mu(I_\ell), \quad \ell \in [L].
\label{eq:discretization}
\end{equation}
Throughout the rest of the paper, for simplicity we use the parameterization 
\begin{equation}
s  \triangleq 1-\sigma^2,  \quad \delp \triangleq \sigma^2
\label{eq:sigmadelta-def}
\end{equation}
 to denote the sampling and deletion probability respectively, where $\sigma$ corresponds to the magnitude of the ``effective noise''.

\subsection{Matching via degree profiles}
	\label{sec:dist}
We present our first algorithm which matches the vertices in $A$
to vertices in $B$ based on the pairwise distance statistic $\{Z_{ik}\}$
in \prettyref{eq:Z}.
\begin{algorithm}
\caption{Graph matching via degree profiles}\label{alg:dist}
\begin{algorithmic}[1]
\STATE {\bfseries Input:} Graphs $A$ and $B$ on $n$ vertices, an integer $L$.
\STATE {\bfseries Output:} A permutation $\hat\pi \in S_n$.
\STATE  For each $i,k \in [n]$, compute $Z_{ik}$ in \prettyref{eq:Z}.
\STATE Sort $\{Z_{ik}: i,k\in[n]\}$ and let $\calS$ be the set of indices of the smallest $n$ elements.
\IF{$\calS$ defines a perfect matching on $[n]$, i.e., $\calS = \{(i,\hat\pi(i)): i \in [n]\}$ for some permutation $\hat\pi$}
\STATE Output $\hat\pi$;
\ELSE 
\STATE Output error.
\ENDIF
\end{algorithmic}
\end{algorithm}

The key  intuition underlying  \prettyref{alg:dist} is as follows:
\begin{itemize}
    \item For true pairs $k=\pi^*(i)$, we expect $i$ and $k$ to share many (about $nqs$) ``common neighbors'' $j$, in the sense that $j$ is $i$'s neighbor in
    $A$ and $\pi^*(j)$ is $k$'s neighbor in $B$.  For each such common neighbor
    $j$, its outdegree $a_j^{(i)}$ in $A$ is statistically correlated with 
    the outdegree $b_{\pi^*(j)}^{(k)}$ in $B$. As a consequence, the 
    two empirical distributions are strongly correlated, leading to a small distance $Z_{ik}$.
    \item For wrong pairs $k\neq\pi^*(i)$, we expect $i$ and $k$ share very few
    (about $nq^2$) ``common neighbors''. Hence, the two empirical distributions
    $\mu_i$ and $\nu_k$ are weakly correlated, leading to a large distance $Z_{ik}$. 
    \end{itemize}

\begin{remark}[Time complexity]
 Implementing \prettyref{alg:dist} entails three steps. 
 First, we precompute all outdegrees. Assuming the graph is represented as an adjacency list and the list of degrees are given, for each $i$ and each $j \in N_A(i)$, we have $a_j^{(i)} = a_j - 1 - |N_A(i)\cap N_A(j)|$, where 
$a_j$ is the degree of $j$ and $|N_A(i)\cap N_A(j)|$ is the number of common neighbors, which can be computed in $\title{O}(a_i+a_j)$ time. Thus, computing all outdegrees can be done in time that is $\sum_{i \sim j} \title{O}(a_i+a_j) = \title{O}(\sum_i a_i^2) = \title{O}(|E| |d_{\max}|)$.\footnote{Alternatively, outdegrees can be computed via the number of common neighbors by squaring the adjacency matrix using fast matrix multiplication.}
Next, we compute the discretized and centered degree profiles $[\bar \mu_i]_L$ for each $i$ in graph $A$ and $[\bar \nu_k]_L$ for each $k$ in graph $B$. 
These are identified as $L$-dimensional vectors (where $L=\polylog(n)$) and can be done in $\tilde{O}(|E|)$ time.
Finally, we compute the distance statistic $Z_{ik}$ in \prettyref{eq:Z} for all pairs $i$ and $k$ and implement greedy matching via sorting. Since $Z_{ik}$ is the $\ell_1$-distance between two $L$-dimensional vectors, this step can be computed in a total of $\tilde{O}(n^2)$ time. 
In summary, the total time complexity of \prettyref{alg:dist} is at most $\tilde{O}(|E| |d_{\max}| + |V|^2)$, which, for \ER graphs under the assumption of \prettyref{thm:main}, reduces $\tilde{O}(n^3 q^2+n^2)$.
 

The reason we use outdegrees instead of degrees in 
\prettyref{alg:dist} is a technical one, which aims at reducing the dependency and facilitating the theoretical analysis. In practice we can use degree profiles defined through the usual degrees and empirically the algorithm performs equally well. In this case, the time complexity reduces to $\tilde{O}(n^2)$.
\end{remark}

\begin{theorem}[Performance guarantee of \prettyref{alg:dist}]
\label{thm:guarantee-distance}
Let $s=1-\sigma^2$ and $q \le q_0$ for some sufficiently small positive constant $q_0$.
Assume that 
\begin{equation}
\sigma \leq \frac{\sigma_0}{\log n},
\label{eq:cond-main1}
\end{equation}
for some sufficiently small absolute
constant $\sigma_0$. 
Set 
\begin{equation}
L = L_0 \log n
\label{eq:L}
\end{equation}
and assume that 
\begin{equation}
n q \geq C_0 \log^2 n
\label{eq:cond-main2}
\end{equation}
for some large absolute constants $L_0, C_0$.
Then with probability $1-O(1/n)$, \prettyref{alg:dist} outputs $\hat\pi=\pi^*$.
\end{theorem}

\subsection{Dense graphs: Combining with high-degree vertices}
\label{sec:dense}

For relatively dense graphs, \prettyref{alg:dist} can be improved as follows.
Recall the notion of seeded graph matching previously mentioned in \prettyref{sec:gaussian-imp}, where a number of correctly matched vertices are given, known as seeds, and the goal is to match the remaining vertices. It turns out that for $G(n,q)$, 
\nb{provided $m = \Omega(\frac{\log n}{q})$ seeds}, solving a linear assignment problem (maximum bipartite matching) can successfully match the rest of the vertices with high probability. 
Note that the condition $\sigma = O((\log n)^{-1})$ in \prettyref{thm:guarantee-distance} ensures \prettyref{alg:dist} succeeds in one shot, in the sense that with high probability the distance statistics are below the threshold for \emph{all} $n$ true pairs and 
above the threshold for \emph{all} $\binom{n}{2}$ wrong pairs.
Thus, we can weaken this condition so that even if the distance statistics for most of the pairs are not correctly separated, those high-degree vertices can provide enough seeds that allow bipartite matching to succeed.
This idea leads to the improvement to $\sigma = O((\log n)^{-1/3})$ when the edge density 
$q = \exp(-O((\log n)^{1/3}))$.

Specifically, fix some thresholds $\tau$ and $\xi$. Consider the collection of pairs of vertices 
whose degrees are atypically high and the degree profiles are close:
\begin{equation}
\calS = \{(i,k): a_i \geq \tau, b_k \geq \tau+1, Z_{ik} \leq \xi \}.
\label{eq:seed}
\end{equation}
We show that, with high probability,
\begin{enumerate}
	\item $\calS$ does not contain any fake pairs, i.e., $(i,k) \not\in \calS$ for any $k\neq \pi^*(i)$.
	\item $\calS$ contain enough true pairs, i.e., $|\calS| = \Omega(\frac{\log n}{q})$.	
\end{enumerate}
Finally, we use the matched pairs in $\calS$ as seeds to resolve the rest of the matching by linear assignment; this is done in \prettyref{alg:seed}. 
The full procedure is given in \prettyref{alg:distdeg}.

As for the time complexity, 
compared to \prettyref{alg:dist},  \prettyref{alg:distdeg} has an extra step of computing the maximum matching on an $n\times n$ unweighted bipartite graph, which can be done in either $O(n^3)$ time using Ford--Fulkerson algorithm~\cite{ford1956maximal} or $O(n^{2.5})$ time using the Hopcroft--Karp algorithm~\cite{Hopcroft1971}. 

\begin{algorithm}
\caption{Combining degree profiles and large-degree vertices}\label{alg:distdeg}
\begin{algorithmic}[1]
\STATE {\bfseries Input:} Graph $A$ and $B$ on $n$ vertices; thresholds $\tau,\xi > 0$.
\STATE {\bfseries Output:} A permutation $\hat\pi \in S_n$.
\STATE Compute the distance statistic $Z_{ik}$ for each $i,k\in[n]$. Let $\calS$ be given in \prettyref{eq:seed}.
\IF{$\calS$ defines a matching, i.e., there exists $S\subset [n]$ and an injection $\pi_0: S \to [n]$, such that $\calS = \{(i,\pi_0(i)): i \in S\}$,}
\STATE Run \prettyref{alg:seed} using $\pi_0$ as the seeds and output $\hat\pi$. 
\ELSE 
\STATE output error;
\ENDIF
\end{algorithmic}
\end{algorithm}

\begin{algorithm}
\caption{Seeded graph matching} \label{alg:seed}
\begin{algorithmic}[1]
\STATE {\bfseries Input:} Graphs $A$ and $B$ on $n$ vertices; a bijection $\pi_0: S \to T$, where $S,T \subset [n]$;
\STATE {\bfseries Output:} A permutation $\hat\pi \in S_n$.
\STATE For each $i\in S^c$ and each $k\in T^c$, define $n_{ik} = \sum_{ j \in S} A_{ij} B_{k \pi_0(j)}$.
\STATE Define a bipartite graph with vertex set $S^c \times T^c$ and 
adjacency matrix $H$ given by $H_{ik} = \indc{n_{ik} \ge \kappa}$
for each $i \in S^c $ and each $k \in T^c$, where \nb{$\kappa = \frac{1}{2}|S| qs$}. 
Find a maximum bipartite matching of $H$, \ie, a perfect matching $\tilde{\pi}_1$
between $S^c$ and $T^c$ such that
$
\tilde{\pi}_1 \in \arg \max_{ \pi }  \;  w(\pi),
$
where
\begin{equation}
w(\pi)  \triangleq \sum_{i \in S^c } H_{i \pi(i) }.
\label{eq:weightw}
\end{equation}
Let $\pi_1$ denote a perfect matching on $[n]$ such that $\pi_1 \vert_S=\pi_0$ and $\pi_1 \vert_{S^{c}} = \tilde{\pi}_1$. 
\STATE For each $i,k \in [n]$, define 
$w_{ik} = \sum_{j=1}^n A_{ij} B_{k \pi_1(j)}$. 
\STATE Sort $\{w_{ik}: i,k \in [n] \}$ and let $\calT$ be the set of indices of the largest $n$ elements.
\IF{$\calT$ defines a perfect matching on $[n]$, i.e., $\calT=\{(i,\hat{\pi}(i)): i \in [n]\}$
for some permutation $\hat{\pi}$}
\STATE Output $\hat\pi$;
\ELSE 
\STATE Output error;
\ENDIF
\end{algorithmic}
\end{algorithm}

\begin{theorem}[Performance guarantee of \prettyref{alg:distdeg}]
\label{thm:guarantee-distance-deg}
Assume that $q \le q_0$ and 
\begin{equation}
\sigma \leq \sigma_0 \min\sth{\frac{1}{(\log n)^{1/3}}, \frac{1}{\log \frac{\log n}{q}}},
\label{eq:cond-main1_relax}
\end{equation}
for some small absolute
constants $q_0,\sigma_0$. 
Define
 \begin{equation}
\alpha \triangleq \left( \alpha_0 \frac{\log n}{nq} \right)^{ \frac{(1-p)s}{1-q} }
 \label{eq:alpha}
 \end{equation}
and 
\begin{equation}
L = L_0 \max \left\{  \log^{1/3} (n), \log \frac{\log n}{q} \right\}
\label{eq:L_relax}
\end{equation}
for some large absolute constants $\alpha_0, L_0$.
Let 
\begin{equation}
\tau \triangleq \min \left\{  0 \le k \le n:  
\prob{ \Binom (n-1 ,q ) \ge k } \le \alpha \right\},
\label{eq:tau}
\end{equation}
and 
\begin{equation}
\xi = C \sqrt{\frac{L}{nq}}
\label{eq:xi}
\end{equation}
for some absolute constant $C$.
Assume that 
\begin{equation}
n q^2 \geq C_0 \log^2 n
\label{eq:cond-main2_relax}
\end{equation}
for some large absolute constant $C_0$.
Then with probability $1-O\left( \frac{q}{\log n} \right)$, 
\prettyref{alg:distdeg} outputs $\hat\pi=\pi^*$.
\end{theorem}


We briefly explain the choice of parameters and the condition \prettyref{eq:cond-main1_relax} on $\sigma$. 
According to \prettyref{eq:tau}, the threshold $\tau$ is chosen to be the $(1-\alpha)$-quantile of $a_i$, so that $\prob{a_i \geq \tau} \approx \alpha$.
The crucial observation is the following:
\begin{itemize}
	\item For true pairs $k=\pi^*(i)$, the degrees $a_i$ and $b_k$ are both sampled from the same vertex in the parent graph and are hence positively correlated. Indeed, we have
\begin{align}
\prob{a_i \geq \tau , b_k \geq \tau+1}  
= \Omega\left( \alpha^{ \frac{1-q}{(1-p)s} } \right), \quad 
k = \pi^*(i). \label{eq:deg_correlation0}
\end{align}
Here the exponent $\frac{1-q}{(1-p)s}$ is slightly bigger than one:
\begin{align}
\frac{1-q}{(1-p)s} = 1+ \frac{1-s}{(1-p)s} = 1+ \frac{\sigma^2}{(1-p)s}.
\label{eq:exponent_break_1}
\end{align}
\item For fake pairs $k \neq \pi^*(i)$, the degrees $a_i$ and $b_k$ are almost independent, and indeed we have  
\begin{align}
\prob{a_i \geq \tau , b_k \geq \tau+1}  
= O\left( \alpha^2 \right), \quad k \neq \pi^*(i). 
\label{eq:deg_correlation1}
\end{align}
\end{itemize}
Both \prettyref{eq:deg_correlation0} and \prettyref{eq:deg_correlation1} will be made precise in \prettyref{lmm:degcorr}.

In order for \prettyref{alg:distdeg} to succeed, on the one hand, 
we need to ensure the seed set $\calS$ in \prettyref{alg:distdeg}  
contains at least $\Omega(\frac{\log n}{q})$ correctly matched pairs. Indeed, 
under the condition $L=O(1/\sigma)$ and the choice 
of $\xi$ in \prettyref{eq:xi}, we will show 
that for any true pair $(i,k)$ the distance statistic $Z_{ik}$ is below $\xi$ with high probability. Thus, we have in expectation:
 $$
 \Expect[|\calS|] \overset{\prettyref{eq:deg_correlation0}}{\geq} n \alpha^{ \frac{1-q}{(1-p)s} } \overset{\prettyref{eq:alpha}}{=} \alpha_0\frac{\log n}{q},
 $$
and we will show that this holds with high probability as well.

On the other hand, we need to ensure that no 
 fake pair is included in
$\calS$ with high probability. 
We will show that for any wrong pair $(i,k)$, $Z_{ik} \leq \xi$ with probability at most $e^{-\Omega(L)}$ (see \prettyref{lmm:fake}).
By the union bound, in view of \prettyref{eq:deg_correlation1}, it suffices to guarantee that 
\begin{align}
n^2 \alpha^2 \exp \left( - \Omega(L) \right)
\overset{\prettyref{eq:alpha}}{=} & ~ n^2 \left( \alpha_0 \frac{\log n}{nq} \right)^{ 2 \frac{(1-p)s}{1-q} } \exp \left( - \Omega(L) \right)	\nonumber \\
\overset{\prettyref{eq:exponent_break_1}}{\le}  & ~ \left(\alpha_0 \frac{\log n}{q} \right)^{2}  
 \exp \left( \frac{2 \sigma^2}{1-q} \log n - \Omega(L) \right) =o(1). \label{eq:nalpha}
\end{align}
Also, recall that $L=O(1/\sigma)$. Thus, the desired \prettyref{eq:nalpha} holds provided that $\sigma \lesssim \frac{1}{(\log n)^{1/3}} \wedge \frac{1}{\log \frac{\log n}{q}}$,
 and $q$ is bounded away from $1$, by choosing 
 $ L \gtrsim (\log n)^{1/3} \vee \log \frac{\log n}{q}$.
 
 Finally, we mention that since the seed set obtained from \prettyref{alg:dist} and degree thresholding depends on the entire graph, the analysis of \prettyref{alg:distdeg} entails a worst-case analysis of the seeded matching subroutine. This is done in \prettyref{lmm:seed_matching_mm}
 in \prettyref{app:seed}, which guarantees the correctness of \prettyref{alg:seed} even for an adversarially chosen seed set.

 
\subsection{Sparse graphs: Matching via neighbors' degree profiles}
\label{sec:sparse}
For relatively sparse graphs, we can
improve the condition from $\sigma=O(1/\log(n))$ 
to  $ \sigma =O(1/\log(nq))$ by comparing 
neighbors' degree profiles. 
Next we describe our improved local algorithm, which uses the information of $3$-hop neighborhoods.

We start with some basic definitions.
The $\ell$-hop neighborhood of $i$ in graph $G$ is the subgraph of $G$ induced by the vertices within distance $\ell$ from $i$. Let $\tilde{N}_A(i)$ (resp.~$\tilde{N}_B(i)$) 
denote the set of vertices in the $2$-hop neighborhood of $i$ in graph $A$ (resp.~$B$). Denote the size of the $2$-hop neighborhood of $i$ in 
graph $A$ and $B$ by respectively
\begin{align*}
 \tilde{a}_i = | \tilde{N}_A(i)|, \quad \text{ and } \quad \tilde{b}_i = | \tilde{N}_B(i) |. 
\end{align*}
 
For each vertex $i$ and each
vertex $\ell$ at distance  two from $i$ in graph $A$ (resp.~$B$), define $\tilde{a}_{\ell}^{(i)}$ (resp.~$\tilde{b}_\ell^{(i)}$) as 
\begin{align}
\tilde{a}_{\ell}^{(i)} & = \frac{1}{ \sqrt{ ( n- \tilde{a}_i ) q (1-q)} } 
\sum_{k \notin \tilde{N}_A(i) } \left( A_{k \ell} -q \right), \label{eq:outdegree2hop1}\\
\tilde{b}_\ell^{(i)}
&= \frac{1}{ \sqrt{ ( n- \tilde{b}_i) q (1-q)} } 
\sum_{k \notin \tilde{N}_B(i) } \left( B_{k \ell} -q \right),
\label{eq:outdegree2hop2}
\end{align}
Analogous to \prettyref{eq:degmod1} and \prettyref{eq:degmod2}, $\tilde{a}_\ell^{(i)}$ (resp.~$\tilde{b}_\ell^{(i)}$) can also be viewed as the normalized ``outdegree'' of  vertex $\ell$, this time with the closed $2$-hop neighborhood of $i$ in $A$ (resp.~$B$) excluded. 

To each vertex $j \in N_A(i)$, attach the centered empirical distribution of the set $\{ \tilde{a}_\ell^{(i)}:  \ell \in N_A(j) \setminus N_A[i] \}$:
\begin{equation}
\tilde{\mu}^{(i)}_j \triangleq  \frac{1}{\left|  N_A(j) \setminus N_A[i] \right|} \sum_{ \ell \in N_A(j) \setminus N_A[i] } 
\delta_{ \tilde{a}_\ell^{(i)}} \; - \Binomc\left(n-\tilde{a}_i,q \right).
\end{equation}
Similarly, to each vertex $j \in N_B(i)$, attach the 
centered empirical distribution of the set $\{ \tilde{b}_\ell^{(i)}:  \ell \in N_B(i) \setminus N_B[i] \}$:
 \begin{equation}
\tilde{\nu}^{(i)}_{j} \triangleq  \frac{1}{ \left|N_B(j) \setminus N_B[i]  \right| } \sum_{ \ell \in N_B(j) \setminus N_B[i] } 
\delta_{ \tilde{b}_\ell^{(i)}} \; - \Binomc\left(n-\tilde{b}_i,q \right).
\end{equation}
Analogous to \prettyref{eq:degprofile1c}
and \prettyref{eq:degprofile2c}, $\tilde{\mu}_j$ (resp.~$\tilde{\nu}_j $) is the centered ``outdegree'' profile of $j$, this time defined over only $j$'s neighbors which are at exactly distance two from $i$ in $A$ (resp.~$B$). 

We now introduce a new distance statistic $W$ based on aggregating the original $Z$ statistic in \prettyref{eq:Z} over neighbors.
Recall the uniform partition $I_1,\ldots, I_L$ of $[-1/2,1/2]$ such that $|I_\ell|=1/L$. 
For each $j \in N_A(i)$ and $j' \in N_B(k)$, define the following distance statistic:
\begin{align}
\tilde{Z}^{(ik)}_{jj'} \triangleq \sum_{\ell \in [L]}  \left|\tilde{\mu}^{(i)}_j (I_\ell) - \tilde{\nu}^{(k)}_{j'} (I_\ell) \right|, \label{eq:Z_3_hop}
\end{align}
which is analogous to \prettyref{eq:Z} except that the definition of the outdegrees are modified.

For each $i,k \in [n]$, construct a bipartite graph with vertex set $N_A(i) \times N_B(k)$, whose adjacency matrix $Y^{(ik)}$ is given by 
\begin{align}
Y^{(ik)}_{jj'}=
\indc{\tilde{Z}^{(ik)}_{jj'} \le \eta},
\quad j \in N_A(i), j' \in N_B(k).
\label{eq:Y_3_hop}
\end{align}
Here $\eta$ is a threshold to be specified later. 
Define a similarity matrix $W$, where $W_{ik}$ is the size of a maximum bipartite matching of $Y^{(ik)}$:
\begin{align}
W_{ik} = \max \;  & \iprod{Y^{(ik)}}{M} \nonumber \\
\text{s.t. } & \sum_{j} M_{jj'} \le 1, \nonumber \\
&  \sum_{j'} M_{jj'} \le 1,  \nonumber \\
& M_{jj'} \in \{0, 1\}.  \label{eq:def_W_3_hop}
\end{align}
Finally, we match vertices in $A$ to vertices in $B$ greedily by sorting the similarities  $W_{ik}$'s.
The entire algorithm is summarized in 
\prettyref{alg:dist_3_hop} below.

\begin{algorithm}
\caption{Graph matching via neighbors' degree profiles}\label{alg:dist_3_hop}
\begin{algorithmic}[1]
\STATE {\bfseries Input:} Graphs $A$ and $B$ on $n$ vertices, an integer $L$, and a threshold $\eta >0$.
\STATE {\bfseries Output:} A permutation $\hat\pi \in S_n$.
\STATE  For each $i,k \in [n]$, compute $W_{ik}$ as in \prettyref{eq:def_W_3_hop}.
\STATE Sort $\{W_{ik}: i,k\in[n]\}$ and let $\calS$ be the set of indices of the largest $n$ elements.
\IF{$\calS$ defines a perfect matching on $[n]$, i.e., $\calS = \{(i,\hat\pi(i)): i \in [n]\}$ for some permutation $\hat\pi$}
\STATE Output $\hat\pi$;
\ELSE 
\STATE Output error.
\ENDIF
\end{algorithmic}
\end{algorithm}

The intuition behind \prettyref{alg:dist_3_hop} is as follows. Even if the $\tilde Z$ distance statistics of degree
profiles are not correctly separated for all pairs,
the new $W$ statistics are guaranteed to be well separated. Indeed, by setting 
\begin{equation}
\eta=\eta_0 \sqrt{ \frac{L}{nq}}
    \label{eq:eta-threshold}
\end{equation}
 for some
sufficiently small absolute constant $\eta_0$, 
we expect that 
\begin{itemize}
\item for true pairs $k=\pi^*(i)$,  $i$ and $k$ share many (about $nqs$) ``common neighbors''(in the sense that $j \in N_A(i)$ and $\pi^*(j) \in N_B(k)$). Moreover,  most of such common neighbors have $\tilde{Z}$ distance smaller than $\eta$. 
As a consequence, 
$W_{ik}$ is at least  $nq/4$ with high probability;

\item for fake pairs $k \neq \pi^*(i)$, $i$ and $k$ share very few (about $nq^2$) ``common neighbors''. 
Moreover, most of the fake pair of vertices $j \in N_A(i)$ and $j' \in N_B(k)$ have
$\tilde{Z}$ distance larger than $\eta$. 
As a consequence,
when $q$ is small, $W_{ik}$ is smaller than $nq/4$ with high probability. 
\end{itemize}

The performance guarantee of \prettyref{alg:dist_3_hop} is as follows:
\begin{theorem}
\label{thm:sparse}
Fix any constant $\epsilon>9/10$.
Suppose
$$
C_0 \log n \le nq \le n^{1-\epsilon} \quad \text{ and }
\quad \sigma \le \frac{\sigma_0}{\log (nq)}
$$
for some sufficiently large absolute constant $C_0$ and 
some sufficiently small absolute constant
$\sigma_0$. 
Set 
$
L= L_0 \log (nq)
$
and $\eta$ as in \prettyref{eq:eta-threshold}
for some sufficiently large absolute constant $L_0$ and some sufficiently small absolute constant $\eta_0$. Then  
with probability at least $1-O \left( n^{9-10\epsilon}\right)$, \prettyref{alg:dist_3_hop} outputs $\hat{\pi} =\pi^*$.
\end{theorem}
\nb{We briefly explain the condition on the graph sparsity in \prettyref{thm:sparse}. 
On the one hand, the analysis of \prettyref{alg:dist_3_hop} requires the graphs to be sufficiently sparse ($nq \le n^{1-\epsilon} $ for $\epsilon>9/10$), so that all $2$-hop neighborhoods are tangle-free, each containing at most one cycle. On the other hand, \prettyref{thm:sparse} requires the graphs cannot be too sparse (\ie, 
$nq \gtrsim \log n $) so that each vertex has enough neighbors; this lower bound is information-theoretically necessary for exact recovery \cite{cullina2016improved,cullina2017exact}. }

\section{Analysis}
\label{sec:pf}

Throughout this section, without loss of generality, we assume the true permutation $\pi^*$ is the identity.

We introduce a number of events regarding the neighborhoods
$N_A(i)$ and $N_B(k)$. Recall that $a_i=|N_A(i)|$ and $b_k=|N_B(k)|$ denote the degrees.
Put 
\begin{equation}
c_{ik} = \left| N_A(i) \cap N_B(k) \right| . 
\label{eq:cik}
\end{equation}

First, for each $i \in [n]$, define the events
\begin{align}
\Gamma_A(i) & = \left\{  \frac{1}{2} nq  \le  a_i  \le 2nq \right\}, 
\quad \Gamma_B(i) = \left\{  \frac{1}{2} nq  \le  b_i  \le 2nq \right\},
 \label{eq:ab_bound} \\
\Gamma_{ii} & = \left\{ c_{ii}  \geq \frac{1}{2} nq \right\}.  
\label{eq:c_bound}
\end{align}
Second, for each pair of $i, k \in [n]$ with $i \neq k$, define the event 
\nb{\begin{align}
\Gamma_{ik} = \left\{ \sqrt{c_{ik}} \leq \sqrt{nq^2} + \sqrt{2 \log n}  \right\} \label{eq:c_bound_wrong}.
\end{align}}
Note that $a_i, b_i \sim \Bin(n-1,q)$. Moreover, 
$c_{ii} \sim \Bin(n-1, qs)$ which is stochastically larger than
$\Bin(n-1,3q/4)$ under the assumption $s=1-\sigma^2 \geq 3/4$;
for $i \neq k$, 
$c_{ik} \sim \Bin(n-2,q^2).$
Thus, it follows from the binomial tail bounds
\prettyref{eq:bintail} and \prettyref{eq:bintail_upper} that 
\begin{align}
\prob{\Gamma^c_A(i)}, \prob{\Gamma^c_B(i)}, \prob{\Gamma^c_{ii}} 
& \le e^{-\Omega(nq)} \le n^{-3},  \quad \forall i \in [n], \label{eq:thetac}\\
\prob{\Gamma_{ik}^c } & \le n^{-3},  \quad \forall i \neq k \in [n]\label{eq:thetac_diff},
\end{align}
where we use the assumption that $nq \ge C \log n$ for a sufficiently
large constant $C$.

Third, given any $\Delta>0$, for each pair of $i, k \in [n]$, define the event 
\begin{align}
\Theta_{ik} \triangleq  \left\{  | a_i - b_k | \le 4 \sqrt{ nq \Delta} \right\}.
\label{eq:def_Theta_ik}
\end{align}
In view of the binomial tail bounds \prettyref{eq:bintail_lower} and \prettyref{eq:bintail_upper}, we have that
\begin{align*}
\prob{ \sqrt{nq}-\sqrt{\Delta } \leq \sqrt{a_i} \leq \sqrt{nq} + \sqrt{\Delta} 
} \ge 1- 2e^{-\Delta}
\end{align*}
and similarly for $b_k$. Thus it follows from the union bound 
that 
\begin{align}
\prob{\Theta_{ik}} \ge 
\prob{\sqrt{nq}-\sqrt{\Delta } \leq \sqrt{a_i}, \sqrt{b_k} \leq \sqrt{nq} + \sqrt{\Delta} } \ge 1- 4e^{-\Delta}. \label{eq:Theta_ik_tail}
\end{align}

Lastly, for each $i \in [n]$, define the event 
\nb{
\begin{align}
\Theta_i = \left\{  \max\{ \sqrt{a_i - c_{ii}}, \sqrt{b_i - c_{ii} } \}
 \leq  \sqrt{nq (1-s) } + \sqrt{ \Delta }  
 \right\}. \label{eq:def_Theta_i}
\end{align}
}
Since both $a_i-c_{ii}$ and $b_i-c_{ii}$ are distributed as $\Binom(n-1,q(1-s))$,
it follows from the binomial tail bound \prettyref{eq:bintail_upper}  and the union bound that
\begin{align}
\prob{\Theta_i} \ge 1 - 2e^{-\Delta}.  \label{eq:Theta_i_tail}
\end{align}

\subsection{Proof of \prettyref{thm:guarantee-distance}}
	\label{sec:dist_proof}
	
	The proof of \prettyref{thm:guarantee-distance} is structured as follows:
	\begin{center}
	\begin{tikzpicture}[scale=1,transform shape,node distance=2.5cm,auto,>=latex']
    \node [coordinate] (l12) {};
    \node [plan] (l10) [right of=l12,node distance=3cm] {\prettyref{lmm:corsample}};    
		\node [plan] (l122) [below of=l10,node distance=2cm]  {\prettyref{lmm:beta}		};
		\node [plan] (l1) [right of=l10,node distance=3cm] {\prettyref{lmm:true}};    
		\node [plan] (l11) [below of=l1,node distance=2cm] {\prettyref{lmm:tvconc}};    
		\node [plan] (l2) [below of=l11,node distance=2cm] {\prettyref{lmm:fake}};    
		\node [plan] (l9) [left of=l2,node distance=3cm] {\prettyref{lmm:indsample}};
		\node [plan] (t1) [right of=l11,node distance=5cm] {\prettyref{thm:guarantee-distance}};    
	\draw[-Implies,double distance=1pt] (l122) -- (l1);
    \draw[-Implies,double distance=1pt] (l10) -- (l1);
		\draw[-Implies,double distance=1pt] (l11) -- (l1);
		\draw[-Implies,double distance=1pt] (l11) -- (l2);
		\draw[-Implies,double distance=1pt] (l9) -- (l2);
		\draw[-Implies,double distance=1pt] (l1) -- (t1);
		\draw[-Implies,double distance=1pt] (l2) -- (t1);
\end{tikzpicture}	
\end{center}
\medskip


We start with the following results on separating 
the maximum distance among true pairs $\max_{i \in [n]} Z_{ii}$ and 
the minimum distance among wrong pairs $\min_{i \neq k \in [n]} Z_{ik}$: 

\begin{lemma}[True pairs]
\label{lmm:true}	
Assume that 
$\sigma \leq \frac{1}{2}$, $q \le q_0 \le \frac{1}{8}$, $nq \ge C \max\{ \log n, L^2, \Delta\}$
for some sufficiently large constant $C$, and 
\begin{equation}
4 L  \sqrt{ nq\Delta}   \le n. \label{eq:condtrue}
\end{equation}
There exist absolute constants $\tau_1,\tau_2$ such that for each $i \in [n]$,
    \begin{equation}
\prob{ Z_{ii} \ge \xitrue \mid N_A(i), N_B(i)} 
\indc{\Gamma_{A} (i) \cap \Gamma_B(i) \cap \Gamma_{ii}  \cap \Theta_{i} \cap \Theta_{ii} } 
\le O(e^{-\Delta/2}),
    \label{eq:true}
    \end{equation}
where 
\nb{\begin{align}
\xitrue \triangleq L \sqrt{\frac{2\beta}{nq}} +\tau_2 \sqrt{\frac{\Delta}{nq}}+ \tau_2 \sigma \sqrt{\frac{L}{nq}}
\label{eq:xitrue}
\end{align}}
and 
    \begin{equation}
        \beta \triangleq \tau_2 \left( \sigma + \sqrt{\frac{\Delta }{n} } +  \frac{1}{\sqrt{nq}} +  e^{-\Delta} \right) + 
        \frac{1}{L} \exp \left( - \tau_1 \min \left\{ \frac{1}{\sigma^2 L^2}, \frac{n}{L^2 \Delta }, 
        \frac{\sqrt{np}}{L} \right\} \right).
        \label{eq:beta}
    \end{equation}


\end{lemma}

\begin{lemma}[Fake pairs]
\label{lmm:fake}
Assume that  
$\sigma \leq \frac{1}{2}$, $q \le q_0$ for some sufficiently small constant $q_0$, $nq \ge C \max\{ \log n, L^2, \Delta \}$,
and $L  \ge L_0$ for some sufficiently large constant $L_0$. 	 
Then there exist universal constants $c_1,c_2, c_3$, 
such that for each distinct pair $i\neq k$ in $[n]$, 
	\begin{equation}
    \prob{ Z_{ik} \le \xifake \mid N_A(i), N_B(k) } \indc{\Gamma_A(i) \cap \Gamma_B(k) \cap \Gamma_{ik} \cap \Theta_{ik}}
    \le O\left(e^{-\Delta/2}\right)
	\label{eq:fake},
	\end{equation}
where 
\begin{align}
\xifake\triangleq  c_1 \sqrt{\frac{L}{nq}} - c_2 \sqrt{\frac{\Delta}{nq}}.
\label{eq:xifake}
\end{align}	
	
\end{lemma}

Note that the conclusions of Lemma 1 and 2 are stated in a conditional form conditioned on the 
neighborhoods $N_A(i)$ and $N_B(k)$.  This is for the purpose of analyzing Algorithm 2, where we will need to apply these lemmas to high-degree vertices (see proof of \prettyref{thm:guarantee-distance-deg}).

We now prove \prettyref{thm:guarantee-distance}:
\begin{proof}
It suffices to show that with probability $1-O(1/n)$, 
\[
\min_{i \neq k \in [n]} Z_{ik} > \max_{i \in [n]} Z_{ii}.
\]


Choose 
\begin{align}
\Delta = \left( \frac{c_1}{4 \max \{ c_2, \tau_2 \} }  \right)^2 L, \label{eq:choice_Delta}
\end{align}
where $c_1,c_2$ and $\tau_2$ are the absolute 
constants given in \prettyref{lmm:fake} and \prettyref{lmm:true}, respectively.

In view of the theorem assumptions 
$\sigma \le \sigma_0/\log n$, $L=L_0\log n$, and $nq \ge C_0 \log^2 n$, 
we have that $\beta$ in \prettyref{eq:beta} satisfies
\begin{align*}
\beta L & \leq \tau_2  L_0
\left( \sigma_0  
+  \log n \sqrt{\frac{\Delta}{n}}
+  \frac{1}{\sqrt{C_0}} + e^{-\Delta} \log n  \right)
+ \exp \left( -  \tau_1
\min \left\{  \frac{1}{\sigma_0^2 L_0^2}, \frac{n}{ L_0^2  \Delta \log^2 n }, 
\frac{\sqrt{C_0} }{L_0} \right\} \right)  \\
& \leq \frac{c_1^2}{32},
\end{align*}
provided that $\sigma_0 L_0$ is sufficiently small,
and $n$ and $\sqrt{C_0}/L_0$ are sufficiently large. 
Moreover, $\tau_2 \sigma \le 1/8$ when $\sigma_0$ is sufficiently small. 
Thus, in view of \prettyref{eq:xitrue}, \prettyref{eq:xifake}, and 
\prettyref{eq:choice_Delta}, we have
\begin{align}
\xifake \geq  \frac{3c_1}{4} \sqrt{\frac{L}{nq}}
> \frac{5c_1}{8} \sqrt{\frac{L}{nq}} 
\geq \xitrue. \label{eq:xi_separation}
\end{align}

Also, since $L=L_0 \log n$, \prettyref{eq:condtrue} is satisfied for sufficiently large $n$.    
Hence, all the conditions of \prettyref{lmm:true} and 
\prettyref{lmm:fake} are fulfilled. 
Furthermore, for $L_0$ sufficiently large,  
 we have $e^{-\Delta/2} \leq n^{-3}$. 

Applying \prettyref{lmm:true} and 
averaging over $N_A(i)$ and $N_B(i)$ over
both sides of \prettyref{eq:true}, we get that
$$
\prob{ \left\{ Z_{ii} \ge \xitrue \right\} 
\cap \Gamma_A(i) \cap \Gamma_B(i) \cap \Gamma_{ii} \cap \Theta_i \cap \Theta_{ii}}
\le O\left(e^{-\Delta/2}\right).
$$
By the union bound, we get that
\begin{align}
\prob{ \max_{i \in [n]} Z_{ii} \ge \xitrue } 
& \le
\sum_{i \in [n]} 
\bigg( \prob{ \left\{ Z_{ii} \ge \xitrue \right\} 
\cap \Gamma_A(i) \cap \Gamma_B(i) \cap \Gamma_{ii}  \cap \Theta_i \cap \Theta_{ii} } \nonumber \\
&~~~~ +\prob{ \Gamma^c_A(i) } +\prob{ \Gamma^c_B(i) }+ \prob{ \Gamma^c_{ii} } + \prob{ \Theta^c_i} + \prob{\Theta^c_{ii}} \bigg) \nonumber \\
&\le O\left( n^{-2} \right) + O\left( n e^{-\Delta/2}\right) \le O\left( n^{-2} \right).
\label{eq:xitrue_tail}
\end{align}
where the second-to-the-last inequality holds 
due to \prettyref{eq:thetac},
\prettyref{eq:Theta_ik_tail} and \prettyref{eq:Theta_i_tail}.

Similarly, for $i\neq k$, applying 
\prettyref{lmm:fake} and averaging over $N_A(i)$ and $N_B(k)$ over
both hand sides of \prettyref{eq:fake}, we get that
$$
\prob{ \left\{ Z_{ik} \le \xifake \right\} 
\cap \Gamma_A(i) \cap \Gamma_B(k) \cap \Gamma_{ik} \cap \Theta_{ik}}
\le O\left(e^{-\Delta/2}\right).
$$
By the union bound, we get that
\begin{align}
& \prob{ \min_{i \neq k \in [n] } Z_{ik}  \le \xifake }  \nonumber \\
& \le 
\sum_{i \neq k } 
\left(
 \prob{ \left\{ Z_{ik} \le \xifake \right\} \cap \Gamma_A(i)  \cap \Gamma_B(k) \cap \Gamma_{ik} 
 \cap \Theta_{ik}
}
+ \prob{\Gamma_A^c(i) } + \prob{\Gamma_B^c(k) }+  \prob{\Gamma_{ik}^c} + \prob{\Theta^c_{ik}} \right) \nonumber \\
&\le O\left(  n^2 \right) \times \left( e^{-\Delta/2} + n^{-3} \right) 
\le O\left( n^{-1} \right), \label{eq:xifake_tail}
\end{align}
where the second-to-the-last inequality holds due to \prettyref{eq:thetac}, \prettyref{eq:thetac_diff}, 
and \prettyref{eq:Theta_ik_tail}.

Finally, combining \prettyref{eq:xitrue_tail} and \prettyref{eq:xifake_tail}, we conclude that, with probability at least $1- O(1/n)$, 
\[
\min_{i \neq k \in [n]} Z_{ik} 
\geq \xifake > \xitrue \geq \max_{i \in [n]} Z_{ii},
\]
and hence \prettyref{alg:dist} succeeds.
\qed
\end{proof}

\subsection{Proof of \prettyref{thm:guarantee-distance-deg}}
	\label{sec:distdeg}

	The proof of \prettyref{thm:guarantee-distance-deg} is structured as follows:
	\begin{center}
\begin{tikzpicture}[scale=1,transform shape,node distance=2.5cm,auto]
     \node [planwide] (l678) [node distance=4cm] {\prettyref{lmm:degcorr}, \ref{lmm:deg_Theta_ik_cond}, \ref{lmm:deg_Theta_i_cond}};    
		\node [plan] (t3) [right of=l678,node distance=4cm] {\prettyref{thm:guarantee-distance-deg}};    
		\node [planwide] (l12) [right of=t3,node distance=4cm] {\prettyref{lmm:true} and \ref{lmm:fake}};    
		\node [plan] (l3) [below of=t3,node distance=2cm] {\prettyref{lmm:seed}};    
		\node [plan] (l4) [below left of=l3,node distance=3cm] {\prettyref{lmm:seed_matching_mm}};    
		\node [plan] (l5) [below right of=l3,node distance=3cm] {\prettyref{lmm:seed_matching_last}};    
    \draw[-Implies,double distance=1pt] (l678) -- (t3);
		\draw[-Implies,double distance=1pt] (l3) -- (t3);
		\draw[-Implies,double distance=1pt] (l12) -- (t3);
		\draw[-Implies,double distance=1pt] (l4) -- (l3);
		\draw[-Implies,double distance=1pt] (l5) -- (l3);
		\draw [draw=black,dashed] (-0.5,-5) rectangle (8.5,-1);
		\node at (7,-1.5) {\small seeded matching};
\end{tikzpicture}	
\end{center}
\medskip


We start with a few intermediate lemmas, whose proofs are postponed till \prettyref{sec:degcorr}.
Recall that $\alpha$ is defined in \prettyref{eq:alpha} as 
 \begin{equation*}
\alpha \triangleq \left( \alpha_0 \frac{\log n}{nq} \right)^{ \frac{(1-p)s}{1-q} }
 \end{equation*}
and $\tau$ is defined in \prettyref{eq:tau} as 
\begin{equation*}
\tau \triangleq \min \left\{  0 \le k \le n:  
\prob{ \Binom (n-1 ,q ) \ge k } \le \alpha \right\}.
\end{equation*}
Note that $\sigma^2=1-s$ and $p=q/s$. 


The first lemma bounds the correlations
between the degree of vertex $i$ in graph $A$ and
the degree of vertex $k$ in graph $B$.
\begin{lemma}
\label{lmm:degcorr}
Suppose $q \le \frac{1}{8}$, 
$nq \to + \infty$, 
$1/(nq) \le \alpha \le 1/4$,
and $\sigma^2 \log \log (nq) =o(1)$.
Then 
\begin{align}
\prob{a_i \geq \tau , b_k \geq \tau+1}  
\begin{cases}
\ge \Omega\left( \alpha^{ \frac{1-q}{(1-p)s} } \right) & \text{ if } i =k \\
\le \alpha^2 & \text{ o.w. } \label{eq:deg_correlation}
\end{cases}
\end{align}
\end{lemma}

We also need the following two auxiliary lemmas.  
\begin{lemma}\label{lmm:deg_Theta_ik_cond}
Suppose $q \le 1/8$, $1/(nq) \le \alpha \le \alpha_1$
for a sufficiently small constant $\alpha_1>0$,
$ nq \ge C_0 \Delta^2$, and 
$\Delta \ge C_0$ 
for a sufficiently
large constant $C_0>0$.
Let event $\Theta_{ik}$ be given in \prettyref{eq:def_Theta_ik} as
$\Theta_{ik} \triangleq  \left\{  | a_i - b_k | \le 4 \sqrt{ nq \Delta} \right\}.$
Then 
\begin{align}
\prob{ \left\{ a_i \ge \tau, b_k \ge \tau + 1 \right\} \cap \Theta^c_{ik} }
\le O\left( \alpha^{1+ \indc{i\neq k}} e^{-\Delta/2}  \right).
 \label{eq:deg_Theta_ik_cond}
\end{align}
\end{lemma}

\begin{lemma}\label{lmm:deg_Theta_i_cond}
Let the event $\Theta_{i}$ be defined in \prettyref{eq:def_Theta_i}. Then
\begin{align}
\prob{ \left\{ a_i \ge \tau, b_i \ge \tau +1  \right\} \cap \Theta^c_{i} }
\le  2 \alpha e^{-\Delta/2} + 2 e^{-\Delta/(2 \sigma^2 ) } .
 \label{eq:deg_Theta_i_cond}
\end{align}
\end{lemma}

\begin{proof}[Proof of \prettyref{thm:guarantee-distance-deg}]
Recall that $L$ is given in \prettyref{eq:L_relax} as 
$L = L_0 \max \left\{  \log^{1/3} (n), \log \frac{\log n}{q} \right\}.$
Choose $\Delta$ as per \prettyref{eq:choice_Delta}: $\Delta = \left( \frac{c_1}{4 \max \{ c_2, \tau_2 \} }  \right)^2 L$
and set $\xi= \frac{3c_1}{4}\sqrt{\frac{L}{nq}}$, 
where $c_1,c_2$ are from \prettyref{lmm:fake} and $\tau_2$ are from \prettyref{lmm:true}.
Then $\xifake$ in \prettyref{eq:xifake} satisfies
$\xifake\triangleq  c_1 \sqrt{\frac{L}{nq}} - c_2 \sqrt{\frac{\Delta}{nq}} \ge \xi$. Under the condition \prettyref{eq:cond-main1_relax}:
$\sigma \leq \sigma_0 \min\sth{\frac{1}{(\log n)^{1/3}}, \frac{1}{\log \frac{\log n}{q}}}$, 
we have $\sigma L \le \sigma_0 L_0$. Moreover,
under the assumption 
\prettyref{eq:cond-main2_relax}: $nq^2 \ge C_0 \log^2 n$ for some large absolute constant $C_0$, we have 
$nq \ge C L^2$ for a sufficiently large constant 
$C$. 
Thus, $\beta$ in \prettyref{eq:beta} satisfies 
$\beta L \le c_1^2/32$. Moreover, $ \tau_2 \sigma \le \frac{1}{4}$
provided that $\sigma_0$ is a sufficiently small constant. 
Hence, $\xitrue$ in \prettyref{eq:xitrue}
satisfies $\xitrue \triangleq L \sqrt{\frac{2\beta}{nq}} + \tau_2 \sigma  \sqrt{\frac{L}{nq}} 
+ \tau_2 \sqrt{\frac{\Delta}{nq}} \le \xi$.

For ease of notation, for each pair of $i, k \in [n]$, denote 
the event that $\calD_{ik}=\{a_i \geq \tau,b_k \geq \tau+1 \}$. Then, for wrong pairs $i \neq k$, 
\begin{align*}
& \prob{  a_i \geq \tau, b_k \geq \tau +1 , Z_{ik} \leq \xi  } \\
& =\expect{ \prob{Z_{ik}  \le \xi \mid N_A(i), N_B(k) } \indc{ \calD_{ik}} } \\
& \le \expect{ \prob{Z_{ik}  \le \xi \mid N_A(i), N_B(k) } \indc{ \calD_{ik} \cap \Gamma_A(i) \cap \Gamma_B(k) \cap \Gamma_{ik}
\cap \Theta_{ik} } }
+ \prob{ \calD_{ik} \cap \left( \Gamma_A(i) \cap \Gamma_B(k) \cap \Gamma_{ik} \cap \Theta_{ik} \right)^c } \\
& \stepa{\le} O\left( e^{-\Delta/2} \right) \prob{ \calD_{ik} \cap \Gamma_A(i) \cap \Gamma_B(k) \cap \Gamma_{ik}  
\cap \Theta_{ik} }  + \prob{\calD_{ik} \cap \Gamma_A^c(i) }  \\
& ~~+ \prob{\calD_{ik} \cap \Gamma_B^c(k) } + \prob{\calD_{ik} \cap\Gamma_{ik}^c} + \prob{ \calD_{ik} \cap \Theta^c_{ik}}  \\
& \le O\left( e^{-\Delta/2} \right) \prob{ \calD_{ik}}  + \prob{ \Gamma_A^c(i)} + \prob{\Gamma_B^c(k)}  + \prob{\Gamma_{ik}^c}+
\prob{ \calD_{ik} \cap \Theta^c_{ik} } \\
& \stepb{\le}  O\left( \alpha^2 e^{-\Delta/2} \right) + O\left(n^{-3}\right),
\end{align*}
where 
(a) is due to \prettyref{lmm:fake} and $\xifake \ge \xi$; 
(b) is due to \prettyref{lmm:degcorr}, \prettyref{lmm:deg_Theta_ik_cond},
\prettyref{eq:thetac}, and \prettyref{eq:thetac_diff}.
Therefore,  it follows from the union bound that
\begin{align*}
\prob{ \exists (i,k) \in \calS : i\neq k } &\le 
 \sum_{i\neq k} 
\prob{ a_i \geq \tau, b_k \geq \tau +1 , Z_{ik} \leq \xi } \\
&\leq O\left(n^2 \right)  \alpha^{2} \exp \left( -\Delta/2\right)  + O\left( n^{-1} \right) \\
&\stepa{\leq} O \left( \alpha_0 \frac{\log n}{q} \right)^{2}  
 \exp \left( \frac{2 \sigma^2}{1-q} \log n - \Omega(L) \right) + O\left( n^{-1} \right)\\
 &\stepb{=} O \left( e^{-\Omega(L)} \right) +  O\left( n^{-1} \right),
\end{align*}
where 
(a) was previously explained in \prettyref{eq:nalpha};
(b) is due to the condition
\prettyref{eq:cond-main1_relax} on $\sigma$ and the choice of $L$ in 
\prettyref{eq:L_relax}. 

For true pairs, let 
\[
T = \sum_{i \in [n]} \indc{a_i \geq \tau, b_i \geq \tau+1, Z_{ii} \leq \xi} .
\]
To show that $T = \Omega( \alpha_0 \frac{\log n}{q})$ with high probability, we compute its first and second moment. 
Since $Z_{ii}$ and the degrees $a_i,b_i$ are dependent, one needs to be careful with respect to conditioning.
Note that
\begin{align}
 \prob{ a_i \geq \tau, b_i \geq \tau+1, Z_{ii} \leq \xi}
 & = \expect{ \prob{Z_{ii} \leq \xi \mid N_A(i), N_B(i)} \indc{\calD_{ii}}}
 \nonumber \\
 & \ge \expect{ \prob{Z_{ii} \leq \xi \mid N_A(i), N_B(i)} \indc{\calD_{ii} \cap \Gamma_A(i)
 \cap \Gamma_B(i) \cap \Gamma_{ii}
 \cap \Theta_i \cap \Theta_{ii}}} \nonumber \\
& \ge \left( 1- O\left( e^{-\Delta/2} \right) \right)
\prob{ \calD_{ii} \cap \Gamma_A(i) \cap \Gamma_B(i)\cap \Gamma_{ii} \cap \Theta_i \cap \Theta_{ii}},
\label{eq:aiic1}
\end{align}
where the last inequality holds due to \prettyref{lmm:true} and $\xitrue \le \xi$.
 
By \prettyref{lmm:degcorr},
\begin{equation}
t \triangleq \prob{a_i \geq \tau, b_i \geq \tau+1  }
=\prob{\calD_{ii}} \geq \Omega\left( \alpha^{ \frac{1-q}{(1-p)s} } \right) \overset{\prettyref{eq:alpha}}{=} \Omega\left( \alpha_0 \frac{\log n}{nq} \right).
\label{eq:aibi}
\end{equation}

Combining \prettyref{lmm:deg_Theta_ik_cond} and \prettyref{lmm:deg_Theta_i_cond}
together with the union bound, we get that 
\begin{align}
\prob{\calD_{ii}  \cap \left( \Theta_i \cap \Theta_{ii} \right)^c  } 
\le  O\left( \alpha e^{-\Delta/2} + e^{-\Delta/ (2\sigma^2)}  \right)
\label{eq:Theta_ii_tail_cond}
\end{align}
Combining the last two displayed equations yields that 
\begin{align*}
& \prob{ \calD_{ii} \cap \Gamma_A(i) \cap \Gamma_B(i) \cap \Gamma_{ii} \cap \Theta_i \cap \Theta_{ii}} \\
&\ge \prob{ \calD_{ii}} - \prob{ \calD_{ii}  \cap \left( \Theta_i \cap \Theta_{ii}
\right)^c} - \prob{\Gamma_A^c (i) } - \prob{\Gamma_B^c (i) } - \prob{\Gamma_{ii} }  \\
& \ge t - O\left( \alpha e^{-\Delta/2} +  e^{-\Delta/ (2\sigma^2)}  \right) - 3n^{-3},
\end{align*}
where in the last inequality we used $\prob{\Gamma_A^c(i)},  \prob{\Gamma_B^c (i) }, \prob{\Gamma_{ii} }\le 1/n^3$ 
by \prettyref{eq:thetac}. 

In view of the definition
of $\alpha$ given in \prettyref{eq:alpha}, we get that 
\begin{align*}
\alpha e^{-\Delta/2}
& =  \left( \alpha_0 \frac{\log n}{nq} \right)^{ \frac{(1-p)s}{1-q} } e^{-\Delta/2} \\
& \stepa{\le} \alpha_0 \frac{\log n}{nq} \exp \left(  \frac{\sigma^2}{1-q} \log n
- \frac{\Delta}{2} \right) \\
& \stepb{=}  \alpha_0 \frac{\log n}{nq} \exp \left(  \frac{\sigma^2}{1-q} \log n
- \Omega(L) \right)\\
& \stepc{=} O(t) e^{-\Omega(L)},
\end{align*}
where (a) is by \prettyref{eq:exponent_break_1}; (b) is due to $\Delta=\Omega(L)$
by our choice of $\Delta$;
(c) holds because of \prettyref{eq:aibi} and the facts that $\sigma \le \sigma_0/\log^{1/3}(n)$ 
in view of condition \prettyref{eq:cond-main1_relax} 
and $L \ge L_0 \log^{1/3}(n)$ in view of \prettyref{eq:L_relax}. 

Furthermore, by our choice of $\Delta$ and the theorem assumptions, 
$ \Delta/\sigma^2 \ge 6 \log n$ by letting $L_0/\sigma_0^2$ sufficiently large. 
Combining this fact with the last two displayed equations, we get that
\begin{align}
\prob{ \calD_{ii} \cap \Gamma_A(i) \cap \Gamma_B(i) \cap \Gamma_{ii} \cap \Theta_i \cap \Theta_{ii}} 
\ge t \left( 1 - O \left( e^{-\Omega(L)} \right) \right) - 4 n^{-3}.
\label{eq:aiic2}
\end{align}

By \prettyref{eq:aiic1} and \prettyref{eq:aiic2}, we get that 
\begin{equation}
\Expect[T]
\ge nt \left( 1- O\left( e^{-\Delta/2} \right) \right) \left( 1- O\left( e^{-\Omega(L)}\right) \right) - O(n^{-2})
= nt \left( 1- O\left( e^{-\Omega(L)}\right) \right)- O(n^{-2}), \label{eq:T1st}
\end{equation}
where the last equality holds because $\Delta=\Theta(L)$.

Next we estimate the second moment of $T$:
\begin{align*}
\Expect[T^2]
\leq & ~ \sum_{i,j} \prob{a_i \geq \tau, b_i \geq \tau +1 , a_j \geq \tau, b_j \geq \tau+1 }\\
=& ~  n t + \sum_{i \neq j} \prob{a_i \geq \tau, b_i \geq \tau +1 , a_j \geq \tau, b_j \geq \tau+1}.
\end{align*}
We will show that for $i \neq j$, 
\begin{equation}
\prob{a_i \geq \tau, b_i \geq \tau+1, a_j \geq \tau, b_j \geq \tau+1} \leq t^2 \left(1 + e^{-\Omega(L)} \right).
\label{eq:aijbij}
\end{equation}
It then follows that 
\begin{equation}
\expect{T^2} \leq nt+ n^2 t^2\left(1 + e^{-\Omega(L)} \right).
\label{eq:T2nd}
\end{equation}
Combining \prettyref{eq:T1st} and \prettyref{eq:T2nd}, 
we get that 
$$
\var(T)=\Expect[T^2] - \left(\Expect[T]\right)^2 
\leq O \left( n^2 t^2  e^{-\Omega(L)}   + n t \right)
$$ and hence by Chebyshev's inequality,  
$$
\prob{T \geq \frac{1}{2} nt } 
\le \frac{\var(T)}{ \left( \expect{T} - nt/2 \right)^2} = O \left( e^{-\Omega(L)}  + \frac{1}{nt} \right) = 
O \left( e^{-\Omega(L)}  + \frac{q }{\log n } \right) = O\left( \frac{q}{\log n} \right),
$$
where the last two equalities holds because $nt =\Omega(\log n/q)$ and $L\ge L_0 \log \frac{\log n}{q}$  in view of \prettyref{eq:L_relax}. 
Therefore, the set $\calS$ defines a partial matching with 
$|\calS| = T \geq  nt/2$ with probability $1-O(q/\log n)$. 
Finally, the success of \prettyref{alg:distdeg} follows from applying the seeded graph matching result
\prettyref{lmm:seed} given in \prettyref{app:seed}.

\medskip
It remains to prove \prettyref{eq:aijbij}. Fix $i \neq j$. 
Recall that $\calD_{ii}$ is the event that $a_i \ge \tau$ and 
$b_i \ge \tau+1$. Also, let $g_i$ denote the degree of vertex $i$ in the parent graph.
Abusing notation slightly, we let $k$ denote the realization of $g_i$ in the remainder of
the proof. Then 
\begin{align*}
 \prob{\calD_{ii} \cap \calD_{jj} } =\sum_{k,k'} 
\prob{ g_i =k, g_j = k'} \prob{ \calD_{ii} \mid g_i =k}
 \prob{ \calD_{jj} \mid g_j=k'}
\end{align*}
and
\begin{align*}
 \prob{g_i = k, g_j =k'}  = & p \cdot \prob{ \Binom(n-2,p) = k-1} \prob{ \Binom(n-2,p)=k'-1}  \\
& + (1-p)  \prob{ \Binom(n-2,p) = k} \prob{ \Binom(n-2,p)=k'}.
\end{align*}
For ease of notation, we write $c_k\triangleq\prob{ \Binom(n-2,p) = k}$. 
Then 
\begin{align*}
 & \prob{g_i = k, g_j =k'} - \prob{g_i=k} \prob{g_j = k'} \\
 & = p c_{k-1} c_{k'-1} + (1-p) c_{k} c_{k'} - \left( p c_{k-1} + (1-p) c_k \right)
 \left( p c_{k'-1} + (1-p) c_{k'} \right) \\
 & = p (1-p) \left( c_{k-1} - c_{k} \right)\left( c_{k'-1} - c_{k'} \right).
 \end{align*}
By definition,
$$
\frac{c_{k-1} - c_{k}}{c_{k-1}} = \left( 1 - \frac{ (n-k-1) p }{ k (1-p)} \right)
= \frac{ k - (n-1)p }{k (1-p)}
$$
and 
$$
 \frac{ c_{k-1} -c_{k} }{c_k}=  \left( \frac{ k (1-p)  }{(n-k-1) p } - 1  \right)
= \frac{k - (n-1) p }{ (n-k-1) p }.
$$

We let 
$$
\eta  \triangleq \frac{ \sqrt{3} \log (np)}{\sqrt{np}}
$$
and $I \triangleq [ (1-\eta) (n-1) p, (1+\eta) (n-1) p ]$.
Then for all $k \in I$, we have
$$
\frac{ \left| c_{k-1} - c_{k} \right| }{ \min \{ c_{k-1}, c_k \}}
\le \frac{\eta} { \min \{  (1-\eta) (1-p), 1- (1+\eta) p  \} }
 \le \frac{2\eta}{1-\eta},
$$
where the last equality holds due to $p \le 1/2$. 
Thus, for all $k ,k' \in I$, we have 
$$
 \prob{g_i = k, g_j =k'} 
 \le \left( 1+  \frac{4\eta^2 }{ (1-\eta)^2 } \right) \prob{g_i=k} \prob{g_j = k'} . 
$$
Moreover, by Chernoff's bound given in \prettyref{eq:bintail}, 
$$
\prob{ g_i \notin I} \le 2 \exp\left( - \eta^2 n p /3 \right) =
2 \exp\left( - \log^2(np) \right).
$$
Therefore,
\begin{align*}
& \prob{ \calD_{ii} \cap \calD_{jj} } \\
& \le \prob{ g_i \notin I} + \prob{g_j \notin I}
+ \sum_{k,k' \in I } \prob{g_i = k, g_j =k'} 
\prob{ \calD_{ii} \mid g_i =k' }
\prob{ \calD_{jj} \mid g_j =k'} \\
& \le 4 \exp\left( - \log^2(np) \right)
+ \left( 1+  \frac{4\eta^2 }{ (1-\eta)^2} \right) 
\sum_{k,k'} \prob{g_i=k} \prob{g_j = k'} \prob{ \calD_{ii} \mid g_i =k' }
\prob{ \calD_{jj} \mid g_j =k'}\\
& = 4 \exp\left( - \log^2(np) \right) + 
\left( 1+  \frac{4\eta^2 }{ (1-\eta)^2} \right) 
 \prob{\calD_{ii}}  \prob{\calD_{jj}}= \left( 1+ e^{-\Omega(L)}  \right) t^2,
\end{align*}
where the last equality holds due to $\prob{\calD_{ii}}=t =\Omega( \log (n)/(nq))$ 
and  $\eta^2 + \frac{1}{t^2} \exp\left( - \log^2(np) \right)=\exp(-\Omega(L))$ under the assumptions of \prettyref{thm:guarantee-distance-deg}.
\qed
\end{proof}

\subsection{Proof of \prettyref{lmm:true} and \prettyref{lmm:fake}}
	\label{sec:pf-lemma}
Note that for both the case of $i=k$ and $i\neq k$, the empirical distribution $\mu_i$ and $\nu_k$ will both involve correlated samples arising from common neighbors. So we start by decomposing the empirical distribution according to the common neighbors.
Fix $i,k$. Recall that $c_{ik}=|N_A(i) \cap N_B(k)|$.
Then
\begin{align}
\mu_i = & ~ \frac{c_{ik}}{a_i} \pth{\frac{1}{c_{ik}} \sum_{j \in N_A(i) \cap N_B(k)  }  \delta_{a_j^{(i)}}}	+ 
\pth{1-\frac{c_{ik}}{a_i}} \pth{\frac{1}{a_i-c_{ik}} \sum_{j \in N_A(i)\backslash N_B(k)}    \delta_{a_j^{(i)}}}	\label{eq:decomp1}, \\
\nu_k = & ~ \frac{c_{ik}}{b_k} \pth{\frac{1}{c_{ik}} \sum_{j \in N_A(i) \cap N_B(k)  }  \delta_{b_j^{(k)}}}	+ 
\pth{1-\frac{c_{ik}}{b_k}} \pth{\frac{1}{b_k-c_{ik}} \sum_{j \in N_B(k)\backslash N_A(i)}    \delta_{b_j^{(k)}}}	\label{eq:decomp2} .
\end{align}
As a consequence, the centered empirical distribution can be rewritten as
\begin{align}
\bar\mu_i = & ~ 
\rho P  + (1-\rho) P'   \label{eq:muPQ1}\\
\bar\nu_k = & ~ \rho' Q + (1-\rho') Q' \label{eq:muPQ2}
\end{align}
where 
\[
\rho \triangleq  \frac{c_{ik}}{a_i}  , \quad \rho'\triangleq \frac{c_{ik}}{b_k}
\]
and
\begin{align*}
P \triangleq  & ~ \frac{1}{c_{ik}} \sum_{j \in N_A(i) \cap N_B(k)  }  \delta_{a_j^{(i)}} - \nu, \quad P' ~\triangleq  \frac{1}{a_i-c_{ik}} \sum_{j \in N_A(i)\backslash N_B(k)}    \delta_{a_j^{(i)}} - \nu, \\
Q \triangleq  & ~ \frac{1}{c_{ik}} \sum_{j \in N_A(i) \cap N_B(k)  }  \delta_{b_j^{(k)}} - \nu', \quad Q'~\triangleq  \frac{1}{b_k-c_{ik}} \sum_{j \in N_B(k)\backslash N_A(i)}    \delta_{b_j^{(k)}} - \nu',
\end{align*}
and $\nu=\Binomc(n-a_i-1,q)$ and $\nu'=\Binomc(n-b_k-1,q)$.
Note that if $c_{ik}=0$, we set $P=Q=\Bin(n-1,q)$ by default. 

The following lemmas are the key ingredients of the proof:

\begin{lemma}[Independent two samples]
\label{lmm:indsample}	
	Let $X_1,\ldots,X_m $ and 
	$Y_1,\ldots,Y_{m'} $ be two independent sequence of real-valued random variables,
      where $X_i$'s are independently distributed as $\nu_i$ and 
      $Y_i$'s are independently distributed as $\nu'_i$. 
	Assume that for some $m_0$,
	\[
	 \kappa_1 \leq \frac{m}{m_0},\frac{m'}{m_0}  \leq \kappa_2
	\]
	for some absolute constants $\kappa_1,\kappa_2>0$.
	
	Suppose the partition $I_1,\ldots,I_L $ is chosen so that 
      there exists a set $J_0 \subset [m]$ with $|J_0| \ge m/4$ 
      such that for all $i \in J_0$ and for all $\ell \in [L]$,  
	\begin{equation}
	\frac{c_1}{L} \leq \nu_i(I_\ell) \leq \frac{c_2}{L}
	\label{eq:partition}
	\end{equation}
      for some absolute constants $c_1, c_2 \in (0,1]$.

	Given any two distributions $\nu$ and $\nu'$ on the real line, 
	define $\pi = \frac{1}{m}\sum_{i=1}^m \delta_{X_i}  - \nu$ and $\pi' = \frac{1}{m'}\sum_{i=1}^{m'} \delta_{Y_i}  -\nu'$.
	Assume that $m_0 \geq C L$ 
and $L \geq L_0$ for some sufficiently large constants $C,L_0$.
	Then for any $\Delta>0$, 
	\begin{equation}
	d(\pi,\pi') \geq  \alpha_1 \sqrt{\frac{L}{m_0}} - \alpha_2 \sqrt{\frac{\Delta}{m_0}}
	\label{eq:indsample}
	\end{equation}
	with probability at least $1-e^{-\Delta}$, where 
	$d$ is the pseudo-distance defined in \prettyref{eq:distance} with respect to the partition $I_1,\ldots,I_L$, and 
	$\alpha_1,\alpha_2$ are absolute constants. 
	\end{lemma}


\begin{lemma}[Correlated two samples]
\label{lmm:corsample}	
		Let $(X_1,Y_1),\ldots,(X_m,Y_m)$ be iid so that $X_i \sim \nu$ and 	$Y_i\sim\nu'$.
		Let $\pi = \frac{1}{m}\sum_{i=1}^m \delta_{X_i}-\nu$ and $\pi' = \frac{1}{m}\sum_{i=1}^{m} \delta_{Y_i}-\nu'$. 
		Assume that  for any $\ell \in [L]$,
		\begin{equation}
		\prob{X_1 \in I_\ell, Y_1 \notin I_\ell} + \prob{X_1 \notin I_\ell, Y_1 \in I_\ell} \leq \beta.
		\label{eq:corr-key}
		\end{equation}
		Then for any $\Delta>0$,
	\begin{equation}
	d(\pi,\pi') \leq L \sqrt{\frac{\beta}{m}} + c_3 \sqrt{ \frac{\Delta}{m}}
	\label{eq:corsample}
	\end{equation}
    with probability at least $1-e^{-\Delta}$, where $\beta$ is defined in \prettyref{eq:beta} and $c_3$ is an absolute constant.
\end{lemma}

\nb{

\begin{remark}
	In \prettyref{lmm:indsample}, the samples $X_i$'s and $Y_i$'s need not be identically distributed, and $\nu$ and $\nu'$ can be arbitrary so that $\pi$ and $\pi'$ need not be centered (which is the case when we apply \prettyref{lmm:indsample} for proving Lemmas \ref{lmm:fake} and \ref{lmm:fake_3_hop}). This is because \prettyref{lmm:indsample} aims to lower bound the distance and centering tends to make the distance smaller. However, in \prettyref{lmm:corsample} which bounds the distance from above, the samples are required to be iid and the empirical distributions must be correctly centered.
	\end{remark}

\begin{lemma}[Concentration of total variation]
\label{lmm:tvconc}	
	Let $X_1,\ldots,X_m$ be drawn independently from a discrete distribution $\nu$ supported on $k$ elements.
	Then the empirical distribution $\nu_m = \frac{1}{m}\sum_{i=1}^m \delta_{X_i}$ satisfies that for any $\Delta>0$, 
	\[
	\prob{\|\nu-\nu_m\|_1 \geq \sqrt{\frac{k}{m}}  + \sqrt{\frac{\Delta}{m}}  } \leq e^{-\Delta/2}.
	\]
\end{lemma}}

In order to apply \prettyref{lmm:corsample}, we need to quantify the correlation and upper bound the probability $\beta$ in \prettyref{eq:corr-key}.
This is given by the following (elementary but extremely tedious) lemma:
\begin{lemma}
\label{lmm:beta}	
Assume that 
$\sigma \leq 1/2 $, $q \le \frac{1}{8}$, $nq \ge C \max\{L^2, \Delta\}$,   
and \prettyref{eq:condtrue} holds, \ie, 
$
4L  \sqrt{ nq \Delta }  \le n. 
$
Then for any $j \in N_A(i) \cap N_B(i)$ and any 
interval $I \subset [-1/2, 1/2]$ with $|I|=1/L$,
	\begin{align}
	& 
    \left( \prob{ a^{(i)}_j \in I, b^{(i)}_j \notin I \; \Big| \; N_A(i), N_B(i)}
    + \prob{ a^{(i)}_j \notin I, b^{(i)}_j \in I \; \Big| \;  N_A(i), N_B(i)} 
    \right) \indc{\Gamma_A(i) \cap \Gamma_B(i) \cap \Gamma_{ii} \cap \Theta_i \cap \Theta_{ii} }  \nonumber \\
	& \lesssim \sigma + \sqrt{\frac{\Delta }{n} } + \frac{1}{\sqrt{nq}} + \frac{1}{L} \exp \left( - \Omega \left( \min \left\{ \frac{1}{\sigma^2 L^2}, \frac{n}{L^2 \Delta}, \frac{\sqrt{np}}{L} \right\} \right)\right) + e^{-\Delta}.
	\label{eq:beta-sampling}
	\end{align}
\end{lemma}
\begin{remark}
Note that  for the right hand side of \prettyref{eq:beta-sampling}
to be much smaller than $1/L$, it suffices to have
$L \ll \min\{ 1/ \sigma, \sqrt{n/\Delta}, \sqrt{nq} \}$ 
and $\Delta \gg \log L$.
\end{remark}

\subsubsection{Proof of \prettyref{lmm:true} }
\begin{proof}[Proof of \prettyref{lmm:true}]
      Fix $i \in [n]$. Throughout the proof, we condition on the neighborhoods $N_A(i)$
and $N_B(i)$  such that $\Gamma_A(i) \cap \Gamma_B(i) \cap \Gamma_{ii} \cap \Theta_i
    \cap\Theta_{ii}$ holds. 

\nb{Recall the pseudo-distance $d$ defined in \prettyref{eq:distance}, namely,
\begin{equation}
d(\mu,\nu) = \|[\mu]_L-[\nu]_L\|_1
\label{eq:distance-def}
\end{equation}
where $[\mu]_L$ is the discretized version of $\mu$, defined in \prettyref{eq:discretization}, according to the uniform partition 
$I_1, \ldots, I_L$ of $[-1/2,1/2]$ such that $|I_\ell| = 1/L$. }
Using the decomposition in \prettyref{eq:muPQ1}--\prettyref{eq:muPQ2} and the triangle inequality for the total variation distance, we have
\begin{align}
Z_{ii}
= & ~ d\left(\rho P     + (1-\rho) P', \rho Q+ (1-\rho') Q' + (\rho'-\rho) Q \right)      \nonumber \\
\leq & ~ \underbrace{d(P,Q)}_{\rm (I)} + 
\underbrace{ (1-\rho) \|[P']_L\|_1 + (1-\rho') \|[Q']_L\|_1}_{\rm (II)} + 
\underbrace{|\rho-\rho'|}_{\rm (III)},   \label{eq:Zii1}
\end{align}
where $\rho=\frac{c_{ii}}{a_i}$ and $\rho'=\frac{c_{ii}}{b_i}$.

For (I), in view of the assumption \prettyref{eq:condtrue}: $4L \sqrt{nq \Delta} \le n$, 
\prettyref{lmm:beta} yields that  for any $j \in N_A(i) \cap N_B(i)$
and any interval $I \subset [-1/2,1/2]$ with $|I| =1/L$, 
      \begin{align*}
      & \prob{ a^{(i)}_j \in I, b^{(i)}_j \notin I \Big| N_A(i), N_B(i)} + \prob{ a^{(i)}_j \notin I, b^{(i)}_j \in I \Big| N_A(i), N_B(i)}  \\
      & \leq 
      \beta \triangleq O(\sigma) + 
    O \left( \sqrt{\frac{\Delta}{n} } + e^{-\Delta} \right) + \frac{1}{L} \exp \left( - \Omega \left( \min \left\{ \frac{1}{\sigma^2 L^2}, \frac{n}{L^2 \Delta}, \frac{\sqrt{np}}{L} \right\} \right) \right) .
      \end{align*}
We apply \prettyref{lmm:corsample} with $\{X_j\}_{j=1}^m$ given by $ 
\{ a_j^{(i)} \}_{j \in N_A(i) \cap N_B(i)}$, 
$\{ Y_j\}_{j=1}^m $ given by $\{ b_j^{(i)} \}_{ j \in N_A(i) \cap N_B(i)}$, 
and $m=c_{ii} = |N_A(i) \cap N_B(i)|$. Recall that $a_{j}^{(i)}$ is a function 
of $\{A_{j \ell}\}_{\ell \in N_A^c[i]}$ and $b_{j'}^{(i)}$ is a function 
of $\{B_{j' \ell'}\}_{\ell' \in N_B^c[i]}$. For any $j \neq j' \in N_A(i) \cap N_B(i)$,
it holds that  $\{j, \ell\} \neq \{ j', \ell'\}$. Hence, $(a_j^{(i)}, b_{j}^{(i)})$'s
are independently and identically distributed across different $j \in N_A(i) \cap N_B(i)$.
Therefore, \prettyref{lmm:corsample} yields that   with probability at least $1-e^{-\Delta}$,
    \begin{equation} 
d(P,Q) \leq  L \sqrt{\frac{\beta}{c_{ii}} } + 
c_3 \sqrt{ \frac{\Delta}{ c_{ii} } }
\le L \sqrt{\frac{2\beta}{nq} } + c_3 \sqrt{ \frac{2\Delta}{ nq } },  \label{eq:true1}
\end{equation} 
where $c_3>0$ is some absolute constant given in \prettyref{lmm:corsample},
and the last inequality holds due to $c_{ii} \ge nq/2$ by \prettyref{eq:c_bound}.

\nb{For (II), 
applying \prettyref{lmm:tvconc} with $k=L$ implies that 
$ \|[P']_L\|_1 \leq \sqrt{\frac{L}{a_i-c_{ii}}} + \sqrt{\frac{\Delta}{a_i-c_{ii}}}  $
and $\|[Q']_L\|_1 \leq  \sqrt{\frac{L}{b_i-c_{ii}}} +  \sqrt{\frac{\Delta}{b_i-c_{ii}}}$,
each with probability at least $1-e^{-\Delta/2}$.
Therefore, by the union bound, with probability at least $1-2e^{-\Delta/2}$,
\begin{align}
(1-\rho) \|[P']_L\|_1 + (1-\rho') \|[Q']_L\|_1
& \le \frac{1}{a_i} \left( \sqrt{L}+\sqrt{\Delta}  \right) \sqrt{a_i-c_{ii }} + \frac{1}{b_i} \left( \sqrt{L}+\sqrt{\Delta}  \right)  \sqrt{  b_i-c_{ii} } \nonumber  \\
&  \leq \frac{4}{nq}   \left( \sqrt{L}+\sqrt{\Delta}  \right) \left( \sqrt{nq \sigma^2 } + \sqrt{\Delta} \right) ,
\label{eq:true2}
\end{align}
where the last inequality holds due to $a_i, b_i \ge nq/2$ 
and $\sqrt{a_i-c_{ii}}, \sqrt{b_i-c_{ii}}  \leq  \sqrt{nq \sigma^2 } + \sqrt{ \Delta }$ on the event \prettyref{eq:ab_bound} and \prettyref{eq:def_Theta_i}, respectively.
}

Finally, for (III),
\begin{equation}
|\rho-\rho'|  = \frac{c_{ii} |a_i-b_i | }{ a_i b_i} 
\le \frac{|a_i-b_i|}{a_i } \leq   8 \sqrt{ \frac{\Delta} {nq} }, 
\label{eq:true3}
\end{equation}
where the last inequality holds due to $| a_i - b_i | \le 4 \sqrt{nq \Delta}$ by \prettyref{eq:def_Theta_ik}.

Combining \prettyref{eq:Zii1} with \prettyref{eq:true1}, \prettyref{eq:true2}, \prettyref{eq:true3}, 
we get that with probability at least $1-3e^{-\Delta/2}$, 
\begin{align*}
Z_{ii} 
& \leq  L \sqrt{\frac{2\beta}{nq}}+   c_3 \sqrt{ \frac{2\Delta}{ nq } }
+ \frac{4}{nq}   \left( \sqrt{L}+\sqrt{\Delta}  \right) \left( \sqrt{nq \sigma^2 } + \sqrt{\Delta} \right)
+ 8 \sqrt{ \frac{  \Delta}{nq} }\\
& \leq  
L \sqrt{\frac{2\beta}{nq}} + \tau_2 \sigma \sqrt{ \frac{L}{nq}  } + \tau_2 \sqrt{\frac{\Delta }{nq}}
\end{align*}
for some absolute constant $\tau_2>0$, where the last inequality holds due to 
the assumption that $nq \ge C \max\{L^2, \Delta\}$ for some sufficiently large constant $C$.
Thus we arrive at the desired \prettyref{eq:true}. 
\qed
\end{proof}

\subsubsection{Proof of \prettyref{lmm:fake}}
\begin{proof}[Proof of \prettyref{lmm:fake}]
Fix $i\neq k$. We proceed as in the proof of \prettyref{lmm:true} and condition
    on the neighborhoods $N_A(i)$ and $N_B(k)$ such that $\Gamma_A(i) \cap \Gamma_B(k) \cap \Gamma_{ik} \cap \Theta_{ik}$ holds.

By the triangle inequality for the total variation distance, we have
\begin{align}
Z_{ik}
= & ~ d \left(\rho P	+ (1-\rho) P', \rho Q+ (1-\rho) Q' + (\rho-\rho') (Q'-Q) \right)	\nonumber \\
\geq & ~ \underbrace{(1-\rho) d(P',Q')}_{\rm (I)} - \underbrace{\rho d(P,Q)}_{\rm (II)} - \underbrace{2 |\rho-\rho'|}_{\rm (III)}. 	\label{eq:Zik1}
\end{align}
where $\rho=\frac{c_{ik}}{a_i}$ and $\rho'=\frac{c_{ik}}{b_k}$.

For (I), note that $a_i, b_k \ge nq/2$ by \prettyref{eq:ab_bound}, and 
$c_{ik} \le nq/4$ for all $i \neq k$ by \prettyref{eq:c_bound_wrong} and the assumptions that $nq \ge C \log n$
and $q \le q_0$. 
Thus 
\begin{equation}
\rho, \rho' \le 1/2.
\label{eq:fake1}
\end{equation}
Let 
$$
J = N_A(i) \backslash N_B(k), \quad J'= N_B(k) \backslash N_A(i).
$$
To analyze $d(P',Q')$, we aim to apply \prettyref{lmm:indsample} with 
$m=|J|$, $m'=|J'|$, $m_0=nq$,
$\{X_j\}_{j=1}^m$ given by $\{ a_j^{(i)}\}_{j \in J}$, and 
$\{Y_j\}_{j=1}^{m'}$ given by $\{ b_j^{(k)}\}_{j \in J'}$. 
However,  \prettyref{lmm:indsample} is not directly applicable 
because the outdegrees are not independent due to the edges between nodes in $J$ and $J'$ (cf.~\prettyref{fig:lemma2}). 
\begin{figure}[ht]%
\centering
\includegraphics[width=0.7\columnwidth]{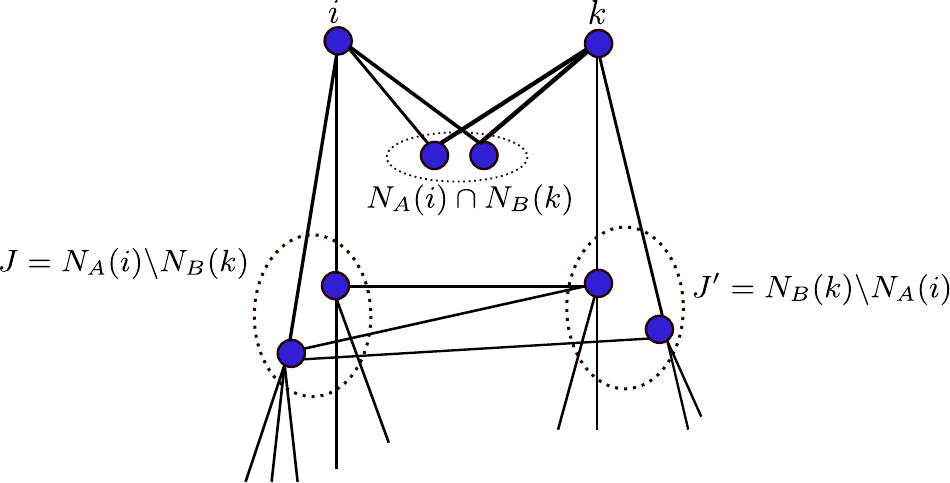}%
\caption{Conditioned on the edge set $E_A(J,J')$ and $E_B(J,J')$, the outdegrees are independent.}%
\label{fig:lemma2}%
\end{figure}
Indeed, note that $a_j^{(i)}$'s are independent across $j$, and 
$ b_{j'}^{(k)}$'s are independent across $j'$, but $a_j^{(i)}$ and 
$ b_{j'}^{(k)}$ are dependent, because
$A_{jj'} $ contributes to the outdegree $a_j^{(i)}$, 
$B_{jj'} $ contributes to the outdegree $b_{j'}^{(k)}$,
and $A_{jj'}$ are correlated with $B_{jj'}$. 
To deal with this dependency issue,
define $E_A(J, J')$ as the set of edges between vertices in $J$ and vertices in $J'$ in $A$
and let $e_A(J,J')=|E_A(J,J')|$. Similarly, define $E_B(J,J')$ and $e_B(J,J')$.
Conditioned on the edge sets $E_A(J, J')$ and $E_B(J, J')$, the outdegrees $\{a_j^{(i)}: j \in J\}$ and $\{b_{j'}^{(k)}: j' \in J'\}$ are mutually independent (although not identically distributed as binomials). 
 Indeed, let $\ell=| J \backslash \{k \}|$ and $\ell'=| J' \backslash \{ i\}|$. 
Then
$$
a_j^{(i)} =  \frac{1}{\sqrt{(n-a_i-1) q (1-q)}} 
\left[ e_A \left( j, N_A^c[i] \backslash J' \right) - (n-a_i-1-\ell') q 
+ e_A \left( j, J' \backslash \{i\} \right) - \ell' q \right]
$$
and
$$
b_{j'}^{(k)} =  \frac{1}{\sqrt{(n-b_k-1) q (1-q)}} 
\left[ e_B \left( j', N_B^c[k] \backslash J \right) - (n-b_k-1-\ell) q 
+ e_B \left( j', J \backslash \{k \}  \right) - \ell q \right]
$$
Note that $\{ e_A \left( j, N_A^c[i] \backslash J' \right)\}_{j \in J}$ are independent from
$\{ e_B \left( j', N_B^c[k] \backslash J \right)\}_{j' \in J}$. 

For each $j \in J \backslash \{k \}$, define the indicator random variable 
$$
\calX (j) =\indc{ \left| e_A \left( j, J' \backslash \{i\} \right) - \ell' q \right| \le \sqrt{nq(1-q)/2} }.
$$
Let 
\begin{equation}
J_0=\{ j \in J \backslash \{k \}: \calX(j)=1\}
\label{eq:J0}
\end{equation} 
Define the event
$$
\calH \triangleq \left\{ |J_0| \ge m/4 \right\}.
$$
Note that for each $j \in  J \backslash \{k \}$, 
$e_A \left( j, J' \backslash \{i\} \right) \sim \Binom(\ell', q)$. Hence, by Chebyshev's inequality, 
$$
\prob{ \calX(j) =1 } \ge 1- \frac{ 2 \ell' }{ n} \ge 1/2,
$$
where the last inequality holds because $\ell' \le 2 nq$  on the event 
$\Gamma_B(k)$ and $q \le 1/8$. Moreover, 
$e_A \left( j, J' \backslash \{i\} \right)$ are independent across $j \in J \backslash \{k \}$. 
Hence, $\sum_{j \in J} \calX_A (j)$ is stochastically lower bounded by 
$\Binom(m-1, 1/2)$. It follows from the binomial tail bound \prettyref{eq:bintail} that
$$
\prob{\calH} = \prob{ |J_0|   \ge m/4  } \ge 1- e^{- m /32 }.
$$

We first condition on $(E_A(J,J'), E_B(J,J'))$ such that the event $\calH$ holds 
and then apply \prettyref{lmm:indsample}. 
In view of \prettyref{eq:fake1},  
$m \ge a_i/2$ and $m' \ge b_k/2$ and thus 
\[
\frac{1}{4} \leq \frac{m}{m_0}, \frac{m'}{m_0} \leq 2. 
\]
Moreover, $nq \ge C L$ and $L \ge L_0$ by assumption. 
It remains to check the condition \prettyref{eq:partition}
in \prettyref{lmm:indsample}.


Let $I$ denote any subinterval of $[-1/2,1/2]$ with length $1/L$. 
Let 
$$
u_j = \frac{1}{\sqrt{(n-a_i-1-\ell') q (1-q)}} 
\left[ e_A \left( j, N_A^c[i] \backslash J' \right) - (n-a_i-1-\ell') q  \right]
$$
and 
$$
v_j= \frac{1}{\sqrt{(n-a_i-1) q (1-q)}}  \left[ e_A \left( j, J' \backslash \{i\} \right) - \ell' q \right] . 
$$
Let 
$$
\alpha_j = \sqrt{\frac{ n- a_i -1 - \ell' }{n-a_i-1}} . 
$$
Then 
$
a_j^{(i)} = \alpha_j u_j + v_j. 
$
It follows that 
$$
\prob{ a_j^{(i)} \in I } = \prob{  u_j \in \frac{I - v_j}{\alpha_j}  } . 
$$

Next we fix $j \in J_0$.
Note that on event $\Gamma_A(i) \cap \Gamma_B(k)$, $a_i,  \ell' \le 2nq$. 
By the assumptions $q \le 1/8$ and $n \ge 4$, 
$$
1 \ge \alpha_j \ge \sqrt{ \frac{ n - 4nq -1  }{n -2nq -1} } \ge \sqrt{\frac{1}{2}},
$$
and, by the definition of $J_0$,
$$
|v_j| \le \frac{\sqrt{nq(1-q)/2}}{ \sqrt{(n-2nq-1) q (1-q)}} \le 1.
$$
Hence,  $(I - v_j)/\alpha_j \subset [-3,3]$.
It follows that
$$
 \frac{1}{\sqrt{2\pi} L}  e^{-1/18} \le \prob{ N(0,1) \in \frac{I - v_j}{\alpha_j} } 
 \le \frac{1}{\sqrt{\pi} L} . 
$$
Note that $u_j \sim \Binomc(n- a_i -1 - \ell' ,q)$. 
By the Berry-Esseen theorem \cite[Theorem 5.5]{petrov}, 
we have
$$
\frac{1}{\sqrt{2\pi} L} e^{-1/18} - \frac{O(1)}{\sqrt{nq(1-q)}} 
\le \prob{  u_j \in \frac{I - v_j}{\alpha_j}  } 
\le \frac{1}{\sqrt{\pi} L} + \frac{O(1)}{\sqrt{nq(1-q)}}  .
$$
In view of the assumption $nq \ge C L^2$ for a sufficiently large constant $C$,  
we have for all $j \in J_0 $ and all $\ell \in [L]$,
$$
\frac{c_1}{L} \leq \prob{ a_j^{(i)}  \in I } \leq \frac{c_2}{L}.
$$
for two absolute constants $c_1, c_2 \in [0,1]$. Finally, recall that 
we have conditioned on $E(J,J')$ such that event
$\calH$ holds. Hence, $|J_0| \ge m/4$. 
Thus, condition \prettyref{eq:partition} in \prettyref{lmm:indsample}
is satisfied.

In conclusions, the assumptions of \prettyref{lmm:indsample}
are all satisfied. 
Then it follows from \prettyref{lmm:indsample} that  
\begin{align}
\prob{ d(P',Q') \geq  \alpha_1 \sqrt{\frac{L}{ n q } } 
- \alpha_2 \sqrt{\frac{\Delta}{ n q} }  ~\Bigg|~ E_A(J, J'), E_B(J,J')} 
 \ge \left( 1- e^{-\Delta} \right)  \iindc{\calH},
\end{align}
where $\alpha_1$ and $\alpha_2$ are absolute constants given in 
\prettyref{lmm:indsample}.
Taking the expectation of $(E_A(J,J'),E_B(J,J'))$ over the both hand sides of the last display,
we get that 
\begin{align}
\prob{ d(P',Q') \geq \alpha_1 \sqrt{\frac{L}{ n q } } 
- \alpha_2 \sqrt{\frac{\Delta}{ n q} }}  
& \ge \left( 1- e^{-\Delta} \right)  \prob{\calH}  \nonumber \\
& \ge \left( 1- e^{-\Delta} \right)  \left( 1- e^{- m /32 } \right) 
\ge 1- 2 e^{-\Delta}, \label{eq:fake2}
\end{align}
where the last inequality holds due to $m \ge  nq/4 \ge C \Delta /4$
for a sufficiently large constant $C$.

\nb{For (II), 
\prettyref{lmm:tvconc} implies that 
$
\|[P]_L\|_1 \leq  \sqrt{\frac{L }{c_{ik} } } +  \sqrt{\frac{\Delta }{c_{ik} } }
$
holds with probability at least $1-e^{-\Delta/2}$;
similarly for $\|[Q]_L\|_1 $.  Thus, 
by the triangle inequality and union bound, 
with probability at least $1- 2e^{-\Delta/2}$,
\[
d(P,Q) \leq \|[P]_L\|_1 + \|[Q]_L\|_1 \leq  2\sqrt{\frac{L }{c_{ik} } } +  2\sqrt{\frac{\Delta}{c_{ik}}}.
\]
Therefore, 
\begin{equation}
\rho \cdot d(P,Q)\leq \frac{c_{ik}}{a_i} \left( 2\sqrt{\frac{L }{c_{ik} } } +  2\sqrt{\frac{\Delta}{c_{ik}}} \right) 
\leq \frac{4}{nq} \left( \sqrt{L} + \sqrt{\Delta} \right) \left( \sqrt{nq^2} + \sqrt{2\log n} \right)
\label{eq:fake3}
\end{equation}
where the last inequality holds due to  $a_i \ge \frac{1}{2} nq $ 
and $\sqrt{c_{ik} } \le \sqrt{nq^2} + \sqrt{2\log n}$ on the event \prettyref{eq:ab_bound} and 
\prettyref{eq:c_bound_wrong} respectively.
}

For (III), 
\begin{equation}
|\rho-\rho'| = \frac{c_{ik} |a_i - b_k|}{a_i b_k} \leq \frac{|a_i - b_k|}{a_i} \leq 
8 \sqrt{ \frac{\Delta}{nq} },
\label{eq:fake4}
\end{equation}
where the last inequality holds due to \prettyref{eq:def_Theta_ik}.

Combining \prettyref{eq:Zik1} with \prettyref{eq:fake1}--\prettyref{eq:fake4}, we 
have that with probability at least 
$1-3e^{-\Delta/2}$, 
\[
Z_{ik} \geq  \frac{\alpha_1}{2}  \sqrt{\frac{L}{nq}} -  \alpha_2  \sqrt{\frac{\Delta}{nq}}
-\frac{4}{nq} \left( \sqrt{L} + \sqrt{\Delta} \right) \left( \sqrt{nq^2} + \sqrt{2\log n} \right)
 - 8 \sqrt{\frac{\Delta}{nq} } 
\geq  c_1 \sqrt{ \frac{L}{nq} } - c_2 \sqrt{\frac{\Delta}{nq}},
\]
for some absolute constants $c_1,c_2>0$, where the last inequality holds by
the assumptions that 
$q \le q_0$ and $nq \ge C \log n$ for some sufficiently small constant $q_0$ and sufficiently
large constant $C$.
\qed
\end{proof}


\subsubsection{Proof of \prettyref{lmm:indsample}, \ref{lmm:corsample}, \ref{lmm:tvconc}, and \ref{lmm:beta}}

\begin{proof}[Proof of \prettyref{lmm:indsample}]
Recall from \prettyref{eq:distance} that 
\[
 d(\pi,\pi')= \sum_{\ell \in [L]}  |\pi(I_\ell) - \pi'(I_\ell)|.
\]
We first show that it suffices to establish
\begin{equation}
\Expect d(\pi,\pi') \geq c_0 \sqrt{\frac{L}{m_0}}.
\label{eq:indmean}
\end{equation}
To prove the concentration inequality \prettyref{eq:indsample}, note that
$d(\pi,\pi')$, as a function of the independent random variables $(X_1,\ldots,X_m,Y_1,\ldots,Y_{m'})$, satisfies the bounded difference property.
Indeed, let 
\[
d(\pi,\pi')=f(X_1,\ldots,X_m,Y_1,\ldots,Y_{m'})
\] 
for some function $f$.
Then for any $i$ and any $x_i,x_i'$, we have, for some $\ell,\ell' \in [L]$,
\begin{align}
& ~ |f(x_1,\ldots,x_i,\ldots,x_m,y_1,\ldots,y_{m'})-f(x_1,\ldots,x_i',\ldots,x_m,y_1,\ldots,y_{m'})| \nonumber \\
\leq & ~ \left| |\pi(I_\ell) + \frac{1}{m} - \pi'(I_\ell)| + |\pi(I_{\ell'}) - \frac{1}{m} - \pi'(I_{\ell'})| - 
|\pi(I_\ell) - \pi'(I_\ell)| - |\pi(I_{\ell'}) - \pi'(I_{\ell'})| \right| \nonumber \\
\leq & ~ \frac{2}{m}. \label{eq:mcdia}
\end{align}
Thus, $f$ satisfies the bounded difference property with parameter $\frac{2}{m \wedge m'}$.
By McDiarmid's inequality, we have
\[
\prob{d(\pi,\pi')  \leq \Expect d(\pi,\pi') - c_1 
\sqrt{\frac{\Delta}{m_0}} } \leq e^{-\Delta},
\]
where $c_1$ depends only on $\kappa_1$ and $\kappa_2$.

It remains to show \prettyref{eq:indmean}. 
For any $\ell \in [L]$, 
\begin{align}
\expect{ \left|\pi(I_\ell) - \pi'(I_\ell) \right| } 
= & \expect{ \left| \frac{1}{m} \sum_{i=1}^m \indc{X_i \in I_\ell} - \frac{1}{m'} \sum_{i=1}^{m'} \indc{Y_i \in I_\ell} 
-  \nu\left( I_\ell \right) 
+ \nu' \left(  I_\ell \right) 
 \right| } \nonumber \\
\geq & ~ \frac{1}{m} \inf_{x \in \reals} \expect{ \left|  \sum_{i \in J_0} 
\indc{X_i \in I_\ell} - x \right| },
\label{eq:indmean_bound}
\end{align}
where the last inequality holds because $X_i$'s and $Y_i$'s are independent. 

For $i \in J_0$, define $\alpha_i \triangleq \prob{ X_i \in I_\ell}$ and $\alpha\triangleq 1/L$.
It follows from assumption \prettyref{eq:partition} that $c_1 \alpha \le \alpha_i \le c_2 \alpha$
for two absolute constants $c_1, c_2 \in (0,1]$. Therefore, 
we can write 
$\indc{X_i \in I_\ell}= W_i Z_i$, where 
$Z_i \iiddistr \Bern(\alpha)$ and $W_i$'s are independently distributed as $\Bern(\eta_i)$ where 
$c_1 \le \eta_i \le c_2$. Let $T=\{i \in J_0: W_i =1\}$. Then for any $x \in \reals$, 
conditional on $T$, 
\begin{align}
\expect{ \left|  \sum_{i \in J_0} \indc{X_i \in I_\ell} - x \right| \mid T}
=\expect{ \left| \sum_{i \in T} Z_i -x \right| }
\ge \expect{ \left|  \Binom(|T| ,\alpha)- x_0 \right| }, \label{eq:mean_absolute_deviation_bound}
\end{align}
where $x_0$ is the median of $\Binom(|T|,\alpha)$, which satisfies $\left|x_0- |T| \alpha \right| \leq 1$~\cite{KB80}.
Using the estimate for the mean absolute deviation of binomial distribution (e.g.~\cite[Theorem 1]{BK13}), we have
\[
\expect{ \left|  \Binom( |T| ,\alpha)- |T| \alpha\right| } \geq \frac{\sqrt{ |T| \alpha (1-\alpha)}}{\sqrt{2}},
\quad  \frac{1}{|T| } \leq \alpha \leq 1-\frac{1}{|T| }.
\]
By assumption, $L \geq L_0$ for some large constant $L_0$. Thus if 
$|T| \ge 16 L$, then $|T| \alpha (1-\alpha) \ge 8$. 
Hence, by triangle inequality, 
\[
\expect{ \left|  \Binom( |T| ,\alpha)- x_0 \right| } \geq 
\left( 
\frac{\sqrt{ |T| \alpha (1-\alpha)}}{\sqrt{2}} -1  \right) \indc{|T| \ge 16 L}
\geq \frac{\sqrt{ |T| \alpha (1-\alpha)}}{2 \sqrt{2}}\indc{|T| \ge 16 L}
\]
Therefore, combining the last displayed equation with \prettyref{eq:mean_absolute_deviation_bound},
we get that for any $x \in \reals$, 
$$
\expect{ \left|  \sum_{i \in J_0} \indc{X_i \in I_\ell} - x \right| \mid T}
\ge  \frac{\sqrt{ |T| \alpha (1-\alpha)}}{2 \sqrt{2}}\indc{|T| \ge 16 L}.
$$
Taking expectation over $T$ and then infimum over $x \in \reals$
on both hand sides of the last displayed equation yields that 
$$
\inf_{x \in \reals} \expect{ \left|  \sum_{i \in J_0} \indc{X_i \in I_\ell} - x \right|}
\ge \frac{\sqrt{\alpha (1-\alpha)}}{2 \sqrt{2} }  
\expect{ \sqrt{|T| } \indc{|T| \ge 16 L} }
$$
It remains to bound $\expect{ \sqrt{|T| } \indc{|T| \ge 16 L} }$ from the below. 
By assumption, it holds that $|J_0| \ge m/4$. Further, recall that
$W_i$'s are independently distributed as $\Bern(\eta_i)$ where 
$c_1 \le \eta_i \le c_2$. Hence 
$|T|$ is stochastically lower bounded by $ U \sim \Binom(m/4, c_1)$ and
thus 
$$
\expect{ \sqrt{|T| } \indc{|T| \ge 16 L} } \ge 
\expect{ \sqrt{U } \indc{U \ge 16 L} } = \expect{\sqrt{U}} - \expect{ \sqrt{U} \indc{U < 16 L}}.
$$
Note that for any $y>0$, $\sqrt{y} \ge 1+ (y-1)/2- (y-1)^2/2$. 
Plugging $y= U/\expect{U}$ and taking expectation, we get that 
$$
\expect{\sqrt{U}} \ge \sqrt{\expect{U}} \left( 1- \frac{ \var(U) } { 2 \left( \expect{ U} \right)^2 } \right)
=\sqrt{m c_1/4}- \frac{1-c_1}{2 \sqrt{ m c_1/4}}. 
$$
Moreover, 
$$
\expect{ \sqrt{U} \indc{U < 16 L}} 
\le 4 \sqrt{L} \prob{ U < 16 L} 
\le 4 \sqrt{L} e^{-\Omega(m)},
$$
where the last inequality follows from the Chernoff bound \prettyref{eq:bintail}
and the fact that $m \geq \kappa_1 m_0 \geq \kappa_1 C L$ for some large constant $C$. 
Combining the last four displays, we have that
$$
\inf_{x \in \reals} \expect{ \left|  \sum_{i \in J_0} \indc{X_i \in I_\ell} - x \right|}
\ge c_3 \sqrt{\frac{m}{L}},
$$
for some absolute constant $c_3$. Combining the last display with
\prettyref{eq:indmean_bound}, we get that 
\[
\expect{ \left|\pi(I_\ell) - \pi'(I_\ell) \right| } \geq c_3 \sqrt{\frac{1}{mL}}.
\]
Summing over $\ell \in [L]$ and noting that $m \ge \kappa_1 m_0$ yields \prettyref{eq:indmean}.
\qed
\end{proof}

\begin{proof}[Proof of \prettyref{lmm:corsample}]
Similar to the proof of \prettyref{lmm:indsample}, observe that 
 $d(\pi,\pi')$ is a function of the independent randomness  $(X_1,Y_1),\ldots,(X_m,Y_m)$ 
satisfying the bounded difference property with parameter $\frac{4}{m}$. 
Thus, by McDiarmid's inequality, to show \prettyref{eq:corsample}, it suffices to show
\begin{equation}
\Expect d(\pi,\pi') \leq L \sqrt{\frac{\beta}{m}}.
\label{eq:cormean}
\end{equation}

Note that 
$$
\expect{d(\pi,\pi')} = \sum_{\ell=1}^L \expect{ \left|\pi(I_\ell) - \pi'(I_\ell) \right| }
$$
and
$$
\pi(I_\ell) - \pi'(I_\ell) = \frac{1}{m} \sum_{i=1}^m 
\left[ \left(  \indc{X_i \in I_\ell } - \indc{Y_i \in I_\ell } \right)
- \left( \prob{ X_i \in I_\ell} - \prob{Y_i \in I_\ell} \right) \right].
$$
For each $i$, 
$$
 \indc{X_i \in I_\ell } - \indc{Y_i \in I_\ell } =
 \begin{cases}
 1 & \text{ w.p. } \prob{X_i \in I_\ell, Y_i \notin I_\ell} \\
 -1 & \text{ w.p. } \prob{X_i \notin I_\ell, Y_i \in I_\ell} \\
 0 & \text{ o.w. }
 \end{cases}
$$
Hence,
\begin{align*}
& \expect{ \left|\pi(I_\ell) - \pi'(I_\ell) \right| } \\
& \le 
\sqrt{\expect{ \left( \pi(I_\ell) - \pi'(I_\ell) \right)^2 } } \\
& =  \frac{1}{m} 
\sqrt{\expect{  \left( \sum_{i=1}^m \left( \indc{X_i \in I_\ell }  - \indc{Y_i \in I_\ell }
- \prob{X_i \in I_\ell} + \prob{Y_i \in I_\ell } \right)  \right)^2 }} \\
& = \frac{1}{m} \sqrt{ \sum_{i=1}^m \expect{  \left(  \indc{X_i \in I_\ell }  - \indc{Y_i \in I_\ell }
- \prob{X_i \in I_\ell} + \prob{Y_i \in I_\ell }  \right)^2 }}  \\
& = \frac{1}{\sqrt{m}} 
\sqrt{ \prob{X_1 \in I_\ell, Y_1 \notin I_\ell} + \prob{X_1 \notin I_\ell, Y_1 \in I_\ell}
- \left(\prob{X_1 \in I_\ell, Y_1 \notin I_\ell} - \prob{X_1 \notin I_\ell, Y_1 \in I_\ell} \right)^2 }\\
& \le \frac{1}{\sqrt{m}} 
\sqrt{ \prob{X_1 \in I_\ell, Y_1 \notin I_\ell} + \prob{X_1 \notin I_\ell, Y_1 \in I_\ell}} \leq \sqrt{\frac{\beta}{m}}.
\end{align*}
Summing over $\ell\in[L]$ gives the desired \prettyref{eq:cormean}.
\qed
\end{proof}

\begin{proof}[Proof of \prettyref{lmm:tvconc}]
Let $\nu$ be supported on the set $\{a_1,\ldots, a_k\}$ with $\nu_i=\nu(\sth{a_i})$. Then by Cauchy-Schwarz inequality, 
\begin{align*}
\Expect\|\nu-\hat\nu\|_1 
& =\frac{1}{m} \sum_{i=1}^k \Expect \left|   \sum_{j=1}^m \left( \indc{X_j=a_i} -\nu_i \right) \right| \\
& \le \frac{1}{m} \sum_{i=1}^k 
\sqrt{\Expect \left(   \sum_{j=1}^m \left( \indc{X_j=a_i} -\nu_i \right) \right)^2 } \\
& = \frac{1}{m} \sum_{i=1}^k  \sqrt{m\nu_i(1-\nu_i)} 
\leq \sum_{i=1}^k \sqrt{\frac{\nu_i}{m}} \leq \sqrt{\frac{k}{m}},
\end{align*}
where the last inequality follows from Jensen's inequality.
	Note that $(X_1,\ldots,X_m) \mapsto \|\nu-\hat{\nu}\|_1$ satisfies the bounded difference property with parameter $\frac{2}{m}$. 
	Thus, by McDiarmid's inequality, we have 
	\[
	\prob{\|\nu-\hat{\nu}\|_1 \geq \sqrt{\frac{L}{m}} + \sqrt{\frac{\Delta}{m} }  } \leq e^{-  \frac{2 \Delta/m}{m (2/m)^2}}  = e^{-\Delta /2}.
	\]	
	\qed
\end{proof}

\begin{proof}[Proof of \prettyref{lmm:beta}]
Let us suppress $i$ and $j$, and 
abbreviate  $a^{(i)}_j$ and $b^{(i)}_j$ as $a$ and $b$. 
Throughout the proof, we condition on $N_A[i]=S$ and $N_B[i]=T$ 
such that event $\Gamma_A(i) \cap \Gamma_B(i) \cap \Gamma_{ii} \cap \Theta_i \cap \Theta_{ii}$ holds,
and aim to show that 
\begin{align}
& \prob{a \in I, b \notin I } \nonumber \\
& \lesssim \sigma + \sqrt{\frac{\Delta}{n} } + \frac{1}{\sqrt{nq}} + \frac{1}{L} \exp \left( - \Omega \left( \min \left\{ \frac{1}{\sigma^2 L^2}, \frac{n}{L^2 \Delta}, \frac{\sqrt{np}}{L} \right\} \right)\right) +  e^{-\Delta}.
\label{eq:ab1} 
\end{align}
The second probability in \prettyref{eq:beta-sampling} follows from the same bound.

Define 
\begin{align}
\zeta=\sqrt{(n-|S|)q(1-q)} \quad \text{ and } \quad \eta=\sqrt{(n-|T|) q (1-q)}. \label{eq:def_zeta_eta}
\end{align}
Recall that on the event $\Theta$,
\begin{align}
 |S \cup T| \le |S| + |T| \le 4 nq \le  n/2,
 \label{eq:theta_2_bound}
\end{align}
where the last inequality holds due to $q \le 1/8$. Hence,
\begin{align}
\sqrt{nq}/2 \le \zeta, \eta \le \sqrt{ n q } . \label{eq:zeta_bound}
\end{align}

Then we can rewrite $a$ and $b$ as 
\begin{align}
a &= \frac{1}{\zeta}  \sum_{k \notin S} (\alpha_k g_k-q), \\
b & = \frac{1}{\eta}  \sum_{k \notin T} (\beta_k g_k-q),
\end{align}
where $g_k$'s are iid as $\Bern(p)$ and $\alpha_k,\beta_k$'s are iid as $\Bern(s)$.
Recall that $\sigma^2=1-s$ and $p=q/s$. 

Define 
\begin{align*}
E= \{ k \notin S: \alpha_k =1 \} \quad  \text{ and } \quad F =\{ k \notin T: \beta_k = 1 \} .
\end{align*}
Then we can decompose $a$ and $b$ as
\begin{align*}
a =   \frac{1}{\zeta} \left( c + x \right) \quad \text{ and } \quad b = \frac{1}{\eta} \left( c+ y \right),
\end{align*}
where
\begin{align}
c &=   \sum_{k \in E \cap F} (g_k - p) \nonumber \\
x  & =   \sum_{k \in E \backslash F} \left(g_k - p \right) + p |E| - ( n - |S| ) q  
\quad \text{ and } \quad  y  =    \sum_{k \in F \backslash E} \left(g_k - p \right) + p |F| - ( n- |T|  ) q .
\label{eq:def_x_y}
\end{align}
Conditional on $\{E, F\}$, $c, x, y$ are mutually independent. 


\nb{We pause to give some intuition behind the remaining argument. 
Loosely speaking, the quantity $c$ captures the correlation between the outdegrees $a$ and $b$, while $x$ and $y$ correspond to the fluctuations. A key step of
the proof is to relate the event $ \{a \in I, b \notin I\} $ to the event that $c$ belongs to an interval of length roughly $|x-y|$. We further show that $|x-y|$ is typically $O(\sqrt{np} \sigma)$. Coupled with the anti-concentration of $c$ (the maximum probability mass of which is at most $O(1/\sqrt{np})$), this shows that $c$ belongs to an interval of length $|x-y|$ with probability at most $O(\sigma)$, giving rise to the first (main) term in the upper bound \prettyref{eq:ab1}. 
The complication comes from the fact that we also need to control the large deviation behavior of 
$|x-y|$, the mismatches between the normalization factors $\zeta$ and $\eta$, as well as the atypical behavior of 
$E, F$.}

Returning to the main proof, note that 
\begin{align*}
E \cap F &= \{ k \in S^c \cap T^c: \alpha_k= \beta_k=1 \} \\
E \backslash F  &= \{ k \in S^c \backslash T^c : \alpha_k =1 \}  \cup \{ k \in  S^c \cap T^c : \alpha_k=1, \beta_k=0 \}  \\
F \backslash E  &= \{ k \in T^c \backslash S^c : \beta_k =1 \}  \cup \{ k \in  S^c \cap T^c : \alpha_k=0, \beta_k=1 \} .
\end{align*}
Therefore, 
\begin{align*}
| E\cap F| & \sim  \Bin \left( |S^c \cap T^c|, s^2 \right) \\
 | E \backslash F | & \sim \Bin \left( |S^c \backslash T^c|, s \right)  + \Bin \left( |S^c \cap T^c|, s (1-s) \right) \\
  | F \backslash E | & \sim \Bin \left( |T^c \backslash S^c|, s \right)  + \Bin \left( |S^c \cap T^c|, s (1-s) \right).
\end{align*}

Recall that on event $\Gamma_A(i) \cap \Gamma_B(i) \cap \Theta_{i}$,
\begin{align*}
|S^c \cap T^c | & = n - |S \cup T| \ge n/2  \\
|S^c \backslash T^c| = |T \backslash S| & \le \left( \sqrt{n(1-s)} + \sqrt{\Delta} \right)^2 \\
|T^c \backslash S^c| = |S \backslash T|  & \le \left( \sqrt{n(1-s)} + \sqrt{\Delta} \right)^2.
\end{align*}
Define 
\begin{align*}
 \tau_1 & = \left( \sqrt{n(1-s)} + 2\sqrt{\Delta} \right)^2 
 + \left( \sqrt{n(1-s)} + \sqrt{\Delta} \right)^2 \quad
 \text{ and } \quad \tau_2 = \left( \sqrt{ n s^2/2} - \sqrt{\Delta } \right)^2 
\end{align*}
and event
$$
\calE=\{ |E \cap F| \ge \tau_2 \} \cap \{ |E \backslash F| \le \tau_1 \} \cap \{ |F \backslash E| \le \tau_1 \}.
$$
Then by binomial tail bounds~\prettyref{eq:bintail_lower}
and~\prettyref{eq:bintail_upper}, 
we have $\prob{\calE^c} \le e^{-\Delta} +4 e^{-2\Delta} \le 5 e^{-\Delta}$.
Moreover, we have that  
\begin{align}
\tau_1 \le 4n (1-s) + 10 \Delta. \label{eq:tau_1_bound}
\end{align}
Also, in view of the assumption $\sigma \le 1/2$ so that $s \ge 3/4$, we have that 
\begin{align}
\tau_2 \ge \left( \frac{3}{4} \sqrt{ n/2}   - \sqrt{\Delta} \right)^2 \ge n/4,  \label{eq:tau_2_bound}
\end{align}
where 
the last inequality holds for sufficiently large $n$ due to $nq \ge C \Delta$.

Note that 
\begin{align}
\prob{ a \in I, b\notin I } & =\expects{\prob{ a \in I, b\notin I \mid  E, F}}{E,F} \nonumber \\
&=\expects{\prob{ a \in I, b\notin I \mid E, F}\iindc{\calE}}{E,F} + \expects{\prob{ a \in I, b\notin I \mid E, F}\iindc{\calE^c} 
}{E,F} \nonumber \\
& \le \expects{\prob{ a \in I, b\notin I \mid  E, F}\iindc{\calE} }{E,F}  + \prob{\calE^c} \nonumber \\
& \le \expects{\prob{ a \in I, b\notin I \mid  E, F}\iindc{\calE} }{E,F}  + 5e^{-\Delta}. \label{eq:decomp_1}
\end{align}
Hence, it remains to bound $\expects{\prob{ a \in I, b\notin I \mid  E, F}\iindc{\calE} }{E,F}$.
Note that 
\begin{align*}
\prob{ a \in I, b\notin I \mid E, F}\iindc{\calE} 
& =\prob{ c \in \zeta I-x, c \notin \eta I - y  \mid  E, F} \iindc{\calE} \\
& =\expects{\prob{ c \in (\zeta I-x) \backslash (\eta I -y) \mid E, F,x,y} \iindc{\calE} }{x,y}  
\end{align*}


Next consider the following two cases by assuming $I=[l, r]$ with $-1/2 \le l \le r \le 1/2$.
\begin{itemize}
\item Case 1: Either $\zeta r -x \le \eta l -y$ or $\eta r - y \le \zeta l-x$. In this case, we have
$(\zeta I-x) \cap (\eta I -y) = \emptyset$. Thus, we have 
$$
\prob{ c \in (\zeta I-x) \backslash (\eta I -y) \mid S, T, E, F,x,y} 
\overset{(a)}{\lesssim} \frac{1}{\sqrt{\tau_2 p}}  \left( \frac{\zeta }{L}  + 1 \right) \lesssim  \frac{1}{L},
$$
where $(a)$ holds because the 
maximum probability mass of $c$ is $\Theta(1/\sqrt{|E\cap F|p})$,
and the number of integral points in $\zeta I-x$ is at most $\zeta/L+1$;
the last inequality holds because $\tau_2 \ge n/4$ in view of
\prettyref{eq:tau_2_bound}, $\zeta \le \sqrt{nq}$ in view of \prettyref{eq:zeta_bound},
and $nq \ge C L^2$ for a sufficiently large constant $C$.

\item Case 2: $\zeta r - x \ge \eta l -y$ and $\eta r - y \ge \zeta l-x$. In this case, we have
$(\zeta I-x) \cap (\eta I -y) \neq \emptyset$. Moreover,
$$
(\zeta I-x) \backslash (\eta I -y) \subset \left[ \zeta l-x, \eta l-y \right]
\cup \left[ \eta r-y, \zeta r -x \right]. 
$$
Hence,
$$
\left|  (\zeta I-x) \backslash (\eta I -y) \right|
\le \left|  x -y + (\eta - \zeta) l \right| + \left| y - x + (\zeta - \eta) r \right|
\le 2 |x-y| + |\eta-\zeta|,
$$
where the last inequality follows from the triangle inequality and the assumption that 
$-1/2 \le l \le r \le 1/2$. Thus,
$$
\prob{ c \in (\zeta I-x) \backslash (\eta I -y) \mid S, T, E, F,x,y} 
\lesssim \frac{1}{ \sqrt{ n p} }  \left( 2 |x-y| + |\eta-\zeta|  + 1 \right),
$$
where the last step holds because the maximum probability mass of $c$ is $\Theta(1/\sqrt{|E\cap F|p})$,
$|E \cap F| \ge \tau_2 \ge n/4$ in view of
\prettyref{eq:tau_2_bound}, and the number of integral points in 
$(\zeta I-x) \backslash (\eta I -y)$ is at most $2 |x-y| + |\eta-\zeta|  + 1$.
\end{itemize}

Combining the above two cases, we get that 
\begin{align*}
& \prob{ c \in (\zeta I-x) \backslash (\eta I -y)  | E, F,x,y } \iindc{\calE} \\ 
& \lesssim  
\left( \frac{1}{ \sqrt{ n p} }
\left( \left| x - y \right| + |\eta-\zeta| + 1 \right)  
+ \frac{1}{L} 
\indc{ x -y \in [\zeta r - \eta l, +\infty) \cup (-\infty, \zeta l-\eta r] } 
\right) \iindc{\calE}.
\end{align*}
Taking expectation of $x,y$ over both hand sides of the last displayed equation, we get that
\begin{align*}
& \prob{ a \in I, b\notin I | E, F}\iindc{\calE} \\
& \lesssim 
\left( \frac{1}{ \sqrt{n p} }  \left( 
\expect{\left| x - y \right| \mid  E, F} +
|\eta-\zeta|  +1 \right)  + \frac{1}{L} 
\prob{ x -y  \in [\zeta r - \eta l, +\infty) \cup (-\infty, \zeta l-\eta r] \mid E, F } 
\right) \iindc{\calE}.
\end{align*}
Further taking expectation of $E, F$ over both hand sides of the last displayed equation, we get that
\begin{align}
& \expects{ \prob{ a \in I, b\notin I | E, F}\iindc{\calE} }{E,F} \label{eq:term_0} \\
 & \lesssim  \frac{1}{ \sqrt{n p} } 
\expects{ \expect{\left| x - y \right| \mid  E, F} \iindc{\calE} }{E, F} 
\label{eq:term_1} \\
&~~~~ + \frac{|\eta-\zeta| +1 }{\sqrt{np}} \label{eq:term_2} \\
&~~~~ + \frac{1}{L}  
 \expects{\prob{ x -y  \in [\zeta r - \eta l, +\infty) \cup (-\infty, \zeta l-\eta r] \mid E, F } 
 \iindc{\calE}}{E,F}  \label{eq:term_3}.
\end{align}
Next we upper bound the three terms \prettyref{eq:term_1}, \prettyref{eq:term_2}, and \prettyref{eq:term_3}
separately. 

\medskip 
\textbf{Upper bound \prettyref{eq:term_1}}:  
\begin{align*}
\expect{ |x| \mid  E, F } \iindc{\calE}
&\le \expect{ \sum_{k \in E \backslash F} (g_k - p) \mid E, F} \iindc{\calE} + \left| p |E| - (n-|S|) q \right| \\
& \le \sqrt{ |E \backslash F| p (1-p) }\iindc{\calE} + \bigg| p |E| - (n-|S|) q \bigg|\\
& \le \sqrt{ \tau_1 p (1-p)} + \bigg| p |E| - (n-|S|) q \bigg| \\
& \lesssim \sqrt{ n p (1-s)} + \sqrt{ p \Delta} + p \bigg|  |E| - (n-|S|) s \bigg|,
\end{align*}
where the last inequality holds due to $\tau_1 \lesssim n(1-s) + \Delta$
in \prettyref{eq:tau_1_bound}. It follows that 
$$
\expects{\expect{ |x| \mid  E, F } \iindc{\calE}}{E,F} 
\lesssim \sqrt{ n p (1-s)} + \sqrt{ p \Delta}   + p \sqrt{ (n - |S|) s (1-s) }
\lesssim \sqrt{ np (1-s)} + \sqrt{ p \Delta },
$$
where the first inequality uses the fact that $|E| \sim \Binom(n - |S|, s)$ 
and hence $\expect{ \big|  |E| - (n-|S|) s \big|} \le
\sqrt{ \expect{ \left(   |E| - (n-|S|) s \right)^2 }} =\sqrt{ (n-|S|) s (1-s)}.$
Similarly, 
$$
\expects{ \expect{  |y| \mid E, F } \iindc{\calE}} {E,F}
\lesssim  \sqrt{ np (1-s) } + \sqrt{ p \Delta }. 
$$
Therefore, by triangle inequality, 
\begin{align}
\expects{\expect{ |x-y| \mid  E, F } \iindc{\calE}}{E,F }
\lesssim \sqrt{ np (1-s)}  +  \sqrt{ p \Delta }.  \label{eq:term_1_end}
\end{align}

\medskip 
\textbf{Upper bound \prettyref{eq:term_2}}: 
In view of definitions of $\zeta$ and $\eta$ in \prettyref{eq:def_zeta_eta}, 
\begin{align}
\left| \eta - \zeta \right| & = \sqrt{q(1-q)} \left| \sqrt{n-|S|} - \sqrt{n-|T| } \right| \nonumber \\
& \le \sqrt{q} \; \frac{ \left| |T| - |S|  \right|}{ \sqrt{n-|S|} + \sqrt{ (n-|T|)} } \nonumber \\
 & \lesssim  q \sqrt{\Delta}, \label{eq:term_2_end}
\end{align}
where the last inequality holds because on event $\Gamma_A(i) \cap\Gamma_B(i) \cap \Theta_{ii}
$, $|S\cup T| \le n/2$ and $ \left| |T| - |S| \right| \le 4 \sqrt{nq \Delta} $.

\medskip 
\textbf{Upper bound \prettyref{eq:term_3}}:
It follows from the last displayed equation that  
$$
\left| \frac{\eta}{\zeta} -1  \right| \le  \frac{ \left|  |T\backslash S| - |S \backslash T| \right| }{ 2 \left( n- |S \cup T| \right) } 
\le \frac{ 4 \sqrt{ nq \Delta}  }{n} \le \frac{1}{L},
$$
where the last inequality holds by the assumption \prettyref{eq:condtrue},
i.e, $4 L \sqrt{ nq \Delta} \le n$.
As a consequence, 
$$
\zeta r- \eta l = \zeta (r-l) +  (\zeta - \eta) l \ge \zeta \left( \frac{1}{L} - 
\frac{1}{2} \left| \frac{\eta}{\zeta} -1 \right|  \right) =  \frac{\zeta}{2L}.
$$
Similarly, 
$
\eta r - \zeta l \ge \frac{ \zeta }{2L}. 
$
Therefore, 
\begin{align}
\prob{ x - y \in [\zeta r -\eta l, +\infty) 
\cup (-\infty, \zeta l - \eta r]  ~\Bigg|~ E, F
}\le \prob{ |x|  \ge \frac{\zeta}{4L} ~\Bigg|~  E, F} 
+\prob{ |y|  \ge \frac{\zeta}{4L} \mid  E, F}. \label{eq:term_3_1}
\end{align}

Recall the definition of $x$ in \prettyref{eq:def_x_y}, 
\begin{align*}
\prob{ |x| \ge \frac{\zeta}{4L} ~\Bigg|~   E, F}   \iindc{\calE}
& \le \prob{  \sum_{k \in E \backslash F} \left(g_k - p \right) \ge \frac{\zeta}{8L}  ~\Bigg|~ E, F} \iindc{\calE}
+\indc{ | p |E| - (n-|S|) q \ge \zeta /(8 L) } .
\end{align*}
By Bernstein's inequality,
\begin{align*}
\prob{  \sum_{k \in E \backslash F} \left(g_k - p \right) \ge \frac{\zeta}{8L}  ~\Bigg|~ E, F} \iindc{\calE}
& \le \exp \left(  - \Omega \left( 
\min\left\{ \frac{\zeta^2}{ |E\backslash F| L^2 p }, \frac{\zeta}{L} \right\} \right) \right) \iindc{\calE} \\
& \le \exp \left(  - \Omega \left( 
\min\left\{ \frac{n }{ \left( n(1-s) + \Delta \right) L^2 }, 
\frac{\sqrt{np} }{L} \right\} \right) \right),
\end{align*}
where the last inequality holds because $ \zeta \ge \sqrt{nq}/2 $ in \prettyref{eq:zeta_bound},
$s \ge 3/4$,  and $| E \backslash F| \le \tau_1 \lesssim n(1-s)+\Delta$ on the event $\calE$
in view of \prettyref{eq:tau_1_bound}. 
By Bernstein's inequality again,
\begin{align*}
\prob{ | p |E| - (n-|S|) q \ge \zeta /(8 L)  }  
& \le \exp \left( - \Omega \left( 
\min \left\{  
\frac{ \zeta^2 }{ n (1-s) L^2 p^2 }, \frac{\zeta} {Lp }  \right\} \right) \right)  \\
& \le \exp \left( - \Omega \left( 
\min \left\{  
\frac{ 1 }{  (1-s) L^2 p }, \frac{\sqrt{n}} {L\sqrt{p} }  \right\} \right) \right),
\end{align*}
where the last inequality holds because $ \zeta \ge \sqrt{nq}/2 $ in \prettyref{eq:zeta_bound}
and $s \ge 3/4$. 
Combining the last three displayed equations yields that 
\begin{align*}
 \expects{\prob{ |x| \ge \frac{\zeta}{4L} ~\Bigg|~  E, F}   \iindc{\calE}}{E,F} 
 \le  
\exp \left( - \Omega \left( \min \left\{ \frac{1}{\sigma^2 L^2}, \frac{n}{L^2 \Delta}, \frac{\sqrt{np}}{L} \right\} \right)\right),
\end{align*}
where we used $\sigma^2=1-s$.
Similarly,
\begin{align*}
\expects{\prob{ |y| \ge \frac{\zeta}{4L} ~\Bigg|~  E, F}  \iindc{\calE}  }{E, F} 
\le \exp \left( - \Omega \left( \min \left\{ \frac{1}{\sigma^2 L^2}, \frac{n}{L^2 \Delta}, \frac{\sqrt{np}}{L} \right\} \right)\right).
\end{align*}
Combining the last two displayed equation with \prettyref{eq:term_3_1},
we get that 
\begin{align}
& \expects{\prob{ x - y \in [\zeta r -\eta l, +\infty) 
\cup (-\infty, \zeta l - \eta r]  \mid E, F
} \iindc{\calE}
}{E,F}  \nonumber \\
& \le \exp \left( - \Omega \left( \min \left\{ \frac{1}{\sigma^2 L^2}, \frac{n}{L^2 \Delta}, \frac{\sqrt{np}}{L} \right\} \right)\right).
\label{eq:term_3_2}
\end{align}


\medskip

Assembling \prettyref{eq:decomp_1}, \prettyref{eq:term_0}, \prettyref{eq:term_1_end},
\prettyref{eq:term_2_end}, and \prettyref{eq:term_3_2}, we arrive at the desired bound \prettyref{eq:ab1}:
\begin{align*}
& \expects{\prob{ a \in I, b\notin I | E, F}\iindc{\calE} }{E,F} \\
& \lesssim \sigma +  \sqrt{\frac{\Delta}{n}} + \frac{1}{\sqrt{nq}} +
\frac{1}{L} \exp \left( - \Omega \left( \min \left\{ \frac{1}{\sigma^2 L^2}, \frac{n}{L^2 \Delta}, \frac{\sqrt{np}}{L} \right\} \right)\right) + e^{-\Delta}.
\end{align*}
\qed
\end{proof}

\subsection{Proof of \prettyref{lmm:degcorr}, \prettyref{lmm:deg_Theta_ik_cond}, and \prettyref{lmm:deg_Theta_i_cond}}
	\label{sec:degcorr}

To prove 	\prettyref{lmm:degcorr} and \prettyref{lmm:deg_Theta_ik_cond}, we need a few auxiliary lemmas. 
	
First, we need the following tight Gaussian approximation results for the binomial distributions  \cite[Theorem 1]{ZubkovSerov13}:
Let $D(p||q) \triangleq p \log \frac{p}{q} + (1-p) \log  \frac{1-p}{1-q}  $ denote the Kullback-Leibler divergence between $\Bern(p)$ and $\Bern(q)$.
\begin{lemma}
\label{lmm:ZS}	
Assume that $k \geq nq+1$. Then
\begin{equation}
h(k) \leq \prob{\Binom(n,q) \geq k} \leq h(k-1). 
\label{eq:ZS}
\end{equation}
where 
\[
h(k) \triangleq Q\pth{\sqrt{2n D\pth{\frac{k}{n} \Big\|q}}},
\] 
and $Q(t) = \int_t^\infty \frac{1}{\sqrt{2\pi}} e^{-x^2/2} dx$ is the standard normal tail probability.
\end{lemma}

Also, we need the following bounds on the Kullback-Leibler divergence:
\begin{lemma}
It holds that 
\begin{align}
 D( x  \| q)  & \ge  \frac{(x-q)^2}{2x (1-q)} \quad  \forall \; 0<q \le x \le 1 ,
 \label{eq:Divergence_lower} \\
 D( x  \| q)  & \le  \frac{(x-q)^2}{2q (1-q)}   \quad \forall \; 0<q \le x \le 1/2.
\label{eq:Divergendce_upper}
\end{align} 
\end{lemma}
\begin{proof}
Note that  
$$
\frac{d}{dx}D(x\|q) = \frac{x (1-q)}{q (1-x)}, \qquad 
\frac{d^2}{dx}D(x\|q)  =\frac{1}{x(1-x)},
\qquad
\frac{d^3}{dx^3}D(x\|q) = \frac{1}{(1-x)^2} - \frac{1}{x^2},
$$
The second-order Taylor expansion of $D(x\|q)$ at $x=q$ gives \prettyref{eq:Divergence_lower}
and the third-order Taylor expansion at $x=q$ gives \prettyref{eq:Divergendce_upper}.
\qed
\end{proof}

Finally, we need the following inequalities relating $Q(tr)$ to $Q(t)^{r^2}$. Note that if we use the approximation $Q(t) \approx e^{-t^2/2}$, these two quantities are equal. The lemma below makes this approximation precise:
\begin{lemma} \label{lmm:Q_function_bound_2}
For any $t >0$ and $ r >0$, we have
$$
 \frac{tr}{1+(tr)^2} t^{r^2} \left( \sqrt{2\pi} \right)^{r^2-1} \leq \frac{Q(tr)}{Q(t)^{r^2}} 
 \leq \left( \sqrt{2\pi} \frac{1+t^2}{t} \right)^{r^2-1} \frac{t^2+1}{r t^2}.
$$
\end{lemma}
\begin{proof}
For the lower bound, using $\frac{x}{1+x^2} \varphi(x)\leq Q(x) \le \frac{1}{x} \varphi(x)$, where
$\varphi(x)=e^{-x^2/2}/\sqrt{2\pi}$, we have
\begin{align*}
Q(tr) \ge  \frac{tr}{1+(tr)^2} \varphi(tr)
\end{align*}
and
$$
\left( Q(t) \right)^{r^2} \le \frac{1}{t^{r^2} }\varphi(t)^{r^2}
$$
Combining the last two displayed equations, we get that
$$
\frac{Q(tr)}{Q(t)^{r^2}} \geq  \frac{tr}{1+(tr)^2} t^{r^2} 
\frac{\varphi(tr)}{ \varphi(t)^{r^2}} =  \frac{tr}{1+(tr)^2} t^{r^2} \left( \sqrt{2\pi} \right)^{r^2-1}.
$$

The upper bound follows similarly from combining
$Q(tr) \le \frac{1}{tr} \varphi(tr)$ and $Q(t) \ge \frac{t}{1+t^2} \varphi(t) $.
\qed
\end{proof}

Now we are ready to prove \prettyref{lmm:degcorr}.
Recall that 
$
\tau \triangleq \min \left\{  0 \le k \le n:  
\prob{ \Binom (n-1 ,q ) \ge k } \le \alpha \right\}
$
as defined in \prettyref{eq:tau}.

\begin{proof}[Proof of \prettyref{lmm:degcorr}]
We first prove \prettyref{eq:deg_correlation} for $i \neq k$. 
Let $b'_k= \sum_{j \neq i} B_{jk}$. Then $a_i$ and $b'_k$ are independent.
Since $b_k' \le b_k \le b_k'+1$, it follows that 
\begin{align}
\prob{a_i \geq \tau , b_k \geq \tau +1 }
& \le \prob{ a_i \geq \tau, b'_k \geq \tau  } \nonumber \\
& \le \prob{ a_i \geq \tau} \prob{b'_k \geq \tau } \nonumber \\
& \le \prob{ a_i \geq \tau} \prob{b_k \geq \tau } \label{eq:a_b_indp} \\
& = \left( \prob{ \Binom(n-1,q) \geq \tau} \right)^2 
\le \alpha^2. \nonumber 
\end{align}

Next we prove \prettyref{eq:deg_correlation} for $i = k$. For notational convenience, 
we abbreviate $a_i$ and $b_i$ as $a$ and $b$, respectively. 
Let $g$ denote the degree of vertex $i$ in the parent graph.
Abusing notation slightly, we let $k$ denote the realization of $g$ in the remainder of the proof. 
Then 
\begin{align}
\prob{ a \ge \tau, b \geq \tau +1 } 
&= \sum_{k \ge 0} \prob{   a \geq \tau, b \ge \tau+1, g =k} \nonumber \\
&= \sum_{k \ge 0} \prob{g=k} \prob{ a \geq \tau \mid g = k } \prob{ b \ge \tau+1 \mid g = k}
\label{eq:a_b_total_prob}.
\end{align}

Let 
$$
k_0 = \left\lceil \frac{\tau+2}{s} \right\rceil.
$$
Since conditional on $g=k$, $a \sim \Binom(k, s)$ and $b \sim \Binom(k,s)$. It follows that 
for all $k \ge k_0$,
\begin{align}
\prob{ a \geq \tau \mid g =k} \ge \prob{ b \ge \tau +1 \mid g = k}
\ge \prob{ \Binom ( k ,s ) \ge ks -1 } \ge \frac{1}{2}, \label{eq:a_b_median}
\end{align}
where the last inequality holds because the median of $\Binom ( k ,s )$ is at least $ks-1$. 
Combining \prettyref{eq:a_b_total_prob} and \prettyref{eq:a_b_median} yields that 
\begin{align}
\prob{ a \ge \tau, b \geq \tau+1} \ge \frac{1}{4} \prob{g \ge k_0} 
 = \frac{1}{4} \prob{ \Binom(n-1, p) \ge k_0}
\label{eq:a_b_correlated_1},
\end{align}
where the last equality holds due to $g \sim \Binom(n-1,p)$ with $p=q/s$. 

It remains to prove that $\prob{ \Binom(n-1, p) \ge k_0} \ge  \Omega \left( \alpha^{ \frac{1-q}{(1-p)s} } \right)$. 
By assumption, $\alpha \le 1/4$ and hence by the Berry-Esseen theorem, 
$\tau \ge (n-1)q +2$ for all $n$ sufficiently large. 
Thus $k_0 \ge (\tau+2)/s \ge np+1$.
It follows from \prettyref{lmm:ZS} that
\begin{align}
\prob{ \Binom(n-1, p) \ge k_0} \ge 
Q \left( \sqrt{ 2 (n-1)  D( k_0/(n-1) \| p ) } \right). \label{eq:g_bound}
\end{align}

To proceed, we need to bound $D( k_0/(n-1) \| p )$ from the above. 
We claim that 
\begin{align}
  0 \le \tau -(n-1) q +  \sqrt{ (n-1) q (1-q) }Q^{-1} (\alpha)
\le  (1-q) \left( Q^{-1}(\alpha) \right)^2 +2,
\label{eq:tau_upper_lower}
\end{align}
where $Q^{-1}$ denote the inverse function of $Q$ function. 
We defer the proof of  \prettyref{eq:tau_upper_lower} to the end.

Note that $Q(x) \le e^{-x^2/2}$ for $x\ge 0$. Hence,
$$
Q^{-1}(\alpha) \le 
\sqrt{2 \log \frac{1}{\alpha} }
\le \sqrt{ 2 \log (nq) },
$$
where the last inequality follows due to 
the assumption $\alpha \ge 1/(nq)$.
Thus it follows from  \prettyref{eq:tau_upper_lower} that $k_0 \le (\tau+3)/s \le (n-1)/2$ for sufficiently large $n$ .
Hence, by \prettyref{eq:Divergendce_upper},
$$
\sqrt{ 2 (n-1)  D( k_0/(n-1) \| p ) } 
\le \frac{ k_0 - (n-1) p }{ \sqrt{(n-1) p(1-p)}}
\le
\frac{  Q^{-1} (\alpha) \sqrt{ (n-1) q (1-q) } 
+ (1-q) \left( Q^{-1}(\alpha) \right)^2 + 5 }{  s \sqrt{ (n-1) p (1-p) } },
$$
where the last inequality holds due to 
$k_0 s \le \tau +3 $
and \prettyref{eq:tau_upper_lower}.

Applying the lower bound in  \prettyref{lmm:Q_function_bound_2} with
$$
t \triangleq Q^{-1}(\alpha), 
\quad r \triangleq \frac{  \sqrt{ (n-1) q (1-q) } + (1-q) t + 5/t  }{ s \sqrt{ (n-1) p (1-p) } },
$$
we get that 
$$
Q \left( \sqrt{ 2(n-1)  D( k_0/(n-1) \| p ) } \right)
\ge \frac{tr}{1+(tr)^2} t^{r^2} \left( \sqrt{2\pi} \right)^{r^2-1} Q(t)^{r^2}.
$$
Note that $Q(t)=Q(Q^{-1}(\alpha))=\alpha$. Moreover, 
in view of $\Omega(1) \le t \le \sqrt{2 \log (nq)} $, we have
$$
r = \sqrt{ \frac{ 1-q}{ s (1-p) } } + 
O\left( \sqrt{\frac{\log nq}{ nq }} \right).
$$
Recall from \prettyref{eq:exponent_break_1} 
that $\frac{1-q}{(1-p)s} = 1+ \frac{1-s}{(1-p)s}= 1+ \frac{\sigma^2}{(1-p)s}$.
Therefore, we get that  
$$
t^{r^2-1}  \le \left(  \sqrt{ 2 \log nq } \right)^{ r^2 -1 }
=\exp \left(   \left( \frac{\sigma^2}{ s (1-p) } + O\left( \sqrt{\frac{\log nq}{ nq }} \right)   \right)
 \left( \log \sqrt{2 \log nq } \right) \right)
=1-o(1),
$$ 
where the last inequality holds because by assumptions, $\sigma^2 \log \log (nq) =o(1)$ 
and $nq \to \infty$.
Moreover,
$$
\alpha^{r^2} =\alpha^{ \frac{1-q}{ s (1-p) } + O\left( \sqrt{\frac{\log nq}{ nq }} \right) }
\ge \alpha^{ \frac{1-q}{ s (1-p) }} \exp \left( -  O\left( \sqrt{\frac{\log nq}{ nq }} \log (nq)  \right)
\right) = (1-o(1)) \alpha^{ \frac{1-q}{ s (1-p) }}  ,
$$
where the inequality holds by the assumption
$\alpha \ge 1/(nq)$, and the last equality holds due to $nq \to \infty$.
Therefore, we get that 
\begin{align}
Q \left( \sqrt{ 2 (n-1)  D( k_0/(n-1) \| p ) } \right)
\ge \left(1-o(1) \right) \alpha^{ \frac{1-q}{(1-p)s} }. \label{eq:a_b_correlated_3}
\end{align}

Combining \prettyref{eq:a_b_correlated_1}, \prettyref{eq:g_bound}, and \prettyref{eq:a_b_correlated_3} yields that
$$
\prob{ a \ge \tau, b \ge \tau+1} \ge \Omega \left( \alpha^{ \frac{1-q}{(1-p)s} } \right),
$$
proving \prettyref{eq:deg_correlation} for $i=j$.

Finally, we verify the claim \prettyref{eq:tau_upper_lower}.
By the definition of $\tau$ and \prettyref{lmm:ZS}, we have that 
\begin{align*}
Q \left( \sqrt{ 2 (n-1)  D( \tau/(n-1) \| q ) } \right)  
& \le \prob{ \Binom (n-1 ,q ) \ge \tau }   \\
& \le \alpha \\
& < \prob{ \Binom(n-1 , q) \ge \tau -1 } \\
& \le Q \left( \sqrt{ 2 (n-1)  D( (\tau -2)/(n-1) \| q )}  \right) .
\end{align*}
Thus,
\begin{align}
 \sqrt{2 (n-1) D( (\tau -2)/(n-1) \| q )} \le Q^{-1}(\alpha)  \le 
 \sqrt{2 (n-1) D( \tau/(n-1) \| q )} \; . \label{eq:alpha_bound}
\end{align}

In view of \prettyref{eq:Divergence_lower}, 
$2 D(x \|q) \le t $ for $t\ge 0$ implies
$$
x^2 - \left( 2q +t (1-q) \right)x + q^2 \le 0,
$$
which further implies 
$$
x \le \frac{2q + t(1-q) + \sqrt{ 4q (1-q) t + t^2 (1-q)^2} }{2}
\le q + \sqrt{q (1-q)t} + t(1-q) ,
$$
where the last inequality holds due to $\sqrt{x+y} \le \sqrt{x} +\sqrt{y}$.
Therefore, it follows from the lower inequality in \prettyref{eq:alpha_bound} that 
$$
\frac{\tau -2}{n-1} \le  q +  \sqrt{ q (1-q) } \frac{Q^{-1}(\alpha) }{\sqrt{n-1} }
+ \frac{(1-q)}{n-1} \left( Q^{-1}(\alpha) \right)^2.
$$ 
Since $q \le 1/8$ by assumption and $Q^{-1}(\alpha) \le \sqrt{2 \log (nq)}$,
it follows that for sufficiently large $n$, $\tau/(n-1) \le 1/2$.
Thus, combining the upper inequality in \prettyref{eq:alpha_bound}
with \prettyref{eq:Divergendce_upper} gives that 
$$
\tau \ge (n-1) q +  \sqrt{ (n-1) q (1-q) }Q^{-1} (\alpha).
$$
Combining the last two displayed equations yields the desired \prettyref{eq:tau_upper_lower}.
\qed
\end{proof}

\begin{proof}[Proof of \prettyref{lmm:deg_Theta_ik_cond}]
Recall that $\Theta_{ik}= \left\{ | a_i - b_k| \le 4 \sqrt{nq\Delta}  \right\}$. 
Thus,
\begin{align*}
& \left\{ a_i \ge \tau, b_k \ge \tau+1 \right\} \cap \Theta^c_{ik} \\
& \subset  \left\{ a_i \ge \tau, b_k \ge \tau + 4 \sqrt{nq \Delta} \right\} 
\cup
\left\{ a_i \ge \tau + 4 \sqrt{nq\Delta} +1, b_k \ge \tau+1 \right\}
\\
& \subset \left\{ a_i \ge \tau, b_k \ge \tau + 4 \sqrt{nq \Delta} \right\}
\cup \left\{ a_i \ge \tau + 4 \sqrt{nq\Delta}, b_k \ge \tau \right\}.
\end{align*}
Hence, by the union bound and the symmetry between
$a_i$ and $b_k$, it suffices to prove 
\begin{align*}
\prob{ a_i \ge \tau , b_k \ge \tau+ 4 \sqrt{nq\Delta} }
\le O\left(\alpha^{1+ \indc{i\neq k}} e^{-\Delta/2} \right). 
\end{align*}
If $i \neq k$, analogous to the proof of \prettyref{eq:a_b_indp}, we have that
$$
\prob{ a_i \ge \tau , b_k \ge \tau+ 4 \sqrt{nq\Delta} }
\le \prob{a_i \ge \tau} \prob{b_k \ge \tau+ 4 \sqrt{nq\Delta} - 1}
\le \alpha \prob{b_k \ge \tau+ 4 \sqrt{nq\Delta} - 1}.
$$
If $i =k$, then we have that
$$
\prob{ a_i \ge \tau , b_k \ge \tau+ 4 \sqrt{nq\Delta} }
\le  \prob{b_k \ge \tau+ 4 \sqrt{nq\Delta} - 1}.
$$
Hence, for both cases, it reduces to proving 
\begin{align}
\prob{b_k \ge \tau+ 4 \sqrt{nq\Delta} - 1}
\le O\left(\alpha e^{-\Delta/2} \right) \label{eq:a_b_tail_desired}.
\end{align}

In view of \prettyref{lmm:ZS}, we have that
\begin{align*}
\prob{ b_k \ge \tau + 4 \sqrt{nq\Delta}-1}
&\le Q \left( \sqrt{2(n-1) D \left( \frac{\tau+ 4\sqrt{nq\Delta} -2}{n-1} \| q \right)} \right).
\end{align*}
In view of \prettyref{eq:tau_upper_lower}, we have
$
\tau \ge (n-1)q + \omega t,
$
where $\omega \triangleq \sqrt{(n-1)q(1-q)}$ and $t\triangleq Q^{-1}(\alpha)$.
Let $\eta\triangleq 4\sqrt{nq\Delta} -2$. 
Thus,
\begin{align*}
\sqrt{ 2(n-1) D \left( \frac{\tau+ \eta }{n-1} \| q \right)}
&\ge \sqrt{2(n-1)D \left( q + \frac{ \omega t +\eta}{n-1} \|q \right)} \\
&\ge \frac{\omega t+ \eta}{\sqrt{ \left( (n-1)q+ \omega t + \eta \right)(1-q) }} \\
& \ge \frac{\omega t+ \eta}{\sqrt{ \omega^2 + \omega t + \eta }},
\end{align*}
where the second inequality follows from \prettyref{eq:Divergence_lower}.
Combining the last two displayed equations gives
\begin{align}
\prob{ b_k \ge \tau + 4 \sqrt{nq\Delta}-1}
\le Q \left(  \frac{\omega t+ \eta}{\sqrt{ \omega^2 + \omega t + \eta }} \right)
= Q(t r), \label{eq:a_b_tail_step_1}
\end{align}
where 
$$
r \triangleq \frac{\omega+ \eta/t}{\sqrt{ \omega^2 + \omega t + \eta }}.
$$
By the assumption $nq \ge C_0 \Delta^2$, $\Delta \ge C_0$, and 
$t \le \sqrt{2 \log (nq)}$, 
we have $\eta \le \omega^2/2$, $\eta \ge 4 t^2$,
$\eta^2 \ge 4 \omega t^3$, and $t \le \omega/2$.
Thus, we get that 
\begin{equation}
r^2 \ge \frac{\omega^2 + \eta^2 /t^2 }{\omega^2 + \omega t + \eta}
= 1 + \frac{ \eta^2/t^2 - \omega t - \eta}{\omega^2 + \omega t+\eta} 
\ge 1 + \frac{\eta^2}{4 \omega^2 t^2} .
\label{eq:rrr}
\end{equation}
In view of the upper bound in \prettyref{lmm:Q_function_bound_2}, we have
\begin{align}
Q(tr) \le \left( \sqrt{2\pi} \frac{1+t^2}{t} \right)^{r^2-1} \frac{t^2+1}{r t^2} Q(t)^{r^2} 
\leq t^{c_1 (r^2-1)} \alpha^{r^2},
\label{eq:a_b_tail_step_2}
\end{align}
for a constant $c_1>0$, 
where the last inequality holds because $r > 1$ by \prettyref{eq:rrr} and 
$t=Q^{-1}(\alpha) \ge Q^{-1}(\alpha_1)$
under the assumption $\alpha \le \alpha_1$
for a sufficiently small constant $\alpha_1$.

Note that 
$$
r \le 1 + \frac{\eta}{\omega t} \le 1+ \frac{c_2}{t} \sqrt{\Delta} 
$$
for a constant $c_2>0$.
Therefore, 
\begin{align}
t^{c_1 (r^2-1) } \le t^{ c_1( 2c_2\sqrt{\Delta}/t+c_2^2\Delta/t^2)}  \le e^{\Delta/2}, 
\label{eq:a_b_tail_step_3}
\end{align}
where the last inequality holds 
because $t \ge Q^{-1}(\alpha_1)$ for sufficiently small constant $\alpha_1$.

Finally, it remains to bound $\alpha^{r^2}$. 
Using \prettyref{eq:rrr}, we have
\begin{align}
\alpha^{r^2} \le \alpha \exp \left( - \frac{\eta^2}{4\omega^2 t^2} \log \frac{1}{\alpha} \right)
\le \alpha \exp \left( - \frac{\eta^2}{8\omega^2 } \right) \le 
\alpha \exp (-\Delta), \label{eq:a_b_tail_step_4}
\end{align}
where the second inequality holds due to $t^2 \le 2 \log \frac{1}{\alpha}$
and the last inequality holds because $\eta^2 \ge 8 \omega^2 \Delta$.

In conclusion, by combining 
\prettyref{eq:a_b_tail_step_1}, \prettyref{eq:a_b_tail_step_2}, \prettyref{eq:a_b_tail_step_3}, and \prettyref{eq:a_b_tail_step_4},
we get the desired \prettyref{eq:a_b_tail_desired}.
\qed
\end{proof}

\begin{proof}[Proof of \prettyref{lmm:deg_Theta_i_cond}]
Recall that 
\begin{align*}
\Theta_i = \left\{  \max\{ \sqrt{a_i - c_{ii}}, \sqrt{b_i - c_{ii} } \}
 \leq  \sqrt{nq (1-s) } + \sqrt{ \Delta }  
 \right\}.
\end{align*}
Thus, 
\begin{align*}
&\left\{ a_i \ge \tau, b_i \ge \tau+1 \right\} \cap \Theta^c_{i} \\
& \subset \left\{ a_i \ge \tau, \sqrt{a_i-c_{ii} } > \sqrt{nq(1-s)} + \sqrt{\Delta} \right\}
\cup \left\{ b_i \ge \tau, \sqrt{b_i-c_{ii} } > \sqrt{nq(1-s)} + \sqrt{\Delta} \right\}.
\end{align*}
Hence, by the union bound and the symmetry between
$a_i$ and $b_i$, it suffices to prove 
\begin{align*}
\prob{ a_i \ge \tau , \sqrt{a_i-c_{ii} } > \sqrt{nq(1-s)} + \sqrt{\Delta} }
\le \alpha e^{-\Delta/2} + e^{- \Delta/(2\sigma^2)} . 
\end{align*}
Define 
$$
\overline{\tau} =\left(  \sqrt{nq} + \sqrt{ \frac{ \Delta }{4 (1-s) } } \right)^2. 
$$
Then 
$$
\left\{ a_i \ge \tau, \sqrt{a_i-c_{ii} } > \sqrt{nq(1-s)} + \sqrt{\Delta} \right\} \subset
\left\{ \tau \le a_i \le \overline{\tau}, \sqrt{a_i-c_{ii} } > \sqrt{nq(1-s)} + \sqrt{\Delta}  \right\}
\cup  \left\{ a_i \ge \overline{\tau} \right\} 
$$
and hence
$$
\prob{ a_i \ge \tau , \sqrt{a_i-c_{ii} } > \sqrt{nq(1-s)} + \sqrt{\Delta} }
\le \prob{\tau \le a_i \le \overline{\tau}, \sqrt{a_i-c_{ii} } > \sqrt{nq(1-s)} + \sqrt{\Delta} } 
+\prob{ a_i \ge \overline{\tau}}. 
$$
Since conditional on $a_i=k$, $a_i - c_{ii} \sim \Binom(k, 1-s)$, it follows that
\begin{align*}
& \prob{\tau \le a_i \le \overline{\tau} , \sqrt{a_i-c_{ii} } > \sqrt{nq(1-s)} + \sqrt{\Delta} } \\
& =\sum_{ \tau \le k \le \overline{\tau} } \prob{ a_i =k} \prob{ \sqrt{\Binom(k, 1-s)} > \sqrt{nq(1-s)} + \sqrt{\Delta}} \\
& \le \prob{ \tau \le a_i \le \overline{\tau} }  \prob{ \sqrt{\Binom( \overline{\tau} , 1-s)} > \sqrt{nq(1-s)} + \sqrt{\Delta}}\\
& \le \alpha \exp \left( - 2 \left(  \sqrt{nq(1-s)} + \sqrt{\Delta} - \sqrt{ \overline{\tau} (1-s) }  \right)^2 \right)
= \alpha e^{-\Delta/2}, 
\end{align*}
where the last inequality holds because of
the definition of $\tau$ in \prettyref{eq:tau} 
and the binomial tail bound \prettyref{eq:bintail_upper}.
Moreover,  since $a_i\sim \Binom(n-1,q)$, it follows from the binomial tail bound \prettyref{eq:bintail_upper} that 
\begin{align*}
\prob{ a_i \ge \overline{\tau} } 
\le \exp \left( - 2 \left(  \sqrt{\overline{\tau} } - \sqrt{ (n-1)q } \right)^2 \right)
\le e^{- \frac{\Delta}{2(1-s) } } =  e^{- \Delta / (2 \sigma^2)  } .
\end{align*}
Combining the last three displayed equation completes the proof.
\qed
\end{proof}

\subsection{Proof of \prettyref{thm:sparse}}
\label{sec:pf-sparse}

The following classical result about \ER graphs (cf.~\cite[Lemma 30]{BordenaveLelargeMassoulie:2015dq}) gives an upper bound on the probability that 
the $2$-hop neighborhood of a given vertex $i$ in $G\sim \calG(n,p)$ is tangle-free, i.e., containing at most one cycle.  This result will be used to control the dependency
among outdegrees in analyzing the 
$W$ similarity defined in \prettyref{eq:def_W_3_hop}.

\begin{lemma}\label{lmm:tangle_free}
Consider graph $G \sim \calG(n,p)$
with $np \ge C \log n$ for a large constant $C$.
Let $\calH$ denote the event that all $2$-hop neighborhoods in $G$  are tangle-free.
Then 
$$
\prob{\calH} 
\ge  1 -  n (2np)^{8} p^2 - n^{-1}.
$$
In particular, when $C \log n \le np \le n^{1-\epsilon}$ for $\epsilon>9/10$, 
$
\prob{\calH}  \ge 1- O\left(n^{9-10\epsilon} \right).
$
\end{lemma}
\begin{proof}
Let $\calH_i$ denote the event that 
the $2$-hop neighborhood of the vertex $i$ in $G$  
is tangle-free and let $\calH= \cap_{ i \in [n]} \calH_i$. Let $\ell=2$ throughout the proof. 
Consider the classical graph branching process
to explore the vertices in the $\ell$-hop neighborhood of $i$. See, \eg, \cite[Section 11.5]{AlonSpencer08} for a reference. Such 
a branching process discovers a set of edges which form a spanning tree 
of the $\ell$-hop neighborhood of $i$.
Then the $\ell$-hop neighborhood of $i$ is tangle-free, provided that the number of edges undiscovered by the branching process 
is at most one.

Let $m$ denote the size of the $\ell$-hop neighborhood of $i$ in graph $G\sim \calG(n,p)$. 
There are at most $\binom{m}{2}$ pairs of two distinct 
vertices in the $\ell$-hop neighborhood of $i$. Hence, the number of undiscovered edges 
is stochastically  dominated by $\Binom( \binom{m}{2} ,p)$. Thus, conditional on 
the size of the $\ell$-hop neighborhood of $i$
being $m$, the probability of $\calH_i^c$, by a union bound, is at most
$$
\prob{ \Binom(m(m-1)/2, p) \ge 2} \le\frac{1}{8} m^4 p^2,
$$

Moreover, since $np \ge C \log n$ for a large constant $C$, the maximum degree in $G$ is at most $2np$ with probability at least $1-n^{-2}$. Thus, $m \le (2np)^\ell$ 
with probability at least $1-n^{-2}$.
Therefore, the unconditional probability
$$
\prob{\calH_i^c } \le \frac{1}{8} (2np)^{4\ell} p^2
+ n^{-2}
\le (2np)^{4\ell} p^2 + n^{-2}. 
$$
The proof is complete by applying a union bound over $i \in [n]$ to the last display.
\qed
\end{proof}

Recall that $\tilde{N}_A(i)$ (resp.\ $\tilde{N}_B(i)$) denote the set of vertices in the $2$-hop neighborhood of $i$ in graph $A$ (resp.\ $B$). 
Let $\tilde{G}_A(i)$ (resp.\ $\tilde{G}_B(i)$) denote the $2$-hop neighborhood of $i$ in graph $A$ (resp.\ $B$), \ie, the subgraph induced by 
$\tilde{N}_A(i)$ (resp.\ $\tilde{N}_B(i)$). 
For notational simplicity, we use the same notation $a_j^{(i)}$ and $b_j^{(i)}$ as~\prettyref{eq:degmod1} and~\prettyref{eq:degmod2}  
for unnormalized outdegrees $\left| N_A(j) \setminus N_A[i] \right| $ and $\left| N_B(j) \setminus N_B[i] \right| $, respectively.  
Similar to the high-probability events defined in the beginning of \prettyref{sec:pf}, we also need to condition on a number of events regarding the $2$-hop neighborhoods of $i$ in $A$ and $k$ in $B$ in analyzing the $W$ statistic. 

First, for each $i \in [n]$, 
define the event $\Gamma_A(i)$ such that the following statements hold simultaneously:
\begin{align*}
   \frac{nq}{2} \le a_i& \le 2n q  \\
   \frac{nq}{2}  \le a_j^{(i)} & \le 2nq, \quad \forall j \in N_A(i)  \\
    \tilde{a}_i & \le (2nq)^2. 
 \end{align*}
 Similarly, define the event $\Gamma_B(i)$  such that the following statements hold simultaneously:
\begin{align*}
   \frac{nq}{2} \le b_i & \le 2n q  \\
   \frac{nq}{2}  \le b_j^{(i)} & \le 2nq, \quad \forall j \in N_B(i)  \\
    \tilde{b}_i & \le (2nq)^2. 
 \end{align*}
 Define the event $\Gamma_{ii}$ such that the following statements hold simultaneously:
\begin{align*}
    c_{ii} & \ge \frac{nq}{2}  \\
    c_{j}^{(i)} & \ge  \frac{nq}{2} , \quad \forall j \in N_A(i)  \cap N_B(i)  \\
    \sqrt{a_i - c_{ii}}, \sqrt{b_i- c_{ii} } & \le \sqrt{nq(1-s)} + \sqrt{2\log n},
 \end{align*}
 where
\begin{equation}
c_{j}^{(i)}\triangleq \left|  \left( N_A(j) \setminus N_A[i] \right) \cap \left( N_B(j) \setminus N_B[i]  \right)    \right|.
\label{eq:cjj}
\end{equation} 
Under the assumptions that $nq \ge C \log n$ for some sufficiently large constant $C$, 
and $\sigma \le \sigma_0$ for sufficiently small constants $\sigma_0$, 
using Chernoff bounds for binomial distributions \prettyref{eq:bintail} and the union bound, we have
$
\prob{\Gamma_A^c(i) }, \prob{\Gamma_B^c(i)} , \prob{\Gamma^c_{ii}} \le O(n^{-2}).
$




Second, for each pair of $i, k \in [n]$ with $i \neq k$, define 
the event $\Gamma_{ik}$ such that the following statement holds:
\begin{align*}
| N_A(i) \cap N_B(k) |  & \le 2 \\
| N_A(j) \cap \tilde{N}_B(k)  | & \le 2, \quad \forall j \in N_A(i) \setminus N_B[k] \\
| N_B(j) \cap \tilde{N}_A(i)  | & \le 2, \quad \forall j \in N_B(k) \setminus N_A[i].
\end{align*}
\begin{lemma}
\label{lmm:Gammaik}	
If $1 \le nq \le n^{1-\epsilon}$ for $\epsilon>9/10$, 
we have
$
\prob{  \Gamma_{ik} ^c} 
\le O\left(n^{7} q^{10} \right) 
= O\left( n^{7-10 \epsilon} \right)
$
for all $i \neq k$.
\end{lemma}
\begin{proof}
Fix $i \neq k$. Note that 
\begin{align*}
\prob{ |N_A(i) \cap N_B(k) | \ge 3 } 
&= \prob{ \exists a, b, c \in [n]: A_{ia}=A_{ib}=A_{ic}=1, B_{ka}=B_{kb}=B_{kc}=1} \\
&\le \sum_{a, b, c\in [n]}\prod_{j \in \{a,b,c\}} \prob{A_{ij}=1} \prob{B_{kj}=1} 
\le n^3 q^6 \le n^7q^{10}. 
\end{align*}
Next, suppose we are given any $j \in N_A(i) \setminus N_B[k]$
such that $|N_A(j) \cap \tilde{N}_B(k) | \ge 3$. 
Let $a, b, c$ denote three distinct vertices in $N_A(j) \cap \tilde{N}_B(k) $. 
For each $j' \in \{a,b,c\}$, 
let $p(j')$ denote a vertex in $N_B(k) \cap N_B[j']$ (which is non-empty since $j' \in \tilde N_B(k)$).
Consider the subgraph $S$ of the union graph $A \cup B$ 
induced by vertices in $\{i, j, k, a, b, c, p(a), p(b), p(c)\}$. 
Let $V(S)$ denote the set of distinct vertices in $S$
and $v(S)=|V(S)|$. 
Let $e(S)$ denote the number of edges in $S$. 
Note that $v(S) \le 9$. 
Also, if we delete the two edges $(j,a)$
and $(j,b)$, the graph $S$ is still connected;
thus  $e(S)-v(S) \ge1$. 
Therefore, by letting $\calK_n$ denote the complete graph on $[n]$
and noting that $A\cup B \sim \calG\left(n, q(2-s)\right)$, 
\begin{align*}
&\prob{\exists j \in N_A(i) \setminus N_B[k]: |N_A(j) \cap \tilde{N}_B(k) | \ge 3 }  \\
& \leq \prob{ \exists S \subset A \cup B: v(S) \le 9, e(S) - v(S) \ge 1  } \\
&\le \sum_{ v \le 9 }  \sum_{S \subset \calK_n: v(S) =v }
\indc{ (i, k) \in V(S) } \indc{e(S) - v(S) \ge 1} \prob{ S \subset A \cup B}  \\
& \le   \sum_{v \le 9 } 2^{\binom{v}{2}}  n^{v-2} (2q)^{v+1}  \le   O \left( n^7 q^{10} \right).
\end{align*}
 Similarly, we have $\prob{\exists j \in N_B(k) \setminus N_A[i]: |N_B(j) \cap \tilde{N}_A(i) | \ge 3 }  \le O \left( n^7 q^{10} \right)$ and hence $
\prob{\Gamma_{ik}^c} \le O\left(n^{7} q^{10} \right).
$ 
\end{proof}

Third, let $A \cup B$ denote the union graph of $A$ and $B$. Define
$$
\calH_{ii} = \left\{ \tilde G_{A \cup B} (i) \text{ is tangle-free}  \right\}
$$
and 
$$
\calH_{ik} = \left\{ \tilde G_A(i) \text{ and } \tilde G_B(k) \text{ are both tangle-free} \right\}.
$$

The next two lemmas are the counterparts of 
\prettyref{lmm:true} and \prettyref{lmm:fake}, which 
establish the desired separation of the $W$ statistic for true pairs and fake pairs.

\begin{lemma}[True pairs]\label{lmm:true_3_hop}
Assume that $nq \ge C \max\{ \log n, L^2 \}$, $L \ge L_0$ for some sufficiently large constants 
$C$ and $L_0$, $\sigma \le \sigma_0/L$ for some sufficiently small constant $\sigma_0>0$,
and $n^2 q^3 \sqrt{L} \le c_0$ for some  sufficiently small constant $c_0>0$. 
Then 
\begin{align}
\prob{ W_{ii} \le nq/4  \mid \tilde G_A(i) , \tilde G_B(i), \tilde G_{A \cup B}(i)  } \indc{\calH_{ii} \cap \Gamma_A(i) \cap \Gamma_B(i) \cap \Gamma_{ii} } \le e^{-\Omega(nq) }.
\end{align}
\end{lemma}
\begin{proof}
Throughout the proof, we condition on the  $2$-hop neighborhoods 
of $i$ in $A$, $B$, and $A \cup B$  such that event $\calH_{ii} \cap \Gamma_A(i) \cap \Gamma_B(i) \cap \Gamma_{ii}$ holds.

On the event $\calH_{ii}$, there is at most one cycle in the $2$-hop neighborhood of $i$ in the union graph $A \cup B$. 
Hence, there is at most one pair of vertices $j_0  \in N_A(i)$ and $j_0' \in N_B(i) $ with $j_0 \neq j_0'$ such that in the union graph 
$A \cup B$,
\begin{enumerate}[(a)]
\item either $j_0$ and $j_0'$ are adjacent; 
\item or there exist a neighbor $\ell \neq i $ of $j_0$
and a neighbor $\ell' \neq i$ of $j'_0$, where either
$\ell=\ell'$ or $\ell$ and $\ell'$ are adjacent.
\end{enumerate}
Then we claim that 
$\tilde{Z}^{(ii)}_{jj}$ 
are mutually independent across different $j$ in  $N_A(i) \cap N_B(i) \setminus \{j_0, j_0' \}$. 
Indeed, note that $\tilde{Z}^{(ii)}_{jj}$ is a function of 
$\{ \tilde{a}_\ell^{(i)}: \ell \in N_A(j) \setminus N_A[i] \}$ 
and $\{ \tilde{b}_\ell^{(i)}: \ell \in N_B(j) \setminus N_B[i] \}$. 
Fix a pair of $j \neq j' \in N_A(i) \cap N_B(i) \setminus \{j_0, j_0' \}$ and 
any $\ell \in ( N_A(j) \setminus N_A[i]) \cup ( N_B(j) \setminus N_B[i])$ 
and any $\ell' \in ( N_A(j') \setminus N_A[i]) \cup ( N_B(j') \setminus N_B[i])$.
First, we claim $\ell \neq \ell'$, and $\ell, \ell'$ are non-adjacent in the union graph $A \cup B$; 
otherwise, $(j,j')$ is another pair in addition to $(j_0,j_0')$ satisfying either the condition (a) or (b) mentioned above,
violating the tangle-free property.
Moreover, since we have excluded $i$'s closed $2$-hop neighborhoods
in the definition of outdegree $\tilde{a}_\ell^{(i)}$ and  $\tilde{b}_\ell^{(i)}$, 
it follows that 
 $(\tilde{a}_\ell^{(i)}, \tilde{b}_\ell^{(i)} ) $ is independent
from $(\tilde{a}_{\ell'}^{(i)}, \tilde{b}_{\ell'}^{(i)} ) $. 
Thus, $\tilde{Z}^{(ii)}_{jj}$ and $\tilde{Z}^{(ii)}_{j'j'}$ are independent.


By the definition of $W$ similarity in \prettyref{eq:def_W_3_hop}, we have   
$$
W_{ii} \ge \sum_{j \in N_A(i) \cap N_B(i) \setminus \{ j_0 \} } \indc{ \tilde{Z}^{(ii)}_{jj} \le \eta},
$$
where $\eta=\eta_0 \sqrt{ \frac{L}{nq}}$ as defined in \prettyref{eq:eta-threshold}.
We claim that 
\begin{align}
\prob{ \tilde{Z}^{(ii)}_{jj} \le \eta } \ge 1-  e^{-\Omega(L) }  \ge \frac{3}{4}, \label{eq:Z_ii_3_hop}
\end{align}
where the last inequality holds due to $L \ge L_0$. 
Also, on the event $\Gamma_{ii}$, $c_{ii} = | N_A(i) \cap N_B(i)| \ge nq/2$.
Then it follows from the independence of $\tilde{Z}^{(ii)}_{jj}$ across
different $j  \in N_A(i) \cap N_B(i) \setminus \{j_0 \}$  that 
$$
W_{ii} \overset{s.t.}{\ge}  \Binom\left( \frac{nq}{2} -1,  \frac{3}{4} \right).
$$
Therefore, by Chernoff's bound \prettyref{eq:bintail} for binomials, we get that 
$$
\prob{ W_{ii} \le nq/4 } \le e^{-\Omega(nq) }. 
$$

\nb{
It remains to verify claim \prettyref{eq:Z_ii_3_hop}. 
The proof follows the similar argument as the proof of \prettyref{lmm:true}. 
Specifically, recall that $\tilde{Z}^{(ii)}_{jj}= d\left( \tilde{\mu}_j^{(i)} , \tilde{\nu}_j^{(i)} \right)$,
where 
\begin{equation*}
\tilde{\mu}^{(i)}_j \triangleq  \frac{1}{ a^{(i)}_j } \sum_{ \ell \in N_A(j) \setminus N_A[i] } 
\delta_{ \tilde{a}_\ell^{(i)}} \; - \Binomc\left(n-\tilde{a}_i,q \right),
\end{equation*}
and
 \begin{equation*}
\tilde{\nu}^{(i)}_{j} \triangleq  \frac{1}{b^{(i)}_j } \sum_{ \ell \in N_B(j) \setminus N_B[i] } 
\delta_{ \tilde{b}_\ell^{(i)} } \; - \Binomc\left(n-\tilde{b}_i, q \right).
\end{equation*}
Recall that  $\tilde{a}_\ell^{(i)}$ (resp.~$\tilde{b}_\ell^{(i)}$) are the normalized ``outdegree'' of  vertex $\ell$ with the closed $2$-hop neighborhood of $i$ in $A$ (resp.~$B$) excluded;
$\tilde{a}_i$ (resp.~$\tilde{b}_i$) are the size the 2-hop neighborhood of $i$ in graph $A$ (resp.~$B$).

Note that for $\ell \in  N_A(j) \setminus N_A[i] $,
\begin{align*}
\tilde{a}_\ell^{(i)} 
&= \frac{1}{ \sqrt{  (n - \tilde{a}_i ) q (1-q)  } } \sum_{k \in \tilde{N}_A(i)^c } \left( A_{k \ell} - q  \right)  \\
&=\frac{1}{ \sqrt{  (n - \tilde{a}_i ) q (1-q)  } } 
\left[ \sum_{k \in \tilde{N}_A(i)^c  \cap \tilde{N}_B(i)^c }  A_{k \ell}  -  \left( n - \tilde{a}_i \right)q \right],
\end{align*}
where the last equality holds because  if $k \in  \tilde{N}_B(i)$, then $A_{k\ell} = 0$; otherwise, 
$\tilde{G}_{A \cup B} (i)$ is not  tangle-free. Moreover, note that $\ell \notin N_B(i)$; otherwise
$\tilde{G}_{A \cup B} (i)$ is not  tangle-free. Therefore, for all $k \in \tilde{N}_A(i)^c  \cap \tilde{N}_B(i)^c$,
$A_{k\ell} \sim \Bern(q)$. Hence, 
$$
\tilde{a}_\ell^{(i)}  \iiddistr  \frac{1}{ \sqrt{  (n - \tilde{a}_i ) q (1-q)  } }  
\left[ \Binom\left( \left| \tilde{N}_A(i)^c  \cap \tilde{N}_B(i)^c  \right|, q  \right) -  \left( n - \tilde{a}_i \right)q  \right] \triangleq \mu.
$$
Similarly,  for $\ell \in  N_B(j) \setminus N_B[i] $,
\begin{align*}
\tilde{b}_\ell^{(i)} 
&= \frac{1}{ \sqrt{  (n - \tilde{b}_i ) q (1-q)  } } \sum_{k \in \tilde{N}_B(i)^c } \left(B_{k \ell} - q \right)  \\
&=\frac{1}{ \sqrt{  (n - \tilde{b}_i ) q (1-q)  } } 
 \left[ \sum_{k \in \tilde{N}_B(i)^c  \cap \tilde{N}_A(i)^c } B_{k \ell} -   \left( n - \tilde{b}_i \right) q  \right] .
\end{align*}
Thus, 
$$
\tilde{b}_\ell^{(i)}  \iiddistr  \frac{1}{ \sqrt{  (n - \tilde{b}_i ) q (1-q)  } }  
\left[ \Binom\left( \left| \tilde{N}_A(i)^c  \cap \tilde{N}_B(i)^c  \right|, q  \right)  - \left(n - \tilde{b}_i \right) q \right] \triangleq \mu'.
$$

Analogous to \prettyref{eq:muPQ1} and \prettyref{eq:muPQ2}, the centered empirical distribution can be rewritten as 
\begin{align*}
\tilde{\mu}^{(i)}_j = & ~ 
\rho P  + (1-\rho) P'  + \mu - \nu  \\
\tilde{\nu}^{(i)}_{j}  = & ~ \rho' Q + (1-\rho') Q'  + \mu' - \nu',
\end{align*}
where 
\begin{align*}
\rho \triangleq  \frac{c_{j}^{(i)}   }{a^{(i)}_j }  , \quad \rho'\triangleq \frac{c^{(ii)}_{jj}}{b_j^{(i)}}, 
\end{align*}
and
\begin{align*}
P \triangleq  & ~ \frac{1}{c_{j}^{(i)}  } \sum_{ \ell \in  \left( N_A(j) \setminus N_A[i] \right) \cap \left( N_B(j) \setminus N_B[i]  \right)   }  
\delta_{ \tilde{a}_\ell^{(i)}} - \mu, \quad P' ~\triangleq  \frac{1}{a^{(j)}_i -c_{j}^{(i)} } \sum_{ \ell  \in   \left( N_A(j) \setminus N_A[i] \right) \setminus \left( N_B(j) \setminus N_B[i]  \right)   }   \delta_{ \tilde{a}_\ell^{(i)}}- \mu, \\
Q \triangleq  & ~ \frac{1}{c_{j}^{(i)} } \sum_{  \ell \in  \left( N_A(j) \setminus N_A[i] \right) \cap \left( N_B(j) \setminus N_B[i]  \right)     } \delta_{ \tilde{b}_\ell^{(i)} }  - \mu', \quad Q'~\triangleq  \frac{1}{b_j^{(i)} -c_{j}^{(i)}} \sum_{ \ell \in  \left( N_B(j) \setminus N_B[i]  \right) \setminus \left( N_A(j) \setminus N_A[i] \right) }    \delta_{ \tilde{b}_\ell^{(i)} }- \mu',
\end{align*}
and $\nu= \Binomc\left(n-\tilde{a}_i,q \right)$ and $\nu'=\Binomc\left(n-\tilde{b}_i, q \right)$.

Similar to \prettyref{eq:Zii1}, we have that 
\begin{align}
\tilde{Z}^{(ii)}_{jj}
\leq \underbrace{ \| [ \mu - \nu ]_L \|_1 +\| [ \mu' - \nu' ]_L \|_1 }_{\rm (I)}
+\underbrace{d(P,Q)}_{\rm (II)} + 
\underbrace{ (1-\rho) \|[P']_L\|_1 + (1-\rho') \|[Q']_L\|_1}_{\rm (III)} + 
\underbrace{|\rho-\rho'|  \times \|[Q]_L\|_1}_{\rm (IV)}
.  \label{eq:Zii1_3_hop}
\end{align}

For (I), we need the following lemma to control the discrepancy between the distribution $\mu$ (resp.\ $\mu'$)
and the ideal standardized binomial distribution $\nu$ (resp.\ $\nu'$). 
\begin{lemma}\label{lmm:distri_discrepancy}
Let $m, n \in \naturals$ with $m\le n$ and $\eta_1, \ldots, \eta_m, q \in [0,1]$. 
Suppose $X_i \iiddistr \Bern(q)$ for $1 \le i \le n$
and $Y_i$'s are independently distributed as $\Bern(\eta_i)$ for $1 \le i \le m$. 
Let $S=\sum_{i=1}^n X_i$ and $T= \sum_{i=1}^m Y_i + \sum_{i=m+1}^n X_i$. 
Let $\mu_0$ and $\nu_0$  denote the law of $ \frac{S- nq}{\sqrt{nq(1-q)}}$  and 
$\frac{T- nq}{\sqrt{nq(1-q)}}$, respectively. 
Assume $m \le n/2$ and $nq=\Omega(1)$. Then 
\begin{align}
d\left( \mu_0, \nu_0 \right) = \left\|  \left[\mu_0-  \nu_0 \right]_L \right\|_1 \le  O \left( L \frac{ \sum_{i=1}^m |\eta_i -q| }{\sqrt{nq} } \right) .
\end{align}
\end{lemma}
\begin{proof}
For $1 \le i \le m$, we couple $X_i$ and $Y_i$ as follows. 
When $\eta_i \le q$, generate $Y_i \sim \Bern(\eta_i)$, and let $X_i=1$ if $Y_i=1$ and $X_i \sim \Bern(q-\eta_i)$ if $Y_i=0$.
When $\eta_i >	q$, generate $X_i \sim \Bern(q)$, and  let $Y_i=1$ if $X_i=1$ and $Y_i \sim \Bern(\eta_i-q)$ if $X_i=0$. 
Let $X=\sum_{i=m+1}^{n} X_i$, 
$Y=\sum_{i=1}^m Y_i$, and $Z= \sum_{i=1}^m X_i$. 
Then $S= X+Z$
and $T=X+Y$. 
Let $\xi=\sqrt{nq(1-q)}$. 
Then 
\begin{align*}
d\left( \mu_0, \nu_0 \right) 
&= \sum_{\ell=1}^L \left| \mu_0 ( I_\ell) - \nu_0( I_\ell)  \right| \\
&= \sum_{\ell=1}^L \left|  \prob{S \in \xi I_\ell + nq } - \prob{T \in \xi I_\ell + nq }   \right|\\
& \le \sum_{\ell=1}^L \max\left\{  \prob{ S \in \xi I_\ell + nq, T \notin \xi I_\ell +nq } , 
 \prob{ S \notin \xi I_\ell + nq , T \in \xi I_\ell + nq } \right\}.
\end{align*}
It remains to show $\prob{ S \in \xi I_\ell + nq, T \notin \xi I_\ell +nq } \le 
O \left(  \frac{ \sum_{i=1}^m | \eta_i-q | }{\sqrt{nq} } \right)$; the proof for 
$\prob{ S \in \xi I_\ell + nq, T \notin \xi I_\ell +nq }$ is analogous. 
Note that 
\begin{align*}
\prob{ S \in \xi I_\ell + nq, T \notin \xi I_\ell+ nq }
=\prob{ X \in \xi I_\ell+ nq - Z , X \notin \xi I_\ell + nq- Y}
\le O\left( \frac{\expect{|Y-Z|} }{\sqrt{nq}} \right),
\end{align*}
where the last inequality follows analogous to \prettyref{lmm:beta}. 
The conclusion follows since $\expect{|Y-Z|} \le  \sum_{i=1}^m \expect{|X_i-Y_i|}  =(1- \min\{ \eta_i, q \} ) \sum_{i=1}^m | \eta_i-q| $ by definition.
 \qed
\end{proof}

Applying \prettyref{lmm:distri_discrepancy} (with $\eta_i\equiv 0$ and $m \leq \tilde a_i+\tilde b_i$) and noting that $\tilde{a}_i, \tilde{b}_i \le (2nq)^2$, we get that 
\begin{align}
\| [ \mu - \nu ]_L \|_1 +\| [ \mu' - \nu' ]_L \|_1
\le O \left( \frac{L(nq)^2q}{ \sqrt{nq} } \right) .
\end{align}
Analogous to \prettyref{lmm:beta}, under the assumptions that $\sigma \le \sigma_0/L$ and $nq \ge CL^2$, we have that 
for any $ \ell \in  \left( N_A(j) \setminus N_A[i] \right) \cap \left( N_B(j) \setminus N_B[i]  \right)$
and any interval $I \subset [-1/2,1/2]$ with $|I| =1/L$, conditional on the $2$-hop neighborhoods 
of $i$ in both $A$ and $B$,  
      \begin{align*}
 \prob{ \tilde{a}_\ell^{(i)} \in I, \tilde{b}_\ell^{(i)} \notin I  } + \prob{ \tilde{a}_\ell^{(i)} \notin I, \tilde{b}_\ell^{(i)}  \in I  }  \leq \frac{c_1}{L}
      \end{align*}
      for a sufficiently small constant $c_1$. 
      
For (II), applying \prettyref{lmm:corsample} with 
$$
\{X_j\}_{j=1}^m =
\{ \tilde{a}_\ell^{(i)} \}_{ \ell \in  \left( N_A(j) \setminus N_A[i] \right) \cap \left( N_B(j) \setminus N_B[i]  \right)  }, 
\quad 
\{ Y_j\}_{j=1}^m =\{ \tilde{b}_\ell^{(i)} \}_{  \ell \in  \left( N_A(j) \setminus N_A[i] \right) \cap \left( N_B(j) \setminus N_B[i]  \right) },
$$ 
and $m=c_{j}^{(i)} \ge \frac{nq}{2}$ on the event $\Gamma_{ii}$, we get that  with probability at least $1-e^{-\Omega(L) }$,
    \begin{equation} 
d(P,Q) \leq   c_2 \sqrt{\frac{L}{nq} } ,  \label{eq:true1_3_hop}
\end{equation} 
 for a sufficiently small constant $c_2$. 

For (III), 
applying \prettyref{lmm:tvconc} with $k=L$ implies that 
$ \|[P']_L\|_1 \leq 2 \sqrt{\frac{L}{a^{(i)}_j- c_{j}^{(i)}  }}  $
and $\|[Q']_L\|_1 \leq  2 \sqrt{\frac{L}{ b^{(i)}_j- c_{j}^{(i)} } }$,
each with probability at least $1-e^{-L/2}$.
Therefore, by the union bound, with probability at least $1-e^{-\Omega(L)}$,
\begin{align}
(1-\rho) \|[P']_L\|_1 + (1-\rho') \|[Q']_L\|_1
& \le \frac{2}{a^{(i)}_j }  \sqrt{L} \sqrt{ a^{(i)}_j- c_{j}^{(i)}  } + \frac{ 2 }{b^{(i)}_j }   \sqrt{L}   \sqrt{ b^{(i)}_j- c_{j}^{(i)}  } \nonumber  \\
&  \leq \frac{8}{nq}   \sqrt{L} \left( \sqrt{nq \sigma^2 } + \sqrt{2\log n} \right) ,
\label{eq:true2_3_hop}
\end{align}
where the last inequality holds because on the event $\Gamma_A(i) \cap \Gamma_B(i)\cap \Gamma_{ii}$, $a^{(i)}_j, b^{(i)}_j  \ge nq/2$,
\begin{align*}
\sqrt{ a^{(i)}_j- c_{jj}^{ (ii)}  } & \le  \sqrt{ a_j - c_j + b_i - c_i }  \le \sqrt{a_j - c_j} + \sqrt{ b_i - c_i } \le 2 \left(  \sqrt{nq(1-s) } + \sqrt{ 2\log n } \right)
\end{align*} 
and similarly for $\sqrt{ b^{(i)}_j- c_{j}^{(i)}  }$. 


Finally, for (IV), applying \prettyref{lmm:tvconc} with $k=L$ implies that 
with probability at least $1-e^{-L/2}$, 
$$
 \|[Q]_L\|_1 \leq 2 \sqrt{\frac{L}{c_{j}^{(i)}  }} \le 2\sqrt{2} \sqrt{\frac{L}{nq}},
 $$  
 where the last inequality holds due to $c_{j}^{(i)} \ge nq/2$ on event $\Gamma_{ii}$.  
Moreover,
$$
|\rho-\rho'| \le  \max \{1-\rho, 1-\rho' \}  
\le \frac{2}{nq} \left( \sqrt{nq(1-s) } + \sqrt{ 2\log n } \right)^2 \\
\le 4 \sigma^2 + 8 \frac{\log n}{nq} .
$$
Therefore, 
\begin{equation}
|\rho-\rho'|  \times  \|[Q]_L\|_1  
\le 8\sqrt{2}  \sqrt{\frac{L}{nq}} \left(  \sigma^2 + 2  \frac{\log n}{nq}  \right).
\label{eq:true3_3_hop}
\end{equation}

Assembling \prettyref{eq:Zii1_3_hop} with \prettyref{eq:true1_3_hop}, \prettyref{eq:true2_3_hop}, \prettyref{eq:true3_3_hop}, 
we get that with probability at least $1-e^{-\Omega(L)}$, 
\begin{align*}
\tilde{Z}^{(ii)}_{jj}
& \leq  
c_2 \sqrt{\frac{L}{nq} } 
+ \frac{8}{nq}   \sqrt{L} \left( \sqrt{nq \sigma^2 } + \sqrt{2 \log n} \right)+ 8\sqrt{2}  \sqrt{\frac{L}{nq}} \left(  \sigma^2 + 2  \frac{\log n}{nq}  \right) + O \left( \frac{L(nq)^2q}{ \sqrt{nq} } \right)  \\
& \leq  \eta_0 \sqrt{ \frac{L} {nq} } = \eta 
\end{align*}
for some sufficiently small absolute constant $\eta_0>0$, where the last inequality holds due to 
the assumptions that $nq \ge C L$ for some sufficiently large constant $C$, $\sigma\le \sigma_0/L$,
and $n^2 q^3 \sqrt{L} \le c_0$ for some  sufficiently small constant $c_0>0$. 
Thus we arrive at the desired \prettyref{eq:Z_ii_3_hop}.
}
\qed
\end{proof}

\begin{lemma}[Fake pairs]\label{lmm:fake_3_hop}
Suppose $L \ge C \log (nq)$, $nq \ge C \max\{ \log n, L^2 \}$ for some sufficiently large constant $C$,
and $q \le n^{-\epsilon}$ for $\epsilon> 9/10$. 
Fix $i \neq k$. 
Then 
\begin{align}
\prob{ W_{ik} \ge nq/4  \mid \tilde G_A(i) , \tilde G_B(k)  } 
\indc{ \calH_{ik} \cap \Gamma_A(i)  \cap \Gamma_B(k) \cap \Gamma_{ik} } \le e^{-\Omega \left(nq \right)}.
\end{align}
\end{lemma}
\begin{proof}
Fix a pair of vertices $i \neq k$ and condition on the $2$-hop neighborhoods 
of $i$ in $A$ and $k$ in $B$ such that the event $H_{ik} \cap \Gamma_A(i) \cap \Gamma_B(k) \cap \Gamma_{ik}  $ holds. 
Fix a feasible solution $M$ in \prettyref{eq:def_W_3_hop}; in other words, $M$ is a bipartite matching (possibly imperfect) between 
the neighborhoods ${N}_A(i)$ and ${N}_B(k)$. 

For the ease of notation, let $J=N_A(i) \setminus N_B[k]$ and $J'=N_B(k) \setminus N_A[i]$. 
Recall the matrix $Y^{(ik)}$ defined in \prettyref{eq:Y_3_hop}.
Since $M$ is a matching, it follows that 
$$
\iprod{Y^{(ik)}}{M} \le 2 |N_A [i]  \cap N_B [k] | + \sum_{j \in J , j' \in J' } Y^{(ik)}_{jj'} M_{jj'} \le \frac{nq}{8} + \sum_{j \in J , j' \in J' } Y^{(ik)}_{jj'} M_{jj'},
$$
where the last inequality holds 
because  $ |N_A[i] \cap N_B[k]| \le 4 \le nq/16$ on the event $\Gamma_{ik}$ under the assumption that $nq \ge C \log n$. 



Note that on the event $\calH_{ik}$, 
there is at most one cycle in the $2$-hop neighborhood
of $i$ in $A$, and at most one cycle in the $2$-hop neighborhood
of $k$ in $B$.

We next bound $\sum_{j \in J, j' \in J'} Y^{(ik)}_{jj'} M_{jj'}$
using McDiarmid's inequality, where $Y^{(ik)}_{jj'}=
\indc{\tilde{Z}^{(ik)}_{jj'} \le \eta}$ and $\eta=\eta_0 \sqrt{ \frac{L}{nq}}$ as defined in \prettyref{eq:eta-threshold}.
To circumvent the discontinuity of the indicator function, 
define a piecewise linear function $F$ which decreases linearly from $1$ to $0$ from $\eta$ to $2\eta$, so that 
$\indc{x \leq \eta} \leq F(x)$ for all $x$. 
 Furthermore, $F$ is Lipschitz with constant $1/\eta$. Define
\begin{equation}
W' \triangleq \sum_{j \in J, j' \in J'} F \left( \tilde Z^{(ik)}_{jj'} \right) M_{jj'}.
    \label{eq:WW2}
\end{equation}
Then we have
\begin{align}
 \iprod{Y^{(ik)}}{M}
 \le \frac{nq}{8} +  W'. \label{eq:Y_M_bound}
\end{align}
Let $\calL=\cup_{j \in J}  \left( N_A(j) \setminus N_A[i] \right)$ 
and $\calL'= \cup_{j' \in J'} \left( N_B(j') \setminus N_B[k] \right)$. 
Next we claim that, on the event $\calH_{ik} \cap \Gamma_A(i) \cap \Gamma_B(k)$, 
$W'$, 
as a function of $\{ 
( \tilde a_\ell^{(i)}, \tilde b_{\ell}^{(k)}): 
\ell \in \calL \cap \calL'\} $, 
$\{\tilde a_\ell^{(i)}: \ell \in \calL \setminus \calL' \}$, and 
$\{\tilde b_{\ell'}^{(k)}:\ell' \in \calL' \setminus \calL \}$, satisfies the bounded difference property with constant $O(\frac{1}{nq\eta})$.
This is verified by the following reasoning:
\begin{itemize}
    \item Fix $\ell \in \calL \backslash \calL'$.
    We consider the impact of modifying the value of $\tilde a_\ell^{(i)}$ on that of $W'$.
    On the tangle-free event $\calH_{ik}$, 
there are at most two distinct choices of $j$
such that $\ell \in N_A(j) \setminus N_A[i]$. Therefore 
$\tilde a_\ell^{(i)}$ appears in the empirical distribution $\tilde \mu_j^{(i)}$ for at most two different $j\in N_A(i)$.
Furthermore, since $\ell \notin \calL'$, $\tilde a_\ell^{(i)}$ does not appear in any $\tilde \nu_j^{(k)}$.
    Recall that any $m$-observation empirical distribution as a function of each observation satisfies the bounded difference property (with respect to the total variation distance) with constant $O(\frac{1}{m})$ 
    (cf.~\prettyref{eq:mcdia}). 
    On the event $\Gamma_A(i)$, we have 
    $a_j^{(i)} \ge  nq/4$.
    Thus modifying $\tilde a_\ell^{(i)}$ can change $\tilde \mu_j^{(i)}$ in total variation by at most $O(\frac{1}{nq})$.
    Furthermore, \textit{crucially, since $M$ is a matching}, for each $j$ there exists at most one $j'$ such that $M_{jj'} \neq 0$ in the double sum \prettyref{eq:WW2}.
    Finally, since $F$ is $(1/\eta)$-Lipschitz continuous by design, we conclude that 
    $\tilde a_\ell^{(i)} \mapsto W'$ has the desired bounded difference property with constant $O(\frac{1}{nq\eta})$.

    \item 
    Entirely analogously, since $b_j^{(k)} \ge  nq/4$ on the event $\Gamma_B(k)$,
		the mappings 
    $\tilde b_{\ell'}^{(k)} \mapsto W'$ for any $\ell' \in \calL' \backslash \calL$ and 
    $(\tilde a_\ell^{(i)}, \tilde b_{\ell}^{(k)}) \mapsto W'$ for any $\ell \in \calL \cap \calL$ all satisfy the bounded difference property with constant $O(\frac{1}{nq\eta})$ on the event $\calH_{ik} \cap \Gamma_A(i) \cap \Gamma_B(k)$.
    
\end{itemize}



Recall that in the definition of outdegree $\tilde a_\ell^{(i)}$, we have excluded the $2$-hop neighborhood of $i$ in $A$;
similarly, in the definition of outdegree $\tilde b_{\ell'}^{(k)}$, we have excluded the $2$-hop neighborhood of $k$ in $B$. Therefore, we have that
\begin{itemize}
\item 
$\{ (\tilde a_\ell^{(i)}, \tilde b_{\ell}^{(k)} ) \} $ are independent across different $\ell  \in \calL \cap \calL'$; 
    \item $\{ \tilde a_\ell^{(i)}\} $ are independent across different $\ell \in \calL\setminus \calL'$;
    \item $\{\tilde b_{\ell'}^{(k)} \}$ are independent across different
$\ell' \in \calL' \setminus \calL$;
\item $\{ ( \tilde a_\ell^{(i)}, \tilde b_{\ell'}^{(k)} ): \ell \in \calL \cap \calL' \} $ are independent of $\{ \tilde a_\ell^{(i)}: \ell \in \calL \setminus \calL', \; \tilde b_{\ell'}^{(k)} : \ell \in \calL' \setminus \calL  \} $.
\end{itemize}
However, $\tilde a_\ell^{(i)}$ for $\ell \in \calL \setminus \calL'$ and $\tilde b_{\ell'}^{(k)}$ for 
$\ell' \in \calL' \setminus \calL$ may be dependent, 
because $A_{\ell \ell'}$ may
contribute to the outdegree $\tilde a_\ell^{(i)}$, and $B_{\ell \ell'}$ may contribute 
to the outdegree $\tilde b_{\ell'}^{(k)}$. Fortunately,
similar to the reasoning in \prettyref{fig:lemma2}, 
conditioned on the edge sets $E_A(\calL, \calL')$ and 
$E_B(\calL, \calL')$, the outdegrees $\{ \tilde a_\ell^{(i)}: \ell \in \calL \setminus \calL' \}$
and $\{\tilde b_{\ell'}^{(k)}:\ell' \in \calL' \setminus \calL\}$ are independent, since the definition of the outdegree in \prettyref{eq:outdegree2hop1}--\prettyref{eq:outdegree2hop2}
excludes the two-hop neighborhood.

In particular, 
write $E(\calL, \calL')=(E_A(\calL,\calL'), E_B(\calL, \calL'))$ for simplicity,
and let $\calF_{ik}$ denote an event (to be specified later) that is measurable with respect to $E(\calL, \calL')$
and holds with high probability: $\prob{\calF_{ik}} \ge 1- \exp\left(\Omega(nq)\right)$. 
Conditioned on $E(\calL, \calL')$ such that the event $\calF_{ik}$ holds, 
applying McDiarmid's inequality
and noting that $|\calL|, |\calL'| \le (2nq)^2$ on the event $\Gamma_A(i) \cap \Gamma_B(k)$, we get that 
\begin{equation}
\prob{ W' - \expect{ W' \mid E(\calL, \calL')   } \ge  \frac{nq}{16}~\Big|~  E(\calL, \calL') } 
\le \exp \left( -c_1 (nq\eta)^2  \right),
\label{eq:mcdiarmid2}
\end{equation}
where  $c_1$ is an absolute constant.  

We next compute $\expect{ W' \mid E(\calL, \calL')}$. We first claim that for all $j \in J$ and $j' \in J'$, 
\begin{align}
 \prob{  Z^{(ik)}_{jj' } \le 2\eta ~\Big|~ E(\calL, \calL') } \le e^{-\Omega(L) } .  \label{eq:Z_ik_3_hop}
\end{align}
By definition of $W'$, we have
\begin{align*}
\expect{ W' ~\Big|~ E(\calL, \calL')} & = \sum_{ j \in J, j' \in J' } 
\expect{ F \left( \tilde Z^{(ik)}_{jj' } \right) ~\Big|~ E(\calL, \calL')} M_{jj'} \\
& \le \sum_{ j \in J, j' \in J' } \prob{  Z^{(ik)}_{jj' } \le 2\eta ~\Big|~ E(\calL, \calL') }\\
& \le O\left(e^{-\Omega(L) } \right) \sum_{ j \in J, j' \in J' } M_{jj'} \\
& \le O(e^{-\Omega(L) } nq) \le \frac{nq}{16},
\end{align*}
where the first inequality follows by the definition of $F$; 
the second inequality holds due to \prettyref{eq:Z_ik_3_hop}; 
the third inequality is due to $|J| \le 2nq$ on the event $\Gamma_A(i)$ and that $M$ is a matching; the last inequality holds due to $L \ge L_0 \log n$.
Combining the last displayed equation with \prettyref{eq:mcdiarmid2}, we obtain
$$
\prob{ W' \ge nq/8  \mid E(\calL, \calL')    }  \indc{\calF_{ik} } \le \exp \left( -c_1 (nq\eta)^2  \right).
$$
Averaging over the last displayed equation yields that 
$$
\prob{ \left\{ W' \ge nq/8 \right\}  \cap \calF_{ik}  } \le  \exp \left( -c_1 (nq\eta)^2  \right).
$$
Combining the last displayed equation with \prettyref{eq:Y_M_bound}, we obtain
\begin{align*}
  \prob{
  \left\{ \langle Y^{(ik)}, M \rangle
  \ge nq/4  \right\} \cap \calF_{ik}  } \le  \prob{ \left\{ W' \ge nq/8 \right\}  \cap \calF_{ik}  }  \le \exp \left(
    -c_1 ( nq \eta)^2 \right).
\end{align*}
Finally, applying a union bound over the set of all  possible matching $M$ and recalling the definition of similarity  $W_{ik}$ in \prettyref{eq:def_W_3_hop}, we  get that  
$$
\prob{ \left\{ W_{ik} \ge nq/4 \right\} \cap \calF_{ik} } 
\le (2nq)! \times e^{ -c_1( nq \eta)^2 }
\le e^{-\Omega\left(nq \log(nq) \right)},
$$
where the last inequality holds due to the choice of $\eta$ 
in \prettyref{eq:eta-threshold}
and the assumption that $L \ge L_0 \log (nq)$. 
Therefore, by a union bound, 
$$
\prob{  W_{ik} \ge nq/4} \le \prob{ \left\{ W_{ik} \ge nq/4 \right\} \cap \calF_{ik} } 
+ \prob{\calF^c_{ik} } \le e^{-\Omega(nq)}. 
$$

\nb{
It remains to specify the event $\calF_{ik}$ and verify the claim \prettyref{eq:Z_ik_3_hop} when 
conditioned on $E(\calL, \calL')$ such that the event $\calF_{ik}$ holds.
The proof follows a similar argument as in the proof of \prettyref{lmm:fake}.
Specifically, recall that $\tilde{Z}^{(ik)}_{j j' }= d\left( \tilde{\mu}_j^{(i)} , \tilde{\nu}_{j'} ^{(k)} \right)$,
where 
\begin{equation*}
\tilde{\mu}^{(i)}_j \triangleq  \frac{1}{ a^{(i)}_j } \sum_{ \ell \in N_A(j) \setminus N_A[i] } 
\delta_{ \tilde{a}_\ell^{(i)}} \; - \Binomc\left(n-\tilde{a}_i,q \right),
\end{equation*}
and
 \begin{equation*}
\tilde{\nu}^{(k)}_{j'} \triangleq  \frac{1}{b^{(k)}_{j'}  } \sum_{ \ell \in N_B(j') \setminus N_B[k ] } 
\delta_{ \tilde{b}_\ell^{(k)} } \; - \Binomc\left(n-\tilde{b}_k, q \right).
\end{equation*}

Let $\nu= \Binomc\left(n-\tilde{a}_i,q \right)$ and $\nu'=\Binomc\left(n-\tilde{b}_k, q \right)$.
Observe that for $\ell \in \tilde{N}_B(k)$, $\tilde{a}_\ell^{(i)}$ is no longer distributed as $\nu$ after conditioning on the 
$2$-hop neighborhood $\tilde{N}_B(k)$, and likewise for $ \tilde{b}_\ell^{(k)} $ for $\ell \in \tilde{N}_A(i)$. 
Therefore, we decompose $\tilde{\mu}^{(i)}_j $ and $\tilde{\nu}^{(k)}_{j'} $ as 
\begin{align*}
\tilde{\mu}^{(i)}_j = & ~ 
\kappa \hat{P}  + (1-\kappa) \tilde{P}   \\
\tilde{\nu}^{(k)}_{j'}  = & ~ \kappa' \hat{Q} + (1-\kappa') \tilde{Q} ,
\end{align*}
where 
\begin{align*}
\kappa \triangleq  \frac{ | \hat{S}  |  }{a^{(i)}_j } ,  
\quad \kappa '\triangleq \frac{ | \hat{T} | }{b_{j'}^{(k)}}, 
\end{align*}
and 
\begin{align*}
\tilde{S} &=\left( N_A(j) \setminus N_A[i] \right)  \cap  \tilde{N}_B(k)^c, \quad \hat{S} = \left( N_A(j) \setminus N_A[i] \right)  \cap  \tilde{N}_B(k)  \\
\tilde{T} & = \left( N_B(j') \setminus N_B[k] \right)  \cap \tilde{N}_A(i)^c, \quad \hat{T} = \left( N_B(j') \setminus N_B[k] \right) \cap \tilde{N}_A(i), 
\end{align*}
and 
\begin{align*}
\tilde{P} \triangleq  & ~ \frac{1}{ | \tilde{S} |  } \sum_{ \ell \in  \tilde{S}  }  
\delta_{ \tilde{a}_\ell^{(i)}} - \nu, 
\quad \hat{P} ~\triangleq  \frac{1}{ |  \hat{S} |  } 
\sum_{ \ell  \in  \hat{S}   }   \delta_{ \tilde{a}_\ell^{(i)}}- \nu, \\
\tilde{Q}  \triangleq  & ~ \frac{1}{ | \tilde{T}   | } \sum_{  \ell \in    \tilde{T}   } 
\delta_{ \tilde{b}_\ell^{(k)} }  - \nu', 
\quad \hat{Q}~\triangleq  \frac{1}{ | \hat{T} |} 
\sum_{ \ell \in \hat{T}  }    \delta_{ \tilde{b}_\ell^{(k)} }- \nu'.
\end{align*}

Therefore, we have 
\begin{align}
\tilde{Z}^{(ik)}_{j j' } & \geq   (1-\kappa) d( \tilde{P}, \tilde{Q} )  - 
\kappa \| [ \hat{P}  ]_L\|_1 - \kappa'  \|[ \hat{Q} ]_L\|_1 - |\kappa-\kappa'| \times \|[\tilde{Q} ]_L\|_1.
\end{align}
On the event $\Gamma_A(i) \cap \Gamma_B(k) \cap \Gamma_{ik}$, 
we have $a^{(i)}_j, b_{j'}^{(k)} \ge nq/2$, and $| \hat{S}  | , | \hat{T} | \le 2$.
Therefore, $\kappa, \kappa'  \le \frac{4}{nq}$.
Since $\| [ \hat{P}  ]_L\|_1,   \|[ \hat{Q} ]_L\|_,  \|[\tilde{Q} ]_L\|_1 \le 2$, it follows that 
\begin{align}
\tilde{Z}^{(ik)}_{j j' }  \ge \frac{1}{2} d( \tilde{P}, \tilde{Q} )  - O\left(  \frac{1}{nq} \right).   \label{eq:Zik1_3_hop}
\end{align}

It remains to lower bound $d( \tilde{P}, \tilde{Q} ) $. 
Conditioning on $E(\calL, \calL')$, we aim to apply \prettyref{lmm:indsample} with 
$m=| \tilde{S}  |$, 
$m'=| \tilde{T} |$, $m_0=nq$,
$
\{X_\ell \}_{\ell=1}^m = \{ \tilde{a}_\ell^{(i)}\}_{\ell \in \tilde{S}   },$
and
$
\{Y_\ell \}_{\ell=1}^{m'} = \{ \tilde{b}_\ell^{(k)}\}_{\ell \in \tilde{T}  } 
$.
Note that since $\kappa, \kappa' \le 1/2$ and $nq/2 \le a_j^{(i)}, b_{j'}^{(k)} \le 2nq$, it follows that 
$m, m' =\Theta(m_0)$. Also, as previously argued, after conditioning on $E(\calL, \calL')$, 
$\{ \tilde{a}_\ell^{(i)}\}_{\ell \in \tilde{S}   }$ and 
$ \{ \tilde{b}_\ell^{(k)}\}_{\ell \in \tilde{T}  } $ are two independent sequence of real-valued
random variables.  
It remains to check the assumption \prettyref{eq:partition} in \prettyref{lmm:indsample}, that is, 
there exists a set $\calL_0 \subset \tilde{S}$ with 
$|\calL_0| \ge m/4$ and constants $c_1, c_2 \in (0,1]$
such that 
$
\frac{c_1}{L} \leq \prob{ \tilde{a}_\ell^{(i)} \in I }  \leq \frac{c_2}{L}
$
for all interval $I \subset [-2,2]$ of length $1/L$. 

To this end, recall that 
\begin{align*}
\tilde{a}_\ell^{(i)} 
= \frac{1}{ \sqrt{  (n - \tilde{a}_i ) q (1-q)  } } \sum_{u \in \tilde{N}_A(i)^c } \left( A_{u \ell} - q  \right).
\end{align*}
For $\ell \in  \tilde{S}$:
\begin{itemize}
\item If $u \in N_B[k]$, then $B_{u\ell}=0$; otherwise, $\ell \in \tilde{N}_B (k)$, violating $\ell \in \tilde{S}$. 
Thus, $A_{u\ell} \sim \Bern(q')$ with $q'=\prob{A_{u\ell}=1|B_{u\ell}=0}= \frac{q(1-s)}{1-ps} \le q$; 
\item If $u \in \calL'$, then $A_{u \ell}$ is deterministic when conditioning on $E_A(\calL, \calL')$;
\item If $u \notin \tilde{N}_B(k)$, then $A_{u\ell} \sim \Bern(q)$. 
\end{itemize}

Recall that $e_A(\ell, S)$ denotes the number of edges between vertex $\ell$ and vertices in $S$ in graph $A$. 
Define $\phi=| \calL' \setminus \tilde{N}_A(i) | $, 
$\psi=| N_B[k] \setminus \tilde{N}_A(i) |$, and 
$$
\calL_0 = \left\{ \ell \in \tilde{S} : \left| e_A\left( \ell, \calL' \setminus \tilde{N}_A(i) \right) -   
\phi q \right|
\le \sqrt{nq(1-q)/2}
 \right\}.
$$
Define the event 
\begin{align}
\calF_{ik}= \left\{ |\calL_0| \ge m/4  \right\},\label{eq:def_F0}
\end{align}
which is measurable with respect to $E_A(\calL,\calL')$ since $\tilde{S}\subset \calL$.
Note that for each $\ell\in\calL$,
 $e_A\left( \ell, \calL' \setminus \tilde{N}_A(i) \right)
\sim \Binom\left( \phi , q\right)$. Hence, by Chebyshev's inequality, 
$$
\prob{\ell \in \calL_0} \ge 1- \frac{2 \phi }{n} \ge \frac{1}{2},
 $$
where the last inequality holds because $\phi \le |\calL'| \leq \tilde{b}_k\le (2nq)^2$ and $q\le n^{-\epsilon}$ for $\epsilon>9/10$. 
Moreover, $e_A\left( \ell, \calL' \setminus \tilde{N}_A(i) \right) $ are independent across 
$\ell \in \tilde{S}$. Hence, $|\calL_0|$ is stochastically lower bounded by $\Binom(m,1/2)$.
It follows from the binomial tail bound~\prettyref{eq:bintail} and the fact that $m=\Omega(nq)$ that 
$$
\prob{\calF_{ik} }= \prob{ |\calL_0| \ge m/4} \ge 1- \exp(\Omega(nq)).
$$
Let 
$$
u_\ell = \frac{1}{\sqrt{ (n-\tilde{a}_i - \phi - \psi) q (1-q) + \psi q'(1-q') }} 
\left[ e_A\left(\ell,  \tilde{N}_A(i)^c \setminus L' \right) -  (n-\tilde{a} - \phi - \psi) q - \psi q' \right]
$$
and
$$
v_\ell=\frac{1}{ \sqrt{ (n-\tilde{a}_i) q (1-q) } } \left[ \psi (q'-q) + e_A\left( \ell, \calL' \setminus \tilde{N}_A(i) \right) - \phi q.
\right]
$$
Let 
$$
\alpha_\ell = \sqrt{ \frac{(n-\tilde{a}_i - \phi - \psi) q (1-q) + \psi q'(1-q')}{(n-\tilde{a}_i) q (1-q)  } }.
$$
Then 
$\tilde{a}_\ell^{(i)}  = \alpha_\ell u_\ell + v_\ell$. 
Note that on event $\Gamma_A(i) \cap \Gamma_B(k)$, $\tilde{a}_i \le (2nq)^2$ and 
$\phi, \psi \le 2nq$. 
Since $q' \le q \le n^{-\epsilon}$, it follows that $1/\sqrt{2} \le \alpha_\ell \le 1$.
Moreover,  $|v_\ell | \le 1$ for all $\ell \in \calL_0$. 
By the Berry-Esseen theorem, we have 
$$
\prob{ \tilde{a}_\ell^{(i)} \in I } = \prob{ u_\ell \in \frac{ I - v_\ell }{\alpha_\ell}}
= \prob{ \calN(0,1)  \in \frac{ I - v_\ell }{\alpha_\ell} } \pm \frac{O \left( 1\right)}{\sqrt{nq}}
= \frac{\Theta(1)}{L}   \pm \frac{O \left( 1\right)}{\sqrt{nq}} = \frac{\Theta(1)}{L},
$$
where the last equality holds due to $nq \ge CL^2$.

Conditioning on $E(\calL, \calL')$ such that event $\calF_{ik}$ holds and applying \prettyref{lmm:indsample}, 
we get that 
\begin{align}
\prob{ d( \tilde{P} , \tilde{Q} ) \leq  \alpha_1 \sqrt{\frac{L}{ n q } }   ~ \Bigg| ~ E(\calL, \calL')  }  \indc{\calF_{ik}}
 \le e^{-\Omega(L)}, \label{eq:fake2_3_hop}
\end{align}
where $\alpha_1$ is some absolute constant.

Combining \prettyref{eq:Zik1_3_hop} with 
\prettyref{eq:fake2_3_hop},
 we have that conditioned on  $E(\calL, \calL')$ such that event $\calF_{ik}$ holds, with probability at least 
$1-e^{-\Omega(L)}$, 
\[
\tilde{Z}^{(ik)}_{j j' } \geq  \frac{\alpha_1}{2}  \sqrt{\frac{L}{nq}} 
- O\left(  \frac{1}{nq} \right)
> 2 \eta_0 \sqrt{\frac{L}{nq} }  = 2\eta
\]
for some sufficiently small constant $\eta_0$, where the last inequality holds 
due to $nq \ge C \log n$. 
Thus we arrive at the desired claim \prettyref{eq:Z_ik_3_hop}.
}
\qed
\end{proof}

With \prettyref{lmm:true_3_hop}
and \prettyref{lmm:fake_3_hop}, we are 
ready to prove \prettyref{thm:sparse}.

\begin{proof}[Proof of \prettyref{thm:sparse}]
Let $\calH$ denote the event that 
all $2$-hop neighborhoods in the union graph
$A \cup B$ are tangle-free. 
Under the assumption that 
$q \le n^{-\epsilon}$ for $\epsilon>9/10$ and the fact that
the union graph $A \cup B \sim \calG(n, ps(2-s))$,
it follows from \prettyref{lmm:tangle_free} that
$\prob{ \calH} \ge 1-O\left(n^{9-10\epsilon}\right)$. 
Define the event $\calF= \calH \cap \left(\cap_i \left( \Gamma_A(i) \cap \Gamma_B(i) \right) \right) \cap \left( \cap_{i,k} \Gamma_{ik} \right)$. 
It follows that 
$$
\prob{\calF^c} \le \prob{\calH^c} + \sum_{i \in [n]} \left( \prob{\Gamma_A^c(i) } +\prob{\Gamma_B^c(i) }   \right) 
+ \sum_{i,k \in [n]} \prob{\Gamma^c_{ik}} 
\le O\left(n^{9-10\epsilon}\right). 
$$

Applying  \prettyref{lmm:true_3_hop} with $L=C \log (nq)$ and averaging over the 2-hop neighborhoods $\tilde{N}_A(i)$ and $\tilde{N}_B(i)$
and noting that $nq \ge C_0 \log n$ for a large constant $C_0$, $q \le n^{-\epsilon}$ for $\epsilon>9/10$,
and $\sigma \le \sigma_0/L$ for a sufficiently small
constant $\sigma_0$, we get that 
$$
\prob{ \left\{ W_{ii} \le \frac{nq}{4}  \right\} \cap \calH_{ii} \cap \Gamma_A(i) \cap \Gamma_B(i) \cap \Gamma_{ii} } \le e^{-\Omega(nq)} \le n^{-2} .
$$
Similarly, for $i \neq k$,  applying \prettyref{lmm:fake_3_hop} with $L=C \log (nq)$ and averaging over the 2-hop neighborhoods $\tilde{N}_A(i)$ and $\tilde{N}_B(k)$,
we get that 
\begin{align*}
\prob{  \left\{  W_{ik}  \ge \frac{nq}{4} \right\} \cap \calH_{ik} \cap \Gamma_A(i) \cap\Gamma_B(k)  \cap \Gamma_{ik}  } 
 \le  e^{-\Omega(nq ) } \le n^{-3}.
\end{align*}

By the union bound and the fact that $\calH \subset H_{ii}$, we have
\begin{align*}
\prob{  \left\{ \min_{i \in [n]}  W_{ii} \le \frac{nq}{4}  \right\} \cap \calF  }
&\le \sum_{i \in [n]} \prob{ \left\{ W_{ii} \le \frac{nq}{4} \right\} \cap \calF  } \\
& \le 
\sum_{i \in [n]} \prob{ \left\{ W_{ii} \le \frac{nq}{4}  \right\} \cap \calH_{ii} \cap \Gamma_A(i) \cap \Gamma_B(i) \cap \Gamma_{ii}  }  \le n^{-1}.
\end{align*}

Similarly, by the union bound and the fact that $\calH \subset H_{ik}$, we have
\begin{align*}
\prob{  \left\{ \max_{i\neq k } W_{ik}   \ge \frac{nq}{4}  \right\} \cap \calF  }
&\le \sum_{i \neq k} \prob{ \left\{ W_{ik} \ge \frac{nq}{4} \right\} \cap \calF  } \\
& \le \sum_{i \neq k } \prob{ \left\{ W_{ik} \ge \frac{nq}{4} \right\} \cap \calH_{ik}  \cap \Gamma_A(i) \cap\Gamma_B(k) \cap \Gamma_{ik} }   \le  n^{-1} .
\end{align*}

In conclusion, by the union bound, 
\begin{align*}
&\prob{\min_{i \in [n]}  W_{ii}  \le \max_{i\neq k } W_{ik}   } \\
& \le \prob{\calF^c} + \prob{  \left\{ \min_{i \in [n]}  W_{ii} \le \frac{nq}{4}  \right\} \cap \calF  }+ \prob{  \left\{ \max_{i\neq k } W_{ik}   \ge \frac{nq}{4}  \right\} \cap \calF  } \\
& \le O\left(n^{9-10\epsilon}\right).
\end{align*}
Thus with probability at least
$1-O\left(n^{9-10\epsilon}\right)$,  \prettyref{alg:dist_3_hop} outputs
$\hat{\pi}=\pi^*$. \qed
\end{proof}

\def \figpath {JXMatlab/pdf_files/}

\section{Numerical experiments}
\label{sec:exp}
In this section, we empirically evaluate the performance of 
degree profile matching (DP), a quadratic programming relaxation of QAP based on 
doubly stochasticity (QP), and a spectral relaxation (SP). 

The performance metric is defined as follows:
for a given estimator $\hat{\pi}$ of the ground-truth permutation $\pi^*$, we define its accuracy rate as the fraction of correctly matched pairs: 
\begin{align}
\mathrm{acc}(\hat{\pi}) \triangleq \frac{1}{n}  \sum_{i \in [n]} 
\indc{\pi^*(i)=\hat{\pi}(i)} \, .
 \end{align}

Recall that we use outdegrees instead of degrees in our degree profile matching \prettyref{alg:dist}
to reduce the dependency and facilitate the theoretical analysis.  In all numerical experiments,
we simply use degree profiles defined through the usual vertex degrees. 
Moreover, instead of using the $Z$ distance \prettyref{eq:distance} defined as 
the total variation distance between discretized degree profiles, we directly use the $1$-Wasserstein $W_1$-distance  
between degree profiles; see \prettyref{eq:dp_distance} with $p=1$. 
Note that for two empirical distributions with the same sample size, such as $\mu$ and $\nu$ in \prettyref{eq:mu_nu_def}, one can compute their 
$W_1$-distance by sorting the samples:
$$
W_1 (\mu, \nu) = \sum_{i=1}^n \left| X_{(i)} - Y_{(i)} \right|,
$$
where $X_{(1)} \ge \cdots \geq X_{(n)}$ and $Y_{(1)} \ge \cdots \geq Y_{(n)}$.
If the sample sizes are different, as is the case for \ER graphs, it is more convenient to compute the $W_1$-distance 
using either the CDF characterization \prettyref{eq:mu_nu_def} or the original coupling definition.

For the QP method, note that the optimum solution of the quadratic programming relaxation
of QAP may not be a permutation matrix. 
Thus we round the optimal solution to $\calS_n$ by projection:
$ \min_{ \Pi \in \calS_n} \| \Pi - \hat{D} \|_F^2$,
which is a linear assignment problem and efficiently solvable via max-weighted bipartite matching.

For the SP method, we compute the eigenvectors $u$ of $A$ and 
$v$ of $B$ corresponding to the largest eigenvalue. Then we align $u$ and $v$,
by finding the permutation
$\pi$ that minimizes the Euclidean distance $\sum_{i \in [n]} | u_i - v_{\pi(i)}|^2$.
This is equivalent to $\min_{ \Pi \in \calS_n} \| \Pi - u v^\top \|_F^2$,
which again can be efficiently solved via max-weighted bipartite matching.

For each method, we can potentially boost its accuracy using the iterative clean-up procedure described in \prettyref{alg:clean-up}.
\begin{algorithm}
\caption{Iterative clean-up procedure} \label{alg:clean-up}
\begin{algorithmic}[1]
\STATE {\bfseries Input:} Graphs $A$ and $B$ on $n$ vertices; a permutation $\pi$ on $[n]$; and
the maximum number of iterations $T$;
\STATE {\bfseries Output:} A permutation $\hat\pi$ on $[n]$.
\STATE (Initialization) Initialize $\Pi_0$ to be the permutation matrix corresponding to $\pi$
 \FOR{$t=1,\ldots, T$}
 \STATE Solve the linear assignment problem
 \begin{align}
\Pi_{t+1} \in \arg \max_{\Pi \in \calS_n}  \iprod{\Pi}{A\Pi_t B}
\label{eq:iter_clean_up}
 \end{align}
 \ENDFOR
\STATE Output $\hat\pi$ to be the permutation corresponding to $\Pi_{T+1}$.
\end{algorithmic}
\end{algorithm}

Note that $(A\Pi_t B)_{ik}$ in \prettyref{eq:iter_clean_up} can be viewed as the number of ``common'' neighbors $j$
between $i$ and $k$ under the permutation $\pi_t$ in the sense that
$j$ is $i$'s neighbor in $A$ and $\pi_t(j)$ is $k$'s neighbor in $B$. 
Hence, \prettyref{eq:iter_clean_up} finds 
the matching which maximizes the total sum of ``common'' neighbors under $\pi_t$.
This resembles  the second stage of \prettyref{alg:seed} for seeded graph matching. 
Alternatively, by
rewriting  the objective in \prettyref{eq:iter_clean_up}  as $\vecc(\Pi)^\top  (B \otimes A) \vecc(\Pi_t)$, 
where $B \otimes A$ denotes the Kronecker product and $\vecc(\Pi) \in \reals^{n^2}$ denotes the vectorized version of the matrix $\Pi$, we can reduce \prettyref{eq:iter_clean_up} to 
the projected power iteration discussed in~\cite{onaran2017projected}. 

For ease of notation, we denote by DP+ the degree profile matching algorithm followed by the iterative clean-up procedure. 
Similarly, we define QP+ and SP+. We run the iterative clean-up procedure up to $T=100$
iterations. Also, for the sake of computational efficiency,
 instead of using the max-weighted bipartite matching algorithm to solve \prettyref{eq:iter_clean_up} exactly,
 we use the following standard greedy matching algorithm to approximately solve \prettyref{eq:iter_clean_up} with 
input weight matrix being $A\Pi_t B$.

\begin{algorithm}
\caption{Greedy  Matching} \label{alg:greedy_matching}
\begin{algorithmic}[1]
\STATE {\bfseries Input:} A bipartite graph with $n\times n$ symmetric 
edge weight matrix $W$;
\STATE {\bfseries Output:} A $n\times n$ permutation matrix $\Pi$.
\STATE (Initialization) Initialize $\calM=\emptyset$
 \FOR{all $(i,j)$ in decreasing order of $W_{ij}$}
 \STATE Add $(i,j)$ to $\calM$ if $\calM$ forms a matching 
 \ENDFOR
\STATE 
Output $\Pi$, where $\Pi_{ij}=1$ if $(i,j) \in \calM$ and $\Pi_{ij}=0$ otherwise.
\end{algorithmic}
\end{algorithm}




\subsection{Wigner Matrices}
\label{sec:exp-gaussian}
We  evaluate the performance of all three algorithms as well as their cleaned-up version on the correlated Wigner model given in \prettyref{sec:gaussian}. 
The results are shown in \prettyref{fig:Wigner_comparison} as a function of the noise magnitude $\sigma$ with $n=1000$ fixed.
Clearly, QP dominates DP, which, in turn, significantly outperforms SP in term of the matching accuracy.
Furthermore, the iterative clean-up procedure significantly boosts the accuracy rates for all three methods.
Computationally, QP needs to solve a quadratic program, where the Hessian matrix in the objective function involves Kronecker product 
$B \otimes A$ and thus is of dimension $n^2 \times n^2. $
Hence,  QP is much more computationally expensive and memory costly than either DP and SP. 
In our simulation of QP, we developed a fast solver for QP based on the alternating direction method of multipliers (ADMM)
algorithm~\cite[Section 5.2]{boyd2011distributed} and that avoids computing $B \otimes A$; nevertheless, even with this fast solver,
to generate the simulation results in \prettyref{fig:Wigner_comparison}, QP takes around 85 minutes, 
while DP takes about 7 minutes, and SP takes about 23 seconds.

\begin{figure}[ht]
\centering
\includegraphics[width=0.5\columnwidth]{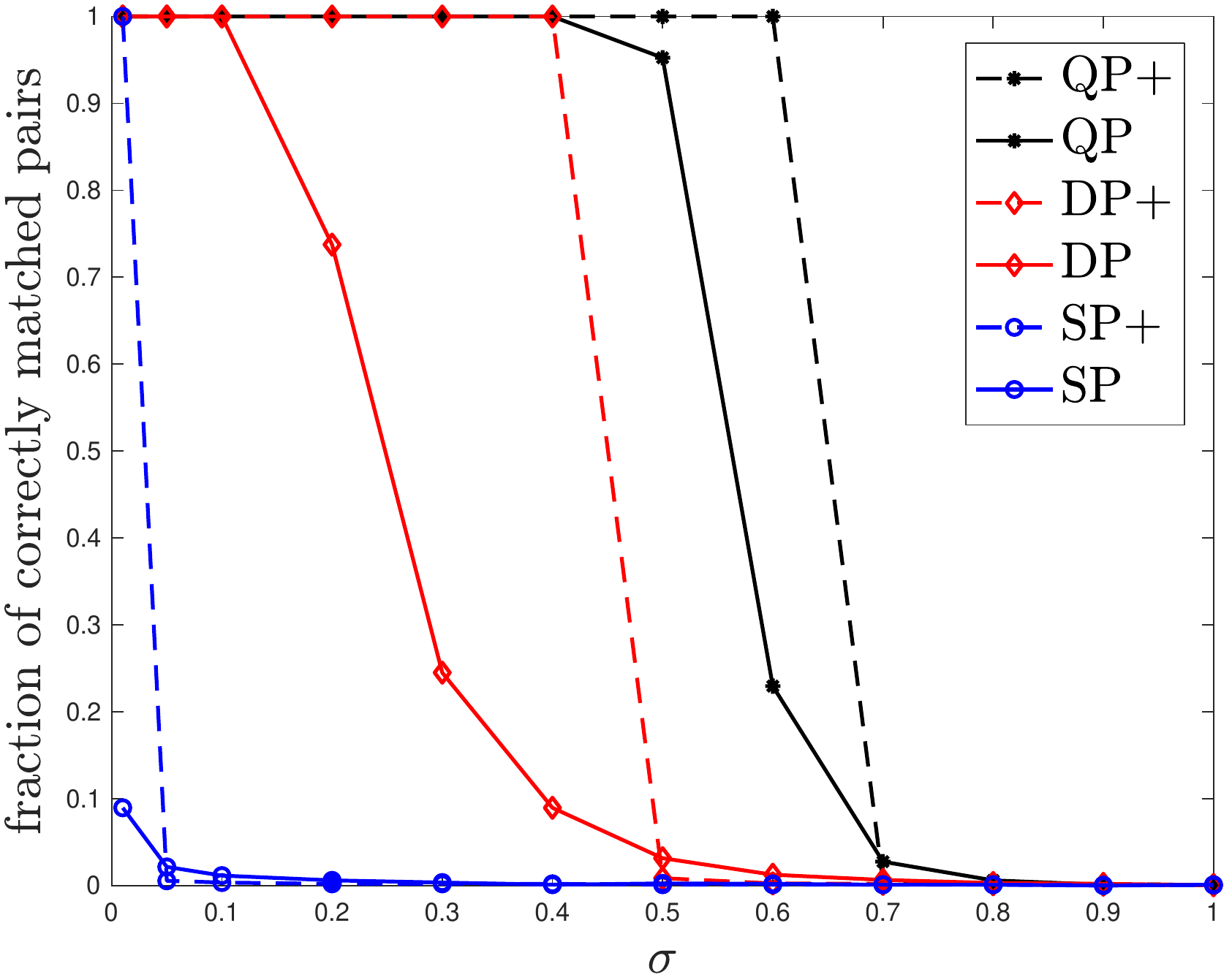}
\caption{Simulated correlated Wigner model with $n=1000$ and varying $\sigma$. 
For each value of $\sigma$, the accurate rate shown is the median of $10$ independent runs.}
\label{fig:Wigner_comparison}
\end{figure}

Next we simulate the performance of DP and DP+ for different matrix sizes ranging 
from $100$ up to $1600$. The results are depicted in~\prettyref{fig:Wigner_phase}. 
Since our theory predicts that DP succeeds in exact recovery
when $\sigma \log n \le c$ for a small constant $c$, we rescale the $x$-axis as
$\sigma \log n$.  As we can see, the curves for different $n$ align well with each other.
Moreover, the accuracy rate of DP gradually drops off from $1$ to $0$ when $\sigma \log n$ is above $0.7$,
while that of DP+ sharply drops off from $1$ to $0$ when $\sigma \log n $ is above $3.3$.

\begin{figure}[ht]
	 \centering
	 \subfigure[The degree-profile (DP)  algorithm.]%
	 {\includegraphics[width=0.45\textwidth]{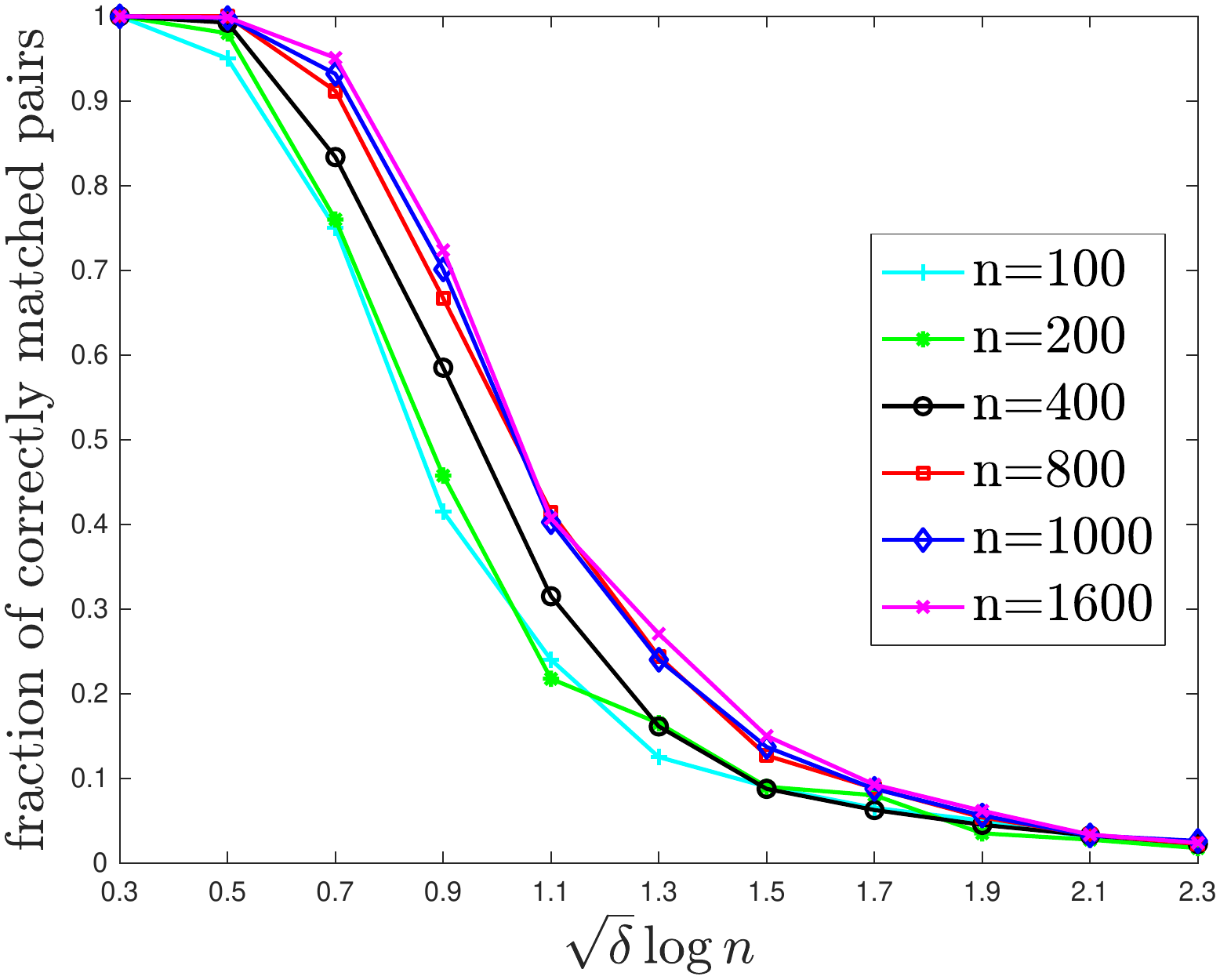} }
	 \subfigure[The degree profile followed by the iterative clean-up procedure (DP+).]%
	 {\includegraphics[width=0.45\textwidth]{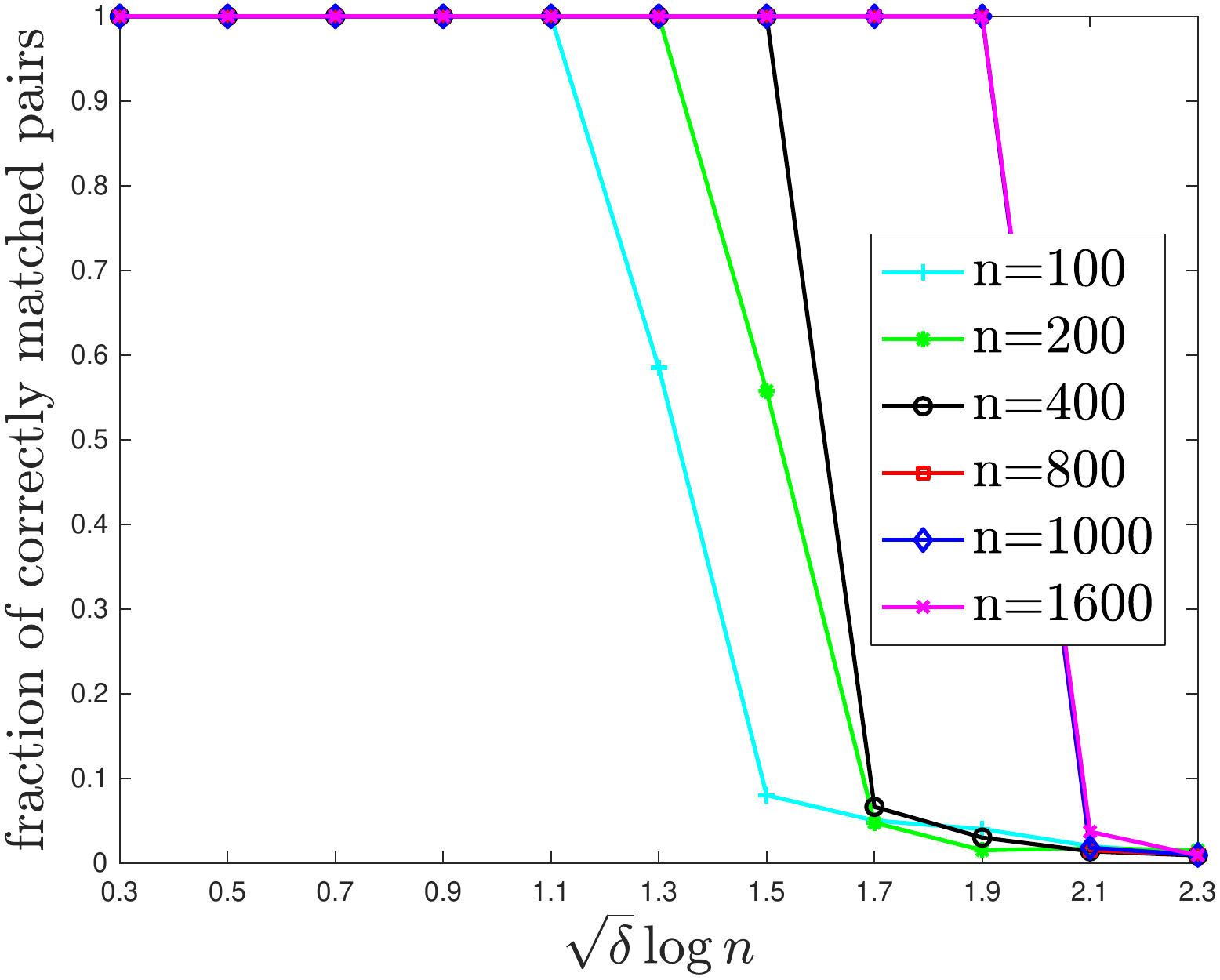} }
	 \caption{Simulated correlated Wigner model with varying $n$ and $\sigma$. 
For each value of $\sigma$, the accurate rate shown is the median of $10$ independent runs.
    }
	\label{fig:Wigner_phase}
\end{figure}


\subsection{\ER Graphs}
\label{sec:exp-er}
We evaluate the performance of all three algorithms as well as their cleaned-up version on the correlated \ER graph model $\calG(n,q;s)$. We focus on sparse graphs where the edge probability of the parent graph is fixed to be $p\triangleq q/s=\log^2(n)/n$. The simulation results for dense graphs (such as $p=1/2$) are similar and thus omitted. 

The results are shown in \prettyref{fig:ER_comparison} as a function of the edge deletion 
probability $\delta \triangleq 1-s$ with $n=1000$ fixed. Analogous to the Wigner case, QP dominates DP, which, in turn, significantly outperforms SP in term of the matching accuracy, and the iterative clean-up procedure significantly boosts the accuracy rates for all three methods. Computationally, 
to generate the simulation results in \prettyref{fig:ER_comparison}, 
QP takes around 51 minutes, DP takes about 2 minutes, and SP takes about 12 seconds.
\nb{Note that each of these methods is run on the same architecture under the same conditions}.

\begin{figure}[ht]
\centering
\includegraphics[width=0.5\columnwidth]{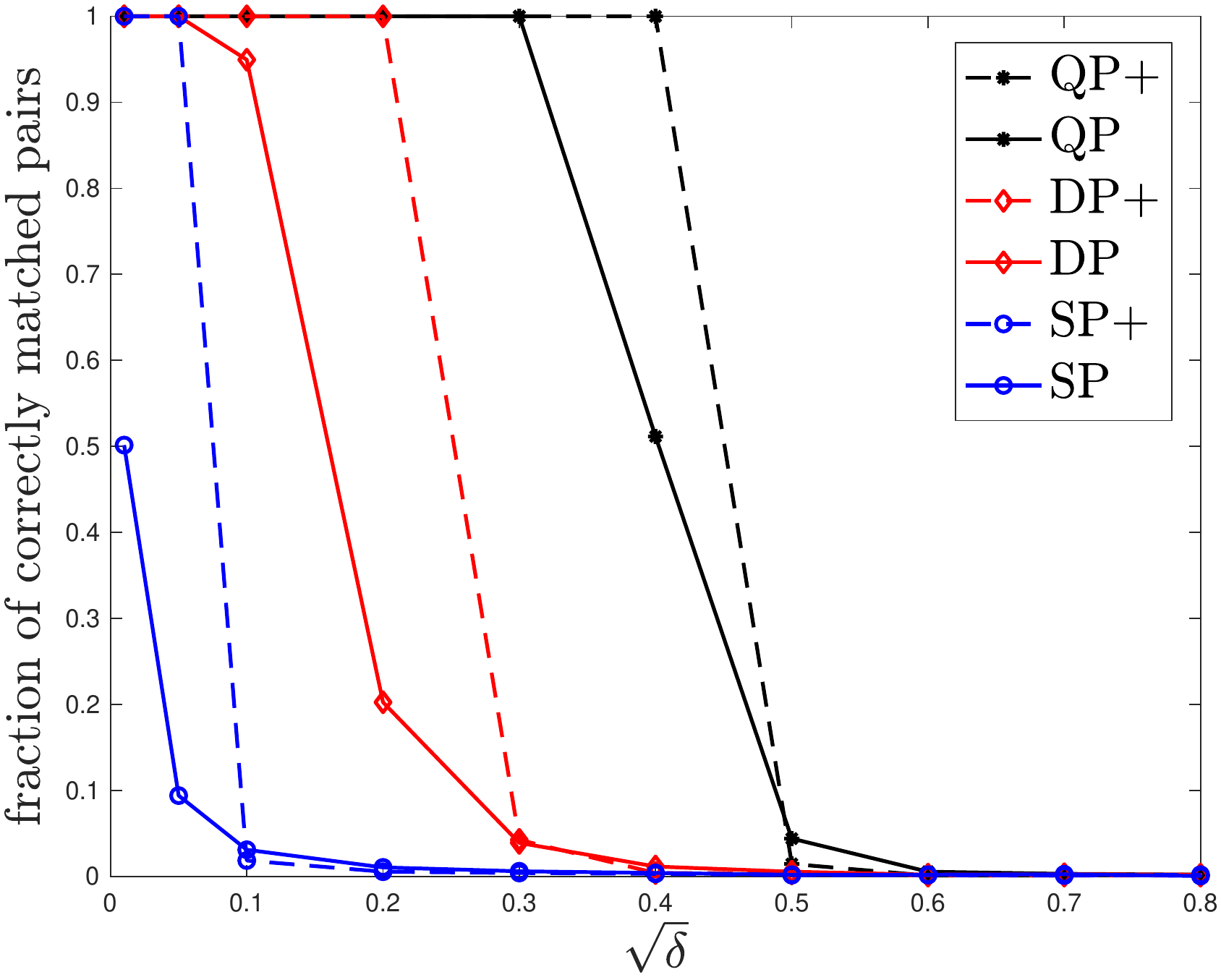}
\caption{Simulated correlated \ER graph model $\calG(n,q;s)$ 
with $n=1000$, $p\triangleq q/s=\log^2(n)/n$, and varying $\sqrt{\delta} =\sqrt{1-s}$. 
For each value of $\sqrt{\delta}$, the accurate rate shown is the median of $10$ independent runs.}
\label{fig:ER_comparison}
\end{figure}

Next we simulate the performance of DP and DP+ for different graph sizes ranging 
from $100$ up to $1600$. The results are depicted in~\prettyref{fig:ER_phase}. 
Since our theory predicts that DP succeeds in exact recovery
when $\sqrt{\delta} \log n \le c$ for a small constant $c$, we  rescale the $x$-axis as
$\sqrt{\delta}\log n$.  
As we can see, the curves for different $n$ align well with each other.
Analogous to the Wigner case, 
the accuracy rate of DP gradually drops off from $1$ to $0$ when $\sqrt{\delta}\log n$ exceeds $0.5$,
while that of DP+ sharply drops off from $1$ to $0$ when $\sqrt{\delta}\log n$ exceeds $2$. 
 \begin{figure}[ht]
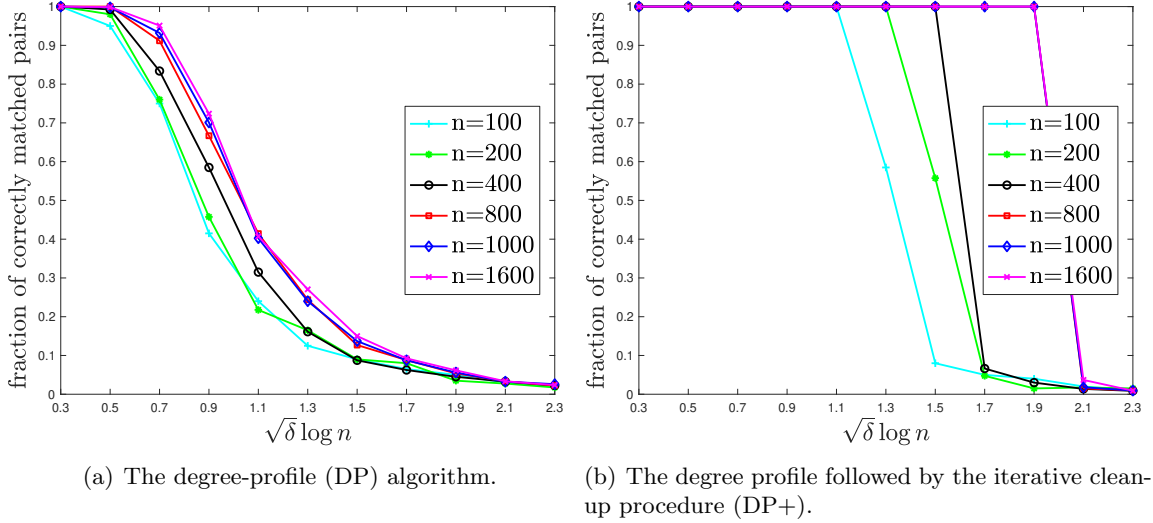

	 \centering
	 \subfigure[The degree-profile (DP)  algorithm.]%
	 {\includegraphics[width=0.45\textwidth]{"\figpath ER_phase_DP"}
        \label{fig:ER}}
	 \subfigure[The degree profile followed by the iterative clean-up procedure (DP+).]%
	 {\includegraphics[width=0.45\textwidth]{"\figpath ER_phase_DPplus"}
    \label{fig:ER_cleanup} }
	 \caption{Simulated correlated \ER graph model with varying $n$ and $\delta$ with 
    edge probability in the parent graph fixed to be $\log^2(n)/n$. 
For each value of $\delta$, the accurate rate shown is the median of $10$ independent runs.
    }
	\label{fig:ER_phase}
\end{figure}


\subsection{Subsampled Real Graphs}
In this section, we generate two graphs $A$ and $B$
by independently subsampling a real parent graph $G$. 

Inspired by previous work~\cite{kazemi2015growing}, we consider the Slashdot network.
The Slashdot network contains links between the users of Slashdot 
(a technology-related news website). The network was obtained in February 2009
and is available on Stanford Large Network Dataset Collection (SNAP)~\cite{SNAP09}.
To generate the parent graph $G$, we first focus on the subnetwork induced by 
the users whose ID is at most $750$, and then connect user $i$ and user $j$
if either $i$ has a directed link to $j$ or vice versa. This gives rise to a graph
$G$ with $750$ vertices and $3338$ edges. The graph $G$ is connected and has a heavy-tailed
degree distribution. In particular,  there are
$216$ degree-$1$ vertices, $102$ degree-$2$ vertices, and the average degree is 
around $9$, while the maximum degree is $524$ and there are $9$ vertices whose degree is at 
least $100.$

To obtain two correlated graphs $A$ and $B$, 
we first independently subsample the edges of $G$ twice 
with probability $s$, and then relabel the vertices in $B$
according to a random permutation $\pi^*$. 

We simulate the performance of the three algorithms (DP, QP, and SP) 
as well as their cleaned-up version, with inputs $A$ and $B$.
The edge subsampling probability 
$s$ varies from $0.6$ to $1$, or equivalently $\delta$ varies from $0$ to $0.4$, 
and the results are shown in \prettyref{fig:SlashDot}. 

\begin{figure}[ht]
\centering
\includegraphics[width=0.5\columnwidth]{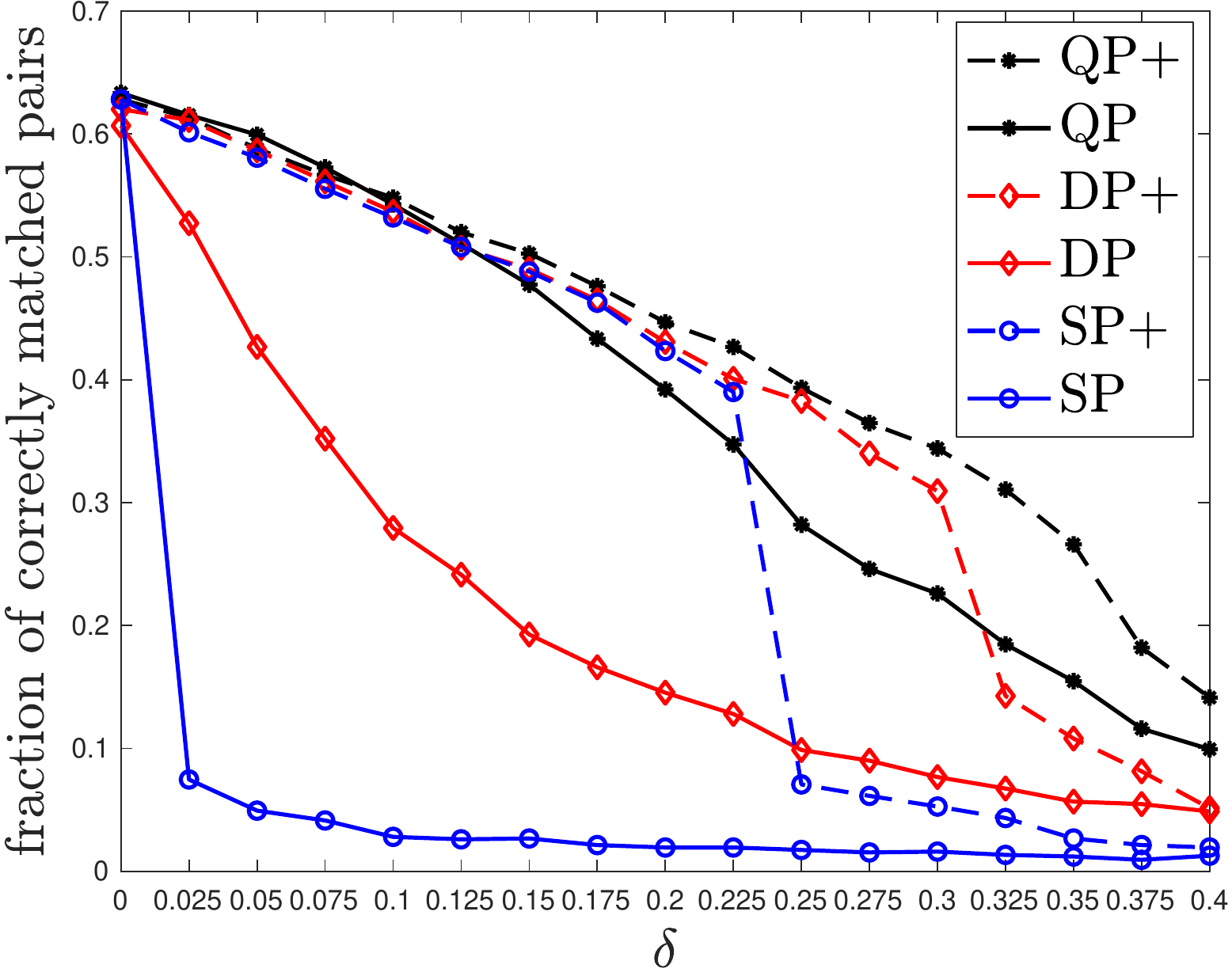}
\caption{ Slashdot network with $n=750$ and varying $ \delta=1-s$. 
For each value of $\delta$, 
the accurate rate shown is the median of $10$ independent runs.}
\label{fig:SlashDot}
\end{figure}

Note that in the noiseless case of $\delta=0$, the accuracy rates of 
all three algorithms as well as their cleaned-up version are about the same 
and around $0.62$. However, in the noisy case,  
QP dominates DP, which, in turn, significantly outperforms SP in term of the matching accuracy;
this is consistent with the observations in the previous two subsections.
In particular, as soon as $\delta$ becomes positive, the accuracy of SP
drops off sharply as expected because the leading eigenvectors 
of $A$ and $B$ are highly sensitive to the perturbation. In contrast, 
the accuracy rates of DP and QP drop off gradually as $\delta$ increases.

Analogous to our synthetic experiments, the iterative clean-up
procedure significantly improves the accuracy of all three methods.
In fact, the accuracy rates of 
all three methods after clean-up (QP+, DP+, and SP+)
are about the same for all $\delta \le 0.225$. At $\delta=0.25$,
the accuracy rate of SP+ drops off sharply, while the 
accuracy rates of QP+ and DP+ continue to decrease gradually 
and match each other 
until $\delta \le 0.3$. At $\delta=0.325$, the accuracy rate of DP+
drops off sharply, while the 
accuracy rate of QP+ continues to decrease gradually.  

Computationally, to generate the simulation results in \prettyref{fig:SlashDot}, 
QP takes about $290$ minutes, DP takes about $2$ minutes, and SP takes about $18$ seconds.

\begin{appendices}

\section{Auxiliary results}
	\label{app:aux}
	
	Recall the following tail bound for binomial random variable 
    $X\sim \Binom(n,p)$ \cite[Theorems 4.4, 4.5]{Mitzenmacher05} 
	\begin{align}
\prob{X \geq (1+t) np}  &\leq e^{-\frac{t^2}{3} np}, \quad 0 \leq t \leq 1 \nonumber \\
	\prob{X \leq (1-t) np}  &\leq e^{-\frac{t^2}{2} np}, \quad 0 \leq t \leq 1
	\label{eq:bintail}
	\end{align}
      and 
    \begin{align}
    \prob{X \ge R} \le 2^{-R}, \quad R \ge 6np. \label{eq:bintail_large}
    \end{align}

    \begin{theorem}[\cite{Okamoto1959}]\label{thm:binom_lower_tail}
Let $X \sim \Bin(n,p)$. It holds that
\begin{align}
\prob{X \le n t} &\le \exp \left( - n \left( \sqrt{p} - \sqrt{t} \right)^2\right), \quad \forall 0 \le t \le p 
\label{eq:bintail_lower}\\
\prob{X \ge n t} & \le \exp \left( - 2n  \left( \sqrt{t} - \sqrt{p} \right)^2\right), \quad \forall p \le t \le 1. 
\label{eq:bintail_upper}
\end{align}
\end{theorem}

\section{Analysis for seeded graph matching}
\label{app:seed}

In this section we analyze \prettyref{alg:seed} for seeded graph matching. 
Note that when \prettyref{alg:seed} is used as a subroutine in \prettyref{alg:distdeg}, 
the seed set $S$ is obtained from \prettyref{alg:dist} based on 
matching degree profiles, which can potentially depend on the edges between the non-seeded vertices.
To deal with this dependency, the following lemma gives a sufficient condition for the seeded graph matching subroutine 
(\prettyref{alg:seed}) to succeed, even if the seed set is chosen adversarially:

\begin{lemma}[Seeded graph matching]
\label{lmm:seed}	
Assume $n\geq 4$, $s \ge 30 q$, and 
\begin{equation}
n (qs)^2 \ge 2^{11} \times 3 \log^2 n.
\label{eq:nqslog2}
\end{equation}
\nb{If the number of seeds satisfies $m \ge \frac{96 \log n}{q s}$, then with probability $1 - 5n^{-1}$, the following holds}:
for any $\pi_0:S \to T$ that coincides with true permutation $\pi^*$ on the seed set $S$, 
(i.e. $\pi_0 = \pi^*|_S$) with $|S|=m$, 
\prettyref{alg:seed} with $\pi_0$ as the seed set and threshold $\kappa =\frac{1}{2} mqs$ 
 outputs $\hat \pi = \pi$.	
\end{lemma}

We start by analyzing the first stage of \prettyref{alg:seed}, which upgrades a partial (but correct) permutation
$\pi_0: S \to T$ to a full permutation $\pi_1:[n] \to [n]$ with at most $O(\log n/q)$ errors, 
even if the seed set $S$ is adversarially chosen.

\begin{lemma}\label{lmm:seed_matching_mm}
Assume 
$n\geq 2$, 
$m q s \ge 96 \log n$, and $s \ge 12q$.
Recall the threshold
$
\kappa =\frac{1}{2} mqs 
$
in \prettyref{alg:seed}.
Then with probability at least $1- 2n^{ -m }$, the following holds
in \prettyref{alg:seed}:
for any partial permutation $\pi_0: S \to T$ such that $\pi_0 = \pi^*|_S$ and $|S|=m$, 
$\pi_1$  is guaranteed to have at most $\frac{192\log n}{qs}$ errors with respect to $\pi^*$,
i.e., $|\{i\in[n]: \pi_1(i) \neq \pi^*(i) \}| \leq \frac{192\log n}{qs}$.
\end{lemma} 

\begin{proof}[Proof of \prettyref{lmm:seed_matching_mm}]
Without loss of generality, we assume $\pi^*$ is the identity permutation.

Fix a seed set $S$ of cardinality $m$.
Since $\pi_0 = \pi^*|_S$, it follows that 
$$
n_{ik} = \sum_{j \in S} A_{ij} B_{k \pi_0(j)} =\sum_{j \in S} A_{ij} B_{k \pi^*(j)}.
$$

Recall that according to the definition of the weights in \prettyref{eq:weightw}, we have
\[
w(\pi^*) = \sum_{i \in S^c} \indc{n_{ii} \geq \kappa}. 
\]
First, we show that 
\begin{align}
\prob{ w(\pi^*)  \le  n -m-  \frac{ 32 \log n}{qs}} \le \exp \left(  - 2 m \log n  \right),
\label{eq:weight_true}
\end{align}
Indeed, for $i \in S^c$ we have $n_{i i} \iiddistr \Binom(m, qs)$. 
It follows from the Chernoff bound \prettyref{eq:bintail} that
$$
\prob{ n_{ii} \le \kappa} 
=\prob{ n_{ii} \le \frac{1}{2} mqs} \le \exp \left(  - \frac{1}{8} m q s \right).
$$
Therefore, 
$$
(n-m)-w(\pi^*) = \sum_{i \in S^c} \indc{n_{ii} < \kappa} 
\overset{s.t.}{\leq} \Binom \left( n-m,  \exp \left(  - \frac{1}{8} m q s \right) \right).
$$
Using the following fact (which follows from a simple union bound)
\begin{align}
\prob{ \Binom \left(n, p \right) \ge t} \le \binom{n}{t} p^t, \label{eq:binom_union_bound}
\end{align}
we get that
$$
\prob{ (n-m)-w(\pi^*) \ge t } \le \binom{n-m}{t} \exp \left(  - \frac{t}{8} m q s \right)
\le n^t \exp \left(  - \frac{t}{8} m q s \right) 
\le \exp \left(  - \frac{t}{16} m q s \right),
$$
where the last inequality holds due to the assumption that $mqs \ge 16 \log n$.
Setting $t=\frac{ 32 \log n}{qs}$, we arrive at the desired \prettyref{eq:weight_true}.

Next, fix any permutation $\pi$ such that $\pi |_S = \pi_0$ and it has $\ell$ non-fixed points. 
Since  by assumption $\pi_0=\pi^* |_S$ and $\pi^*$ is the identity permutation,
it follows that $\pi(i)=i$ for all $i \in S$.
Let $F = \{i \in S^c: \pi(i) = i\}$ denote the set of fixed points in $S^c$. Then 
$|F|=n-m-\ell$ and $|S^c\backslash F|=\ell$. Thus
$$
w(\pi) = \sum_{i \in F} \indc{n_{ii} \ge \kappa} + \sum_{ i \in S^c \backslash F} \indc{n_{i\pi(i)} \ge \kappa }
\le n-m-\ell + \sum_{ i \in S^c\backslash F} \indc{n_{i\pi(i)} \ge \kappa }.
$$
Note that for each $i \in S^c \backslash F$, $n_{i\pi(i)} \sim \Binom(m, q^2)$.
Since by assumption $s \ge 12q$, it follows that 
$\kappa = mqs/2 \ge 6 m q^2$. Hence, the Chernoff bound \prettyref{eq:bintail_large} 
yields that 
for each $i \in S^c \backslash F$,
$$
\prob{ n_{i \pi(i)} \ge \kappa } \le 2^{-  m q s/2 }
\le \exp \left(  - \frac{1}{4} m q s  \right).
$$

Note that $\{n_{i\pi(i)}: i \in S^c \backslash F\}$ are not mutually independent.
For instance, $n_{i \pi (i)}$ and $n_{\pi(i), \pi(\pi(i))}$ are dependent. To deal with this
dependency issue, we construct a subset $\calI
\subset S^c \backslash F$ with $|\calI| \ge \ell/3$
such that $\{n_{i\pi(i)}: i \in \calI\}$ are mutually independent.
In particular, consider the canonical cycle decomposition of permutation $\pi|_{S^c \backslash F}$. 
Let $\calC_1, \ldots, \calC_{a}$ denote the cycles. 
Since $\pi$ has no fixed point
in $S^c \backslash F$, each cycle $\calC_i$ has length $\ell_i \ge 2$.
Let $\Gamma$ denote the graph formed by the union of these cycles. 
Each cycle 
$C_i$ has an independent set $\calI_i$ of size $\lfloor \ell_i /2 \rfloor \ge \ell_i/3$. 
Let $\calI= \cup_{i=1}^a \calI_i$. 
Then $\calI$ is an independent set in $\Gamma$ and $|\calI| \ge \sum_{i=1}^a \ell_i/3=\ell/3$.
Since $\calI$ is an independent set, 
it follows that $\{i, \pi(i)\} \cap \{j, \pi(j)\} =\emptyset$ for all 
$i \neq j \in \calI$. Therefore, $\{n_{i\pi(i)}: i \in \calI\}$ are mutually independent.
Therefore, 
$$
\sum_{ i \in \calI } \indc{n_{i\pi(i)} \ge \kappa }
\overset{s.t.}{\leq} \Binom \left( |\calI| , \exp \left(  - \frac{1}{4} m q s \right) \right).
$$
Note that 
$$
w(\pi) \le n-m-\ell + \sum_{ i \in S^c\backslash F} \indc{n_{i\pi(i)} \ge \kappa }
\le n-m-|\calI| + \sum_{i \in \calI } \indc{n_{i\pi(i)} \ge \kappa }
$$
Using \prettyref{eq:binom_union_bound} again, we have
\begin{align*}
\prob{ w(\pi) \ge n -m-  \frac{ 32 \log n}{qs} }
& \le \prob{ \sum_{ i \in \calI} \indc{n_{i\pi(i)} \ge \kappa } \ge |\calI| -\frac{ 32 \log n}{qs} } \\
& \le \binom{|\calI| }{|\calI| - \frac{ 32 \log n}{qs}} \exp \left(  - \frac{1}{4} m q s \left( |\calI| - \frac{ 32 \log n}{qs}\right) \right) \\
& \le 2^{ \ell } \exp \left(  - \frac{1}{4} m q s \left( \frac{\ell}{3} - \frac{ 32 \log n}{qs}\right)\right) 
\le 2^{\ell} \exp \left(  - \frac{1}{24} m q s \ell  \right),
\end{align*}
where the last inequality holds provided $\ell qs\ge 192 \log n$.
Let $\Pi_\ell$ denote the set of permutations 
$\pi$ which has $\ell$ non-fixed points and satisfies $\pi |_S = \pi_0$.
Then $|\Pi_\ell| \le \binom{n-m}{\ell} \ell! \le n^\ell$. 
By the union bound, we have that for any $\ell \ge \frac{ 192 \log n}{qs}$,
$$
\prob{ \max_{\pi \in \Pi_\ell} w(\pi) \ge n -m-  \frac{ 32 \log n}{qs} }
\le (2n)^{\ell} \exp \left(  - \frac{1}{24} m q s \ell \right)
\le \exp \left(  - \frac{1}{48} m q s  \ell \right),
$$
where the last inequality holds due to the assumption that $mqs \ge 96 \log n$ and $n \ge 2$.
Applying the union bound again over $\ell$, we get that
\begin{align*}
\prob{ \max_{\ell \ge\frac{ 192 \log n}{qs}} \max_{\pi \in \Pi_\ell} 
w(\pi) \ge n -m-  \frac{ 32 \log n}{qs} }
& \le \sum_{\ell \ge \frac{ 192 \log n}{qs}} \exp \left(  - \frac{1}{48} m q s  \ell \right) \\
& \le \frac{ \exp \left(  - 4 m \log n   \right) }{ 1 - \exp \left(  - 4 m \log n  \right)} \\
& \le \exp \left(  - 2 m \log n   \right),
\end{align*}
where the last inequality holds due to $m \log n \ge \log 2$.
Combining the last displayed equation with \prettyref{eq:weight_true}
we get that with probability at least $1- 2 n^{-2m}$,
$\pi_1$ has at most $192\log n/(qs)$ errors with respect to $\pi^*$.

Finally, applying a simple union bound over all the 
$\binom{n}{m} \le n^m$
possible choices of seed set $S$ with $|S|=m$, 
we complete the proof. 
\qed
\end{proof}

The second stage of \prettyref{alg:seed} upgrades an almost exact full permutation $\pi_1:[n] \to [n]$ to an exact full permutation $\hat{\pi}: [n] \to [n]$.
The following lemma provides a worst-case guarantee even if $\pi_1$ is adversarially chosen.

\begin{lemma}\label{lmm:seed_matching_last}
Let $0 \le \ell \le n$.  
Assume that $(\ell-1) qs \ge 12 nq^2 +2 $ and $ (\ell-1) q s \ge 16 \max\{ 1, n-\ell\} \log n$. 
Then with probability at least $1-3n^{-1}$, the following holds
for \prettyref{alg:seed}:
for any $\pi_1$ with at most $n-\ell$ errors with respect to the true permutation $\pi^*$, we have
$\hat{\pi}=\pi^*$. 
\end{lemma}

\begin{proof}

Without loss of generality, we assume $\pi^*$ is the identity permutation. 

We first fix a permutation $\pi_1$
which has at least $\ell$ fixed points.
Let $F \subset [n]$ denote the set of fixed points of $\pi_1$.
Then $|F| \ge \ell$.  
Recall that
$$
w_{ik} = \sum_{j \in [n]} A_{ij} B_{k \pi_1(j)}. 
$$
Then for $i=k$, 
$$
w_{ii} \ge \sum_{j \in F \setminus \{i\} } A_{ij} B_{i j} 
\overset{s.t.}{\geq} \Binom( |F| -1, qs ).
$$
\nb{Similarly, for $i \neq k$, note that $A_{ij} B_{k \pi_1(j)} =0$  if $j=i$ or $j=\pi_1^{-1}(k)$. Thus, 
$w_{ik} = \sum_{j \in [n]\backslash\{i,\pi_1^{-1}(k)\}  }A_{ij} B_{k \pi_1(j)}. $
Moreover, $A_{ij} B_{k \pi_1(j)} \iiddistr \Bern(q^2)$ for all $j \in [n]\backslash\{i,\pi_1^{-1}(k) , k \}$.
Therefore, 
$$
w_{ik}  \le \sum_{j \in [n]\backslash\{i,\pi_1^{-1}(k), k \} } A_{ij} B_{k \pi_1(j)} +1 
 \overset{s.t.}{\leq} \Binom( n-2 , q^2 ) + 1.
$$
}
 It follows from the Chernoff bound \prettyref{eq:bintail} that 
 $$
 \prob{ w_{ii} \le \frac{1}{2} (\ell -1) qs }
 \le \prob{ \Binom\left( |F| -1 , qs \right) \le  \frac{1}{2} (\ell -1) qs}
 \le \exp \left(  - \frac{1}{8} (\ell -1) q s \right).
 $$
 Thus, by the union bound,
 $$
 \prob{ \min_{i \in [n] } w_{ii} \le \frac{1}{2} (\ell -1) qs }
 \le n \exp \left(  - \frac{1}{8} (\ell -1) q s \right) \le \exp \left(  - \frac{1}{16} (\ell -1) q s \right),
 $$
 where the last inequality holds due to the assumption that $ (\ell -1) q s \ge 16 \log n$.
 Moreover, since by assumption $ (\ell -1) qs /2 -1 \ge 6 n q^2$, it follows that
 the Chernoff bound \prettyref{eq:bintail_large} that for any $i \neq k$,
 $$
  \prob{ w_{ik} \ge \frac{1}{2} (\ell -1) qs }
  \le \prob{ \Binom(n-2, q^2) \ge \frac{1}{2} (\ell -1) qs -1 }
  \le 2^{ -  (\ell -1) qs /2 +1 } \le 2\exp \left( -\frac{1}{4}  (\ell -1) qs \right).
 $$
 Thus, by the union bound again, 
$$
 \prob{ \max_{i \neq k } w_{ik} \ge \frac{1}{2} (\ell -1) qs }
 \le 2n^2 \exp \left(  - \frac{1}{4} (\ell -1) q s \right) 
 \le 2 \exp \left(  - \frac{1}{8} (\ell -1) q s \right).
 $$

 In conclusion, 
for a fixed permutation $\pi_1$
 with at least $\ell$ fixed points, 
with probability at least $1-3\exp \left(  - \frac{1}{8} (\ell-1) q s \right) $,
 $$
\min_{i \in [n] } w_{ii} > \max_{i \neq k } w_{ik},
 $$
 and hence $\hat{\pi} = \pi^*$.

Finally, applying a simple union bound over all the 
$\binom{n}{n-\ell} (n-\ell)! \le n^{n-\ell}$
possible choices of permutation $\pi_1$ with at least $\ell$ fixed points, 
we get that even if $\pi_1$ is adversarially chosen, 
$\hat{\pi} = \pi^*$ with probability at least 
$$
1- 3 n^{n-\ell} 
\exp \left(  - \frac{1}{8} (\ell -1) q s \right) 
\ge 1- 3 \exp \left(  - \frac{1}{16} (\ell -1) q s \right) \ge 1-3n^{-1},
$$
where the first inequality holds due to $(\ell -1) qs \ge 16(n-\ell) \log n$
and the last inequality holds due to $(\ell -1) qs \ge 16 \log n$.
\qed
\end{proof}

We now prove \prettyref{lmm:seed}:
\begin{proof}[Proof of \prettyref{lmm:seed}]
In view of \prettyref{lmm:seed_matching_mm}, 
we get that with probability at least $1- 2n^{ -m }$, 
$\pi_1$ is guaranteed to have at most $192 \log n/(qs)$ errors with respect to $\pi^*$, 
even if $\pi_0$, or equivalently the seed set $S$, is adversarially chosen.

We next apply \prettyref{lmm:seed_matching_last} with $\ell = n- 192 \log n/(qs)$.
In view of the assumption $ n (qs)^2 \ge 2^{11} \times 3 \log^2 n$ and $n \ge 4$, 
we have 
$(\ell-1) \ge n/2$. Thus $(\ell-1) qs \ge n qs /2 \ge 16 \log n$,
and $(\ell-1) qs \ge nq s /2 \ge 12 nq^2+2$ in view of 
$s \ge 30 q$ and $nqs \ge 20$. Moreover, 
$(\ell-1) qs \ge n qs /2 \ge 2^{10} \times 3 \log^2 n / (qs) = 16(n-\ell) \log n$.
Therefore, all assumptions of \prettyref{lmm:seed_matching_last}
are satisfied. It follows from \prettyref{lmm:seed_matching_last}
that with probability at least 
$1-3n^{-1}$, $\hat{\pi}=\pi^*$, even if $\pi_1$ is adversarially chosen.

In conclusion, we get that with probability at least $1-5n^{-1}$,
\prettyref{alg:seed} with $\pi_0$ as the seed set outputs $\hat \pi = \pi$.
\qed
\end{proof}

\end{appendices}

\section*{Acknowledgment}

J.~Ding is supported in part by the NSF Grant DMS-1757479 and an Alfred Sloan fellowship.
Z.~Ma is supported in part by an NSF CAREER award DMS-1352060 and an Alfred Sloan fellowship.
Y.~Wu is supported in part by the NSF Grant CCF-1527105, an NSF CAREER award CCF-1651588, and an Alfred Sloan fellowship.
J.~Xu is supported by the NSF Grants CCF-1755960 and IIS-1838124.

J.~Ding and Y.~Wu would like to thank the Centre de Recherches Math\'ematiques at the Universit\'e de Montr\'eal, where some of the work was carried out during the Workshop on Combinatorial Statistics.
Y.~Wu is also grateful to David Pollard for helpful discussions on small ball probability.
J.~Xu would like to thank Nadav Dym  and Shahar Kovalsky for pointing out the connections between
fractional isomorphism and iterated degree sequences.
The authors are grateful to the anonymous referees for helpful comments and corrections.


\begin{thebibliography}{FQRM{\etalchar{+}}16}

\bibitem[ABK15]{aflalo2015convex}
Yonathan Aflalo, Alexander Bronstein, and Ron Kimmel.
\newblock On convex relaxation of graph isomorphism.
\newblock {\em Proceedings of the National Academy of Sciences},
  112(10):2942--2947, 2015.

\bibitem[AS08]{AlonSpencer08}
Noga Alon and Joel~H. Spencer.
\newblock The probabilistic method (the third edition), 2008.

\bibitem[BCL{\etalchar{+}}18]{barak2018nearly}
Boaz Barak, Chi-Ning Chou, Zhixian Lei, Tselil Schramm, and Yueqi Sheng.
\newblock ({N}early) efficient algorithms for the graph matching problem on
  correlated random graphs.
\newblock {\em arXiv preprint arXiv:1805.02349}, 2018.

\bibitem[BCPP98]{burkard1998quadratic}
Rainer~E Burkard, Eranda Cela, Panos~M Pardalos, and Leonidas~S Pitsoulis.
\newblock The quadratic assignment problem.
\newblock In {\em Handbook of combinatorial optimization}, pages 1713--1809.
  Springer, 1998.

\bibitem[BES80]{babai1980random}
L{\'a}szl{\'o} Babai, Paul Erd\"os, and Stanley~M Selkow.
\newblock Random graph isomorphism.
\newblock {\em SIAM Journal on computing}, 9(3):628--635, 1980.

\bibitem[BK13]{BK13}
Daniel Berend and Aryeh Kontorovich.
\newblock A sharp estimate of the binomial mean absolute deviation with
  applications.
\newblock {\em Statistics \& Probability Letters}, 83(4):1254--1259, 2013.

\bibitem[BLM15]{BordenaveLelargeMassoulie:2015dq}
C.~Bordenave, M.~Lelarge, and L.~Massouli{\'e}.
\newblock Non-backtracking spectrum of random graphs: community detection and
  non-regular {R}amanujan graphs.
\newblock In {\em 2015 IEEE 56th Annual Symposium on Foundations of Computer
  Science (FOCS)}, pages 1347--1357, 2015.
\newblock arXiv 1501.06087.

\bibitem[Bol82]{bollobas1982distinguishing}
B{\'e}la Bollob{\'a}s.
\newblock Distinguishing vertices of random graphs.
\newblock {\em North-Holland Mathematics Studies}, 62:33--49, 1982.

\bibitem[Bol01]{bollobas1998random}
B{\'e}la Bollob{\'a}s.
\newblock {\em Random Graphs (2nd Edition)}.
\newblock Cambridge Studies in Advanced Mathematics, 2001.

\bibitem[BPC{\etalchar{+}}11]{boyd2011distributed}
Stephen Boyd, Neal Parikh, Eric Chu, Borja Peleato, and Jonathan Eckstein.
\newblock Distributed optimization and statistical learning via the alternating
  direction method of multipliers.
\newblock {\em Foundations and Trends{\textregistered} in Machine learning},
  3(1):1--122, 2011.

\bibitem[CFSV04]{conte2004thirty}
Donatello Conte, Pasquale Foggia, Carlo Sansone, and Mario Vento.
\newblock Thirty years of graph matching in pattern recognition.
\newblock {\em International journal of pattern recognition and artificial
  intelligence}, 18(03):265--298, 2004.

\bibitem[CK16]{cullina2016improved}
Daniel Cullina and Negar Kiyavash.
\newblock Improved achievability and converse bounds for
  {E}rd\"{o}s-{R}\'{e}nyi graph matching.
\newblock In {\em Proceedings of the 2016 ACM SIGMETRICS International
  Conference on Measurement and Modeling of Computer Science}, pages 63--72.
  ACM, 2016.

\bibitem[CK17]{cullina2017exact}
Daniel Cullina and Negar Kiyavash.
\newblock Exact alignment recovery for correlated {E}rd\"{o}s-{R}\'{e}nyi
  graphs.
\newblock {\em arXiv preprint arXiv:1711.06783}, 2017.

\bibitem[CKMP18]{cullina2018partial}
Daniel Cullina, Negar Kiyavash, Prateek Mittal, and H~Vincent Poor.
\newblock Partial recovery of {Erd\H{o}s-R\'{e}nyi} graph alignment via $ k
  $-core alignment.
\newblock {\em arXiv preprint arXiv:1809.03553}, Nov. 2018.

\bibitem[CP08]{czajka2008improved}
Tomek Czajka and Gopal Pandurangan.
\newblock Improved random graph isomorphism.
\newblock {\em Journal of Discrete Algorithms}, 6(1):85--92, 2008.

\bibitem[dBGM99]{dBGM99}
Eustasio del Barrio, Evarist Gin{\'e}, and Carlos Matr{\'a}n.
\newblock Central limit theorems for the {W}asserstein distance between the
  empirical and the true distributions.
\newblock {\em Annals of Probability}, pages 1009--1071, 1999.

\bibitem[DCKG18]{dai2018performance}
Osman~Emre Dai, Daniel Cullina, Negar Kiyavash, and Matthias Grossglauser.
\newblock On the performance of a canonical labeling for matching correlated
  {E}rd\"{o}s-{R}\'{e}nyi graphs.
\newblock {\em arXiv preprint arXiv:1804.09758}, 2018.

\bibitem[DML17]{dym2017ds++}
Nadav Dym, Haggai Maron, and Yaron Lipman.
\newblock {DS++}: a flexible, scalable and provably tight relaxation for
  matching problems.
\newblock {\em ACM Transactions on Graphics (TOG)}, 36(6):184, 2017.

\bibitem[DN03]{DN70}
H.A. David and H.N. Nagaraja.
\newblock {\em Order Statistics}.
\newblock Wiley-Interscience, Hoboken, New Jersey, USA, 3 edition, 2003.

\bibitem[FAP18]{Fishkind2018Seeded}
Donniell~E. Fishkind, Sancar Adali, and Carey~E. Priebe.
\newblock Seeded graph matching.
\newblock {\em arXiv preprint arXiv:1209.0367}, 2018.

\bibitem[FF56]{ford1956maximal}
Lester~R Ford and Delbert~R Fulkerson.
\newblock Maximal flow through a network.
\newblock {\em Canadian journal of Mathematics}, 8(3):399--404, 1956.

\bibitem[FQRM{\etalchar{+}}16]{feizi2016spectral}
Soheil Feizi, Gerald Quon, Mariana Recamonde-Mendoza, Muriel Medard, Manolis
  Kellis, and Ali Jadbabaie.
\newblock Spectral alignment of graphs.
\newblock {\em arXiv preprint arXiv:1602.04181}, 2016.

\bibitem[FS15]{fiori2015spectral}
Marcelo Fiori and Guillermo Sapiro.
\newblock On spectral properties for graph matching and graph isomorphism
  problems.
\newblock {\em Information and Inference: A Journal of the IMA}, 4(1):63--76,
  2015.

\bibitem[HK73]{Hopcroft1971}
John~E. Hopcroft and Richard~M. Karp.
\newblock An $n^{5/2}$ algorithm for maximum matchings in bipartite graphs.
\newblock {\em SIAM Journal on Computing}, 2(4):225--231, 1973.

\bibitem[HNM05]{haghighi2005robust}
Aria~D Haghighi, Andrew~Y Ng, and Christopher~D Manning.
\newblock Robust textual inference via graph matching.
\newblock In {\em Proceedings of the conference on Human Language Technology
  and Empirical Methods in Natural Language Processing}, pages 387--394.
  Association for Computational Linguistics, 2005.

\bibitem[KB80]{KB80}
Rob Kaas and Jan~M Buhrman.
\newblock Mean, median and mode in binomial distributions.
\newblock {\em Statistica Neerlandica}, 34(1):13--18, 1980.

\bibitem[KHG15]{kazemi2015growing}
Ehsan Kazemi, S~Hamed Hassani, and Matthias Grossglauser.
\newblock Growing a graph matching from a handful of seeds.
\newblock {\em Proceedings of the VLDB Endowment}, 8(10):1010--1021, 2015.

\bibitem[KHGM16]{kazemi2016proper}
Ehsan Kazemi, Hamed Hassani, Matthias Grossglauser, and Hassan~Pezeshgi
  Modarres.
\newblock Proper: global protein interaction network alignment through
  percolation matching.
\newblock {\em BMC bioinformatics}, 17(1):527, 2016.

\bibitem[KKBL15]{kezurer2015tight}
Itay Kezurer, Shahar~Z Kovalsky, Ronen Basri, and Yaron Lipman.
\newblock Tight relaxation of quadratic matching.
\newblock In {\em Computer Graphics Forum}, volume~34, pages 115--128. Wiley
  Online Library, 2015.

\bibitem[KL14]{korula2014efficient}
Nitish Korula and Silvio Lattanzi.
\newblock An efficient reconciliation algorithm for social networks.
\newblock {\em Proceedings of the VLDB Endowment}, 7(5):377--388, 2014.

\bibitem[LFF{\etalchar{+}}16]{lyzinski2016graph}
Vince Lyzinski, Donniell Fishkind, Marcelo Fiori, Joshua Vogelstein, Carey
  Priebe, and Guillermo Sapiro.
\newblock Graph matching: Relax at your own risk.
\newblock {\em IEEE Transactions on Pattern Analysis \& Machine Intelligence},
  38(1):60--73, 2016.

\bibitem[LFP13]{Lyzinski2013Seeded}
Vince Lyzinski, Donniell~E. Fishkind, and Carey~E. Priebe.
\newblock Seeded graph matching for correlated {Erd\H{o}s-R\'{e}nyi} graphs.
\newblock {\em Journal of Machine Learning Research}, 15, 2013.

\bibitem[LR13]{Livi2013}
Lorenzo Livi and Antonello Rizzi.
\newblock The graph matching problem.
\newblock {\em Pattern Analysis \& Applications}, 16(3):253--283, 2013.

\bibitem[LS01]{LS01}
Wenbo~V Li and Q-M Shao.
\newblock Gaussian processes: inequalities, small ball probabilities and
  applications.
\newblock {\em Handbook of Statistics}, 19:533--597, 2001.

\bibitem[LS18]{lubars2018correcting}
Joseph Lubars and R~Srikant.
\newblock Correcting the output of approximate graph matching algorithms.
\newblock In {\em IEEE INFOCOM 2018-IEEE Conference on Computer
  Communications}, pages 1745--1753. IEEE, 2018.

\bibitem[MMS10]{makarychev2010maximum}
Konstantin Makarychev, Rajsekar Manokaran, and Maxim Sviridenko.
\newblock Maximum quadratic assignment problem: Reduction from maximum label
  cover and lp-based approximation algorithm.
\newblock {\em Automata, Languages and Programming}, pages 594--604, 2010.

\bibitem[MR17]{mossel2017shotgun}
Elchanan Mossel and Nathan Ross.
\newblock Shotgun assembly of labeled graphs.
\newblock {\em IEEE Transactions on Network Science and Engineering}, 2017.

\bibitem[MU05]{Mitzenmacher05}
Michael Mitzenmacher and Eli Upfal.
\newblock {\em Probability and Computing: Randomized Algorithms and
  Probabilistic Analysis}.
\newblock Cambridge University Press, New York, NY, USA, 2005.

\bibitem[MX18]{mossel2018seeded}
Elchanan Mossel and Jiaming Xu.
\newblock Seeded graph matching via large neighborhood statistics.
\newblock To appear in 2019 ACM-SIAM Symposium on Discrete Algorithms (SODA),
  arXiv preprint arXiv:1807.10262, 2018.

\bibitem[NK08]{nadarajah2008exact}
Saralees Nadarajah and Samuel Kotz.
\newblock Exact distribution of the max/min of two {G}aussian random variables.
\newblock {\em IEEE Transactions on very large scale integration ({VLSI})
  systems}, 16(2):210--212, 2008.

\bibitem[NS08]{narayanan2008robust}
Arvind Narayanan and Vitaly Shmatikov.
\newblock Robust de-anonymization of large sparse datasets.
\newblock In {\em Security and Privacy, 2008. SP 2008. IEEE Symposium on},
  pages 111--125. IEEE, 2008.

\bibitem[NS09]{narayanan2009anonymizing}
Arvind Narayanan and Vitaly Shmatikov.
\newblock De-anonymizing social networks.
\newblock In {\em Security and Privacy, 2009 30th IEEE Symposium on}, pages
  173--187. IEEE, 2009.

\bibitem[Oka59]{Okamoto1959}
Masashi Okamoto.
\newblock Some inequalities relating to the partial sum of binomial
  probabilities.
\newblock {\em Annals of the Institute of Statistical Mathematics},
  10(1):29--35, Mar 1959.

\bibitem[OV17]{onaran2017projected}
Efe Onaran and Soledad Villar.
\newblock Projected power iteration for network alignment.
\newblock {\em arXiv preprint arXiv:1707.04929}, 2017.

\bibitem[Pet95]{petrov}
Valentin~V. Petrov.
\newblock {\em Limit theorems of probability theory: Sequences of independent
  random variables}.
\newblock Oxford Science Publications, Clarendon Press, Oxford, United Kingdom,
  1995.

\bibitem[PG11]{Pedarsani2011}
Pedram Pedarsani and Matthias Grossglauser.
\newblock On the privacy of anonymized networks.
\newblock In {\em ACM SIGKDD International Conference on Knowledge Discovery
  and Data Mining}, pages 1235--1243, 2011.

\bibitem[PRW94]{Pardalos94thequadratic}
Panos~M. Pardalos, Franz Rendl, and Henry Wolkowicz.
\newblock The quadratic assignment problem: A survey and recent developments.
\newblock In {\em In Proceedings of the DIMACS Workshop on Quadratic Assignment
  Problems, volume 16 of DIMACS Series in Discrete Mathematics and Theoretical
  Computer Science}, pages 1--42. American Mathematical Society, 1994.

\bibitem[SGE17]{Shirani2017Seeded}
F~Shirani, S~Garg, and E~Erkip.
\newblock Seeded graph matching: Efficient algorithms and theoretical
  guarantees.
\newblock {\em arXiv preprint arXiv:1805.02349}, 2017.

\bibitem[SNA09]{SNAP09}
Slashdot social network.
\newblock \url{https://snap.stanford.edu/data/soc-Slashdot0902.html}, Feb.
  2009.

\bibitem[SS05]{schellewald2005probabilistic}
Christian Schellewald and Christoph Schn{\"o}rr.
\newblock Probabilistic subgraph matching based on convex relaxation.
\newblock In {\em EMMCVPR}, volume~5, pages 171--186. Springer, 2005.

\bibitem[SU97]{FGT}
Edward~R Scheinerman and Daniel~H Ullman.
\newblock {\em Fractional graph theory: a rational approach to the theory of
  graphs}.
\newblock Dover, 1997.

\bibitem[SW86]{Shorack.Wellner}
G.~R. Shorack and J.~A. Wellner.
\newblock {\em Empirical processes with applications to statistics}.
\newblock John Wiley \& Sons, 1986.

\bibitem[SXB08]{singh2008global}
Rohit Singh, Jinbo Xu, and Bonnie Berger.
\newblock Global alignment of multiple protein interaction networks with
  application to functional orthology detection.
\newblock {\em Proceedings of the National Academy of Sciences},
  105(35):12763--12768, 2008.

\bibitem[Wri71]{wright1971graphs}
Edward~M Wright.
\newblock Graphs on unlabelled nodes with a given number of edges.
\newblock {\em Acta Mathematica}, 126(1):1--9, 1971.

\bibitem[YG13]{yartseva2013performance}
Lyudmila Yartseva and Matthias Grossglauser.
\newblock On the performance of percolation graph matching.
\newblock In {\em Proceedings of the first ACM conference on Online social
  networks}, pages 119--130. ACM, 2013.

\bibitem[ZKRW98]{zhao1998semidefinite}
Qing Zhao, Stefan~E Karisch, Franz Rendl, and Henry Wolkowicz.
\newblock Semidefinite programming relaxations for the quadratic assignment
  problem.
\newblock {\em Journal of Combinatorial Optimization}, 2(1):71--109, 1998.

\bibitem[ZS13]{ZubkovSerov13}
A.~M. Zubkov and A.~A. Serov.
\newblock A complete proof of universal inequalities for the distribution
  function of the binomial law.
\newblock {\em Theory of Probability \& Its Applications}, 57(3):539--544,
  2013.

\end{thebibliography}
\newcommand{\etalchar}[1]{$^{#1}$}

\end{document}